\documentclass[number]{ReportTemplate}
\usepackage{utopia}
\usepackage{amssymb}
\usepackage{amsfonts}
\usepackage[fleqn]{amsmath}
\usepackage{mathtools}
\usepackage{amsthm}
\usepackage{bm}
\usepackage{graphicx}
\usepackage{algorithmic}
\usepackage{hyperref}
\mathtoolsset{showonlyrefs,showmanualtags}

\newcommand{\expect}[2][]{
    {\mathbb{E}_{#1}[\kern-0.15em[ #2 ]\kern-0.14em]}
    }
\newcommand{\cexpect}[1]{\mathbb{E}( #1 )}

\begin{document}
\begin{frontmatter}

\title{Towards Analyzing Crossover Operators in
Evolutionary Search\\via General Markov Chain Switching Theorem}
\author{Yang Yu}
\ead{yuy@lamda.nju.edu.cn}
\author{Chao Qian}
\ead{qianc@lamda.nju.edu.cn}
\author{Zhi-Hua Zhou\corref{cor1}}
\ead{zhouzh@nju.edu.cn}
\cortext[cor1]{Corresponding author}
\address{National Key Laboratory for Novel Software Technology\\
Nanjing University, Nanjing 210046, China}

\begin{abstract}
Evolutionary algorithms (EAs), simulating the evolution process of natural species, are used to solve optimization problems. Crossover (also called \emph{recombination}), originated from simulating the chromosome exchange phenomena in zoogamy reproduction, is widely employed in EAs to generate offspring solutions, of which the effectiveness has been examined empirically in applications. However, due to the irregularity of crossover operators and the complicated interactions to mutation, crossover operators are hard to analyze and thus have few theoretical results. Therefore, analyzing crossover not only helps in understanding EAs, but also helps in developing novel techniques for analyzing sophisticated metaheuristic algorithms.

In this paper, we derive the \emph{General Markov Chain Switching Theorem} (GMCST) to facilitate theoretical studies of crossover-enabled EAs. The theorem allows us to analyze the running time of a sophisticated EA from an easy-to-analyze EA. Using this tool, we analyze EAs with several crossover operators on the LeadingOnes and OneMax problems, which are noticeably two well studied problems for mutation-only EAs but with few results for crossover-enabled EAs. We first derive the bounds of running time of the (2+2)-EA with crossover operators; then we study the running time gap between the mutation-only (2:2)-EA and the (2:2)-EA with crossover operators; finally, we develop strategies that apply crossover operators only when necessary, which improve from the mutation-only as well as the crossover-all-the-time (2:2)-EA. The theoretical results are verified by experiments.
\end{abstract}

\end{frontmatter}

\newpage
\section{Introduction}

Evolutionary algorithms are inspired by the evolution process of natural species, i.e., nature selection and survival of the fittest, and are used as randomized optimization algorithms, which have been widely applied to diverse areas (e.g. \cite{Higuchi.etal.tec99,Koza.etal.IS03,freitas:03}). Starting from a random population of solutions, EAs iteratively apply reproduction operators to generate a set of offspring solutions from the current population, and then apply a selection operator to weed out bad solutions. As EAs were motivated by simulation of natural evolution process, the reproduction operators are originated from reproduction phenomena of natural animals, typically including {\it mutation} and {\it crossover} (also called {\it recombination}), which simulate the mutation phenomena in DNA transformation and the chromosome exchange phenomena in zoogamy reproduction, respectively. EAs significantly differ from classical optimization techniques, e.g. branch-and-bound strategy \cite{Russell:Norvig:AIMA}, as well as from heuristic search methods, e.g. simulated annealing \cite{Russell:Norvig:AIMA}, in that EAs maintain a population of solutions during the evolution process \cite{back:96} other than a single solution, and apply mutation and crossover operators on the  population.

Contrary to mutation operators, crossover operators are defined on a population of solutions, i.e., they generate offspring solutions by mixing up a set of (usually two) solutions. Moreover, crossover operators were born together with the first genetic algorithm. Crossover is therefore a special characteristic of EAs.
Questions related to crossover, such as why it works and how frequently it should be used, have been argued since the emerging of EAs \cite{goldberg:89}. It was initially explained by the schema theory that crossover can utilize `building blocks' to construct good solutions. However, it had been proved \cite{Eiben:Rudolph:1999} that the schema theory concerns little to the convergence rate of EAs, while the convergence rate was later proved to have a tight connection to the running time \cite{Yu:Zhou:06}, which is the center concern of EAs solving optimization problems, and thus the schema theory is of little use. There have been studies disclosing several properties of crossover operators. In \cite{Lin:Yao:CEC97}, the `step size', measured by inversion of the expected distance between the leftmost bit and the leftmost crossover point in a solution, is assumed critical to crossover operator and is studied empirically. In \cite{spears2000ear}, using schema theory as well as a multimodel problem generator for experiments, the construction and destruction probabilities, and equilibrium of crossover operator were studied thoroughly. While these studies provided valuable hints, we are more interested in the theoretical properties of crossover impacting the running time of EAs.


The field of theoretical analysis of EAs develops very fast in the recent decade \cite{he:yao:01,Yu:Zhou:06,neumann.witt.10,auger.doerr.11}, and a landscape of computational complexity of EAs is emerging. However, due to the hardness of analyzing crossover operators, these theoretical studies did not involve crossover until recently (e.g. \cite{jansen2002analysis,Doerr.Happ.GECCO08,richter2008ignoble}). It has been found that crossover can play an important role on the running time of EAs. In some cases, the crossover operators are the key to solve problems, they drastically reduce the running time. Meanwhile, opposite effect of crossover has also been discovered.  Detailed related analysis work is reviewed in Section 2. The analysis approach used to obtain the results could be summarized as follows. Given a particular EA and a particular problem, define phases of the optimization process, and bound the time consumed in each phase, so that the total time is bounded. Hence the analysis is done case-by-case.

In this paper, we derive the \emph{General Markov Chain Switching Theorem} (GMCST) to facilitate the analysis of EAs with crossover. The theorem compares two Markov chains for their average running time of hitting a target state. GMCST compares the one-step transition behaviors of the two Markov chains, while involving the long-term behavior of only one of the chains. Therefore, to analyze a sophisticated chain, we can compare it with a simply-to-analyze chain via GMCST, so that we only need to consider the one-step transition behavior of the sophisticated chain. Section 5 presents the GMCST.

Modeling EAs by Markov chains, we analyze crossover-enabled EAs using GMCST in this paper. We use two model problems in the analysis, the LeadingOnes and the OneMax problems, which are introduced in Section 3. It is noteworthy that though the two model problems are the most well-studied problems for mutation-only EAs, there are few results on how the crossover effect the running time on the problems, since they are not specifically designed for easing the analysis of crossover-enabled EAs. To theoretically study the EAs with crossover, we particularly concern three aspects:
\begin{itemize}
  \item How to bound the asymptotic running time of crossover-enabled EAs. Asymptotic running time is a central theoretical problem of algorithm analysis, and it is particularly interesting for sophisticated algorithms such as EAs with crossover.
  \item How is the running time of crossover-enabled EAs compared with their counterpart mutation-only EAs. Practically, we choose among possible configurations of an algorithm to solve a problem. Thus it would be helpful to choose a good configuration by knowing whether apply of a crossover operator is good or not.
  \item Can we use crossover smartly rather than throughout the evolutionary process. Usually the operators are used no matter what the current solutions are. Operators however are rarely useful all the time, thus applying the operators only when necessary would improve the performance of EAs.
\end{itemize}
We analyze crossover operators in the three aspects Section 6, 7 and 8, respectively. In Section 6, we use GMCST to bound the running time of (2+2)-EA with crossover on the model problems, and obtain upper bounds as well as lower bounds. The analysis also demonstrates how to use GMCST to bound the running time. In Section 7, we compare the running time of (2:2)-EA with and without crossover, which for the first time discloses that the crossover harm the running time on the model problems.  In Section 8, we then design strategies that use crossover only under some conditions according to the GMCST, rather than throughout the evolutionary process. We prove that the designed strategies result in better running time. The theoretical results are verified by experiments in Section 9. The paper is concluded in Section 10.

\section{Related Work}

There have been cases analyzing the influence of crossover on the running time of EAs. We review some of the results in this section.

\citet{rabani1995computational} proposed to model the crossover process by a quadratic Markov chain, which also reveals the hardness of analysis of crossover. Using a set partition model, the process of crossover was reduced so that the mixing time can be bounded. However, the crossover process analyzed was without mutation or selection, thus the result has little connection with the running time of EAs.

In \cite{watson2001analysis}, the running time of several crossover-only algorithms was analyzed on the H-IFF problem which is hard for any kind of mutation-based EAs. In the H-IFF problem, a solution is divided into sub-blocks with $\log n$ levels, and the fitness of the solution is counted as how many sub-blocks are pure (all 1 or all 0) in all the levels. The problem is so designed that local optima and global optima are distant in Hamming space, but are close in crossover space. In \cite{dietzfelbinger2003analysis}, on the same H-IFF problem, the asymptotic running time of a simplified crossover-only EA was proved to be $\Theta(n \ln n)$, by analysis from two equivalent coloring games on the fitness trees.

In \cite{jansen2002analysis,jansen2005real}, it was proved that on JUMP and Real Royal Road problems, the EA with mutation only requires super-polynomial and exponential running time, but when the crossover operator is applied (with a small probability and a large population), the running time reduces to be polynomial. In the problem of JUMP, solutions with the same number of 1 bits are assigned the same fitness. Similar as in the OneMax problem, the fitness increases as the number of 1 bits increases, except in a valley surrounding the optimal solution, where the fitness are set very low. EA with mutation only reaches the optimal solution only by jumping over the gap, which can take super-polynomial running time, while the gap is eliminated in the crossover space. The problem of Real Royal Road is alike, but encourages long blocks of 1 bits. \citet{Kotzing.etal.gecco11} enriched the analysis results by showing that a small population is enough for crossover, but a large crossover probability will lead to super-polynomial running time.

The Ising model is a graph coloring problem which encourages vectors to have the same color as their neighbors. It was proved that for Ising model on rings \cite{fischer2005one} and on trees \cite{sudholt2005crossover}, crossover operator helps to reduce the running time from $O(n^3)$ to $O(n^2)$ and from $\Omega(2^n)$ to $O(n^3)$, respectively.

In \cite{lehre2008crossover}, a TwoPath problem in computing unique-input-output sequences of finite state machines was studied, where there are local optima trapping the mutation-only EAs. On this problem, the crossover operator reduces the running time from $\Omega(2^n)$ to $O(n^2 \mu \log \mu+exp(n \ln n- \mu /96))$, where the latter can be polynomial if $\mu$, the population size, is larger than $96n \ln n$.

In \cite{oliveto2008analysis}, the vertex cover problem was examined for population-based EAs. It was found that, on a particular class of vertex cover problem, no single bit mutation can save the EA from local optima, but only crossover operator can. Thus the crossover operator helps to reduce the running time from exponential to be polynomial.

In \cite{doerr2008crossover,doerr2009improved}, the all-pairs shortest path problem was studied, which is a non-artificially designed problem class. The crossover operators improve the running time from $O(n^4)$ to $O(n^{3.5} \log^{1/2}(n))$, and then to $O(n^{3.25} \log^{1/4}(n))$ by their improved analysis. The analysis reveals that mutation and crossover operators can work in different phases of the optimization of the problem. In \cite{doerrmore}, crossover on the all-pairs shortest path problem was further studied. Two concepts in crossover, repair mechanisms and parent selection, were analyzed, which led to improved expected running time $O(n^{3.2}(\log n)^{0.2})$ and $O(n^{3}\log n)$, respectively.

In \cite{neumann.etal.gecco11}, the crossover was studied in the setting of the island model of parallel EA. In the island model, the population of the EA is divided into several subpopulations. The subpopulations evolve independently, and only for an interval of time the superior solutions of a subpopulation migrate to other subpopulations in order to share the information of evolution. The island model creates diversity among subpopulations naturally, which are particular suitable for crossover. \citet{neumann.etal.gecco11} proved the effectiveness of crossover for a constructed problem and an instance of the vertex cover problem using the parallel EA.

While most of the studies prove the usefulness of crossover, \citet{richter2008ignoble} presented a show case where it can be proved that the crossover operator harms the running time of the EA. In that work, a problem called Ignoble Trails is constructed. In the problem, over the solution space $\{0,1\}^n$, all the solutions are designed to have increased fitness towards a trap solution $0^n$, except one path towards the optimal solution. For this problem, the mutation-only EA was proved to follow the path easily, which costs $O(n^k \ln k)$ steps with dominating probability and $k$ being a constant. However, when the crossover operator is enabled, the EA will get stuck in a local minimum due to the crossover operator, and will take exponential time to jump out the local minimum with dominating probability.

Most recently, there emerge studies of the effect of crossover in multi-objective EAs. \citet{Neumann.Theile.PPSN10} first showed that crossover can be helpful in multi-objective EAs by constructing solutions that fit better the subsequent mutation process. \citet{qian.etal.gecco11} discovered that crossover can accelerate fulfilling the Pareto front for multi-objective problems.

As a summary of the common analysis approach used to obtain the above results, one can define phases of the optimization process by considering the specific characteristics of the problem, and bound the time consumed in each phase, so that the total time is bounded, and in the analysis of each phase, one finds how mutation and crossover act. To show that a crossover enabled EA is more powerful, the asymptotic time complexity is bounded and is shown lower than that of mutation-only EA. However, there would be no general guild to analyze crossover-enabled EAs.

Moreover, it is notably that diversity among solutions is not only regarded as a key factor for the effectiveness of crossover, and is also a condition for theoretical analysis. For examples, the invariant GA \cite{dietzfelbinger2003analysis} and the the shuffle GA \cite{Kotzing.etal.gecco11} crossover the current solution with a virtual solution which is generated from the current solution to have a large diversity. However, this kind of crossover operators can actually be implemented by mutation operators, and is quite different with the crossover over a population of evolved solutions. We are therefore interested in analyzing crossover-enabled EAs without extra diversity encouraging mechanisms.



\section{Evolutionary Algorithms and Model Problems}

The (1+1)-EA \cite{droste:jansen:wegener:98} is the simplest EA described in Definition \ref{(1+1)-EA}, which maintains one solution and uses a mutation operator to generate one new solution at a time.
\begin{definition}[(1+1)-EA]\label{(1+1)-EA} Given solution length $n$ and objective function $f$, (1+1)-EA
consists of the following steps:\\
    \begin{tabular}{ll}
    1. & (Initialization) $s:=$ randomly selected from $\{0,1\}^{n}$.\\
    2. & (Reproduction) $s':=\textsf{mutation}(s)$; \\
    3. & (Selection) If {$f(s')\geq f(s)$}, $s:=s'$\\
    4. & (Stop Criterion) Terminates if $s$ is optimal.\\
    5. & (Loop) Goto step 2.\\
    \end{tabular}\\
where $\textsf{mutation}(\cdot)$ is a mutation operator.
\end{definition}
The \textsf{mutation} is commonly implemented by the one-bit mutation or the bitwise mutation. The (1+1)-EA with one-bit mutation is also called \emph{randomized local search}.

\begin{description}
    \item[one-bit mutation] for each solution, randomly choose one of the $n$ bits, and flip (0 to 1 or inverse) the chosen bit.
    \item[bitwise mutation] for each solution, flip each bit with probability $1/n$.
\end{description}

When the (1+1)-EA generates solutions with the same fitness as its current solution, it performs a random walk as it will accept the newly generated solutions. We define the (1+1$_>$)-EA as the (1+1)-EA without random walk, i.e., it only accepts better solutions by setting $f(s')> f(s)$ in the line 4 of the (1+1)-EA.

The (1+1)-EA does not involve crossover operators, as a crossover operator recombines parts of at least two solutions. A possible extension of (1+1)-EA for enabling crossover operators is to recombine the current solution with a virtual solution, which can be as the gene invariant GA \cite{dietzfelbinger2003analysis}  recombining with the inverse of the current solution and as the shuffle GA \cite{Kotzing.etal.gecco11} recombining with a random shuffle of the current solution. We however study a more common setting where the EA maintains a population of solutions and applies an crossover operator on the solutions in the population.

We consider the (2:2)-EA as in Definition \ref{(2:2)-EA} and the (2+2)-EA as in Definition \ref{(2+2)-EA} both maintaining two solutions in a population. Note that the case of population size being 2 is sufficient to show the effect of crossover operators as used in \cite{storch2004royal,richter2008ignoble}.
\begin{definition} [(2:2)-EA]\label{(2:2)-EA} Encode each solution by a string with $n$ binary bits, and let every population, denoted by variable $\xi$, contain $2$ solutions. The (2:2)-EA consists of the following steps:\\
    \begin{tabular}{ll}
    1. & (Initialization) $t\leftarrow 0$. $\xi_0 := $ randomly selected two solutions from $\{0,1\}^n$.\\
    2. & Let $\{s_1,s_2\}$ denote the two solutions in $\xi_{t}$. \\
    3. & (Reproduction) Choose $r \in [0,1]$ uniformly at random.\\
        & If $r < p_c$, $\{s'_1,s'_2\}:= \textsf{crossover}(s_1,s_2)$\\
        & else,\, $\{s'_1,s'_2\}:= \{\textsf{mutation}(s_1),\textsf{mutation}(s_2)\}$. \\
    4. & (Selection) $\xi_{t+1} :=\{\mathop{\arg\max}\limits_{s \in \{s_1, s'_1\}} f(s), \mathop{\arg\max}\limits_{s \in \{s_2, s'_2\}} f(s)\}$. \\
    5. & (Stop Criterion) Terminates if an optima is found. \\
    6. & (Loop) $t\leftarrow t+1$. Goto step 2.
    \end{tabular}\\
where $\textsf{mutation}(\cdot)$ is a mutation operator that maps
$\mathcal{X}\to\mathcal{X}$, $\textsf{crossover}(\cdot,\cdot)$ is a
crossover operator that maps $\mathcal{X}\times \mathcal{X} \to
\mathcal{X}\times \mathcal{X}$, $p_c$ is the crossover probability.
\end{definition}

\begin{definition} [(2+2)-EA]\label{(2+2)-EA} Encode each solution by a string with $n$ binary bits, and let every population, denoted by variable $\xi$, contain $2$ solutions. The (2+2)-EA consists of the following steps:\\
    \begin{tabular}{ll}
    1. & (Initialization) $t\leftarrow 0$. $\xi_0 := $ randomly selected two solutions from $\{0,1\}^n$.\\
    2. & Let $\{s_1,s_2\}$ denote the two solutions in $\xi_{t}$. \\
    3. & (Reproduction) Choose $r \in [0,1]$ uniformly at random.\\
        & If $r < p_c$, $\{s'_1,s'_2\}:= \textsf{crossover}(s_1,s_2)$\\
        & else,\, $\{s'_1,s'_2\}:= \{\textsf{mutation}(s_1),\textsf{mutation}(s_2)\}$. \\
    4. & (Selection) $\xi_{t+1} := \text{the best two solutions in $\{s_1,s_2,s'_1,s'_2\}$ which have the two largest fitness value}$. \\
    5. & (Stop Criterion) Terminates if an optima is found. \\
    6. & (Loop) $t\leftarrow t+1$. Goto step 2.
    \end{tabular}\\
where $\textsf{mutation}(\cdot)$ is a mutation operator that maps
$\mathcal{X}\to\mathcal{X}$, $\textsf{crossover}(\cdot,\cdot)$ is a
crossover operator that maps $\mathcal{X}\times \mathcal{X} \to
\mathcal{X}\times \mathcal{X}$, $p_c$ is the crossover probability.
\end{definition}

Both the (2:2)-EA and the (2+2)-EA apply the crossover operator instead of the mutation operator with probability $p_c$. They differ in the way of selection: (2:2)-EA compares between an offspring solution and its direct parent which is equivalent to two parallel (1+1)-EA if $p_c=0$, and (2+2)-EA compares among all the offspring and parent solutions. Note that, for (2:2)-EA, we track the offspring by their names, i.e., $s_1$ and $s_2$. For example, no matter how many bits of $s_1$ are exchanged with $s_2$, the result solution is the offspring of $s_1$.

Our analysis will involve some crossover operators:
\begin{description}
    \item[one-point crossover] for the current two solutions, scan the solutions left-to-right, randomly choose one of the first $n-1$ bit positions, and exchange all the bits after the position.
    \item[uniform crossover] for the current two solutions, exchange each bit with probability $1/n$.
    \item[one-bit crossover] for the current two solutions, randomly choose one of the $n$ bit positions, and exchange the bit on that position.
\end{description}

EAs can be used for various problems, but are often analyzed on some model problems to discover their theoretical properties. The LeadingOnes and the OneMax are two model problems, defined respectively in Definitions \ref{def_leadingones} and \ref{def_onemax}, that have been used to study EAs. It has been known that (1+1)-EA takes expected $\Theta(n^2)$ running time on the LeadingOnes problem, and $\Theta(n\ln n)$ running time on the OneMax problem \cite{Droste.etal.ECJ98}. However, the results about how crossover operators effect EAs are few.

\begin{definition}[LeadingOnes Problem]\label{def_leadingones}
    LeadingOnes Problem of size $n$ is to find an $n$ bits binary
    string $s^*$ such that, defining $LO(s)=\sum^{n}_{i=1} \prod^{i}_{j=1} s_j$,
    $$
        s^*=\mathop{\arg\max}\limits_{s \in \{0,1\}^n} LO(s) .
    $$
\end{definition}

\begin{definition}[OneMax Problem]\label{def_onemax}
    OneMax Problem of size $n$ is to find an $n$ bits binary
    string $s^*$ such that
    $$
        s^*=\mathop{\arg\max}_{s \in \{0,1\}^n} \sum^{n}_{i=1} s_i.
    $$
\end{definition}

\section{Modeling EAs as Markov Chains}

EAs are modeled and analyzed as Markov chains \cite{he:yao:01,Yu:Zhou:06}. Considering combinatorial optimization problems, a Markov chain with discrete state space is constructed to model an EA, by mapping the populations of EAs to the states of the Markov chain. Suppose an EA encodes a solution into a vector with length $n$, each component of the vector is drawn from an alphabet set $\mathcal{B}$, and each population contains $m$ solutions. Let $s(a)$ denote the $a$-th bit of the solution $s$. Let $\mathcal{S}$ denote the solution space, which contains $|\mathcal{S}|=|\mathcal{B}|^n$ solutions. Let $\mathcal{X}$ denote the population space, which contains $|\mathcal{X}| = {{m+|\mathcal{B}|^n-1} \choose m}$  populations \cite{suzuki:95}. A Markov chain $\{\xi_t\}^{+\infty}_{t=0}$ modeling the EA is constructed by taking $\mathcal{X}$ as the state space, i.e. $\xi_t\in \mathcal{X}$.
Let $\mathcal{X}^* \subset \mathcal{X}$ denote the set of all optimal populations, which contains at least one optimal solution. The goal of EAs is to reach $\mathcal{X}^*$ from an initial population. Thus, the process of an EA seeking $\mathcal{X}^*$ can be analyzed by studying the corresponding Markov chain \cite{he:yao:01,Yu:Zhou:06}.

\begin{definition}[Absorbing Markov Chain]\label{def_absorbing}
    Given a Markov chain $\{\xi_t\}^{+\infty}_{t=0} (\xi_t\in \mathcal{X})$ and a target subspace $\mathcal{X}^*\subset \mathcal{X}$, $\{\xi_t\}^{+\infty}_{t=0}$ is said to be an absorbing chain, if
    \begin{align}\label{eq_absorbing}
        \forall t\geq 0: P(\xi_{t+1} \notin \mathcal{X}^* \mid \xi_t \in \mathcal{X}^*) = 0 \ .
    \end{align}
\end{definition}
All practical EAs track the best-so-far solutions during the evolution process. This kind of EAs can be modeled as absorbing Markov chains. EAs with no time-variant operators can be modeled as homogeneous Markov chains, where it holds that
$$
\forall t\geq 0 \forall x,y \in \mathcal{X}:P(\xi_{t+1} \mid \xi_t) = P(\xi_{1} \mid \xi_0).
$$

\begin{definition}[Conditional first hitting time, CFHT]\label{def_CFHT}
    Given a Markov chain $\{\xi_t\}^{+\infty}_{t=0} (\xi_t\in \mathcal{X})$ and a target subspace $\mathcal{X}^*\subset \mathcal{X}$, starting from time $\tilde{t}$ when $\xi_{\tilde{t}}=x$, let $\tau_{\tilde{t}}$ be a random variable that denotes the hitting events:
    \begin{quote}
        $\tau_{\tilde{t}}=0:\, \xi_{\tilde{t}} \in \mathcal{X}^*$,\\
        $\tau_{\tilde{t}}=1:\, \xi_{\tilde{t}+1} \in \mathcal{X}^* \wedge \xi_i \notin \mathcal{X}^* \   (i=\tilde{t})$ ,\\
        $\tau_{\tilde{t}}=2:\, \xi_{\tilde{t}+2} \in \mathcal{X}^* \wedge \xi_i \notin \mathcal{X}^* \   (i=\tilde{t},\tilde{t}+1)$ ,\\
        \dots .
    \end{quote}
    The mathematical expectation of $\tau_{\tilde{t}}$,
    \begin{equation}
        \expect{\tau_{\tilde{t}} \mid \xi_{\tilde{t}}=x} = \sum_{i=0}^{+\infty} i\cdot  P(\tau_{\tilde{t}}=i) ,
    \end{equation}
    is called the conditional first hitting time (CFHT) of the Markov chain from $\tilde{t}$ and $\xi_{\tilde{t}}=x$.
\end{definition}
\begin{definition}[Distribution-CFHT, DCFHT]\label{def_DCFHT}
    Given a Markov chain $\{\xi_t\}^{+\infty}_{t=0} (\xi_t\in \mathcal{X})$ and a target subspace $\mathcal{X}^*\subset \mathcal{X}$, starting from time $\tilde{t}$, if $\xi_{\tilde{t}}$ is drawn from a distribution $\pi$ of states, the mathematical expectation of the CFHT over $\xi_{\tilde{t}}$,
    \begin{align}
        \expect{\tau_{\tilde{t}} \mid \xi_{\tilde{t}}\sim \pi } & = \expect[x\sim \pi]{\tau_{\tilde{t}} \mid \xi_{\tilde{t}}=x} \\
        & = \sum_{x\in \mathcal{X}} \pi(x)\expect{\tau_{\tilde{t}} \mid \xi_{\tilde{t}}=x} ,
    \end{align}
    is called the distribution-conditional first hitting time (DCFHT) of the Markov chain from $\tilde{t}$ and $\xi_{\tilde{t}}\sim \pi$.
\end{definition}
\begin{definition}[Expected first hitting time]\label{def_EFHT}
    Given a Markov chain $\{\xi_t\}^{+\infty}_{t=0} (\xi_t\in \mathcal{X})$ and a target subspace $\mathcal{X}^*\subset \mathcal{X}$, the DCFHT of the chain from $t=0$ and uniform distribution $\pi_u$,
    \begin{align}
        \expect{\tau} & = \expect{\tau_0 \mid \xi_0\sim \pi_u} \\
        & = \expect[x\sim \pi_u]{\tau_0 \mid \xi_0=x} \\
        &  = \sum_{x\in \mathcal{X}} \frac{1}{|\mathcal{X}|} \expect{\tau_{0} \mid \xi_{0}=x},
    \end{align}
    is called the expected first hitting time (EFHT) of the Markov chain.
\end{definition}
Note that this definition of EFHT is equivalent to those used in \cite{he:yao:01,he:yao:03,Yu:Zhou:06}.
The EFHT implies the average computational time complexity of EAs, thus provides a measure of \emph{goodness} for EAs. Since the Markov chain models the essential of EA process, EFHT of an EA or its corresponding Markov chain won't be distinguished for convenience. The following are two lemmas on properties of Markov chain \cite{Freidlin:97}.
\begin{lemma}\label{lem_tau}
        Given an absorbing Markov chain $\{\xi_t\}^{+\infty}_{t=0}(\xi_t\in \mathcal{X})$ and a target  subspace   $\mathcal{X}^*\subset \mathcal{X}$, we have, for CFHT,
        \begin{align}
        &\forall x\notin \mathcal{X}^*:\expect{\tau_t\mid \xi_t=x}
        = 1+
        \sum\limits_{y\in \mathcal{X}} P(\xi_{t+1}=y \mid \xi_t=x)\expect{\tau_{t+1}\mid \xi_{t+1}=y},
        \shortintertext{and for DCFHT,}
        &\expect{\tau_t\mid \xi_t\sim \pi_t} = \expect[x\sim \pi_t]{\tau_{{t}} \mid \xi_{{t}}=x}
         = 1-\pi_t(\mathcal{X}^*) + \expect{\tau_{t+1}\mid \xi_{t+1}\sim \pi_{t+1}},
        \end{align}
        where $\pi_{t+1}(x) = \sum_{y\in\mathcal{X}} \pi_t(y) P(\xi_{t+1}=x \mid \xi_t=y)$.
\end{lemma}
\begin{lemma}\label{lem_homo}
        Given an absorbing homogeneous Markov chain $\{\xi_t\}^{+\infty}_{t=0}(\xi_t\in \mathcal{X})$ and a target subspace $\mathcal{X}^*\subset \mathcal{X}$, 
        it holds
        $
        \forall t_1, t_2:\expect{\tau_{t_1}\mid \xi_{t_1}=x} = \expect{\tau_{t_2}\mid  \xi_{t_2}=x} .
        $
\end{lemma}

\section{General Markov Chain Switching Theorem}

Given two Markov chains $\{\xi_t\}_{t=0}^{+\infty}$ and $\{\xi'_t\}_{t=0}^{+\infty}$ modeling two EAs, let $\tau$ and $\tau'$ denote the hitting events of the two chains, respectively. We present the general Markov chain switching theorem, stated in Theorem \ref{them_main}, that compares $\expect{\tau}$ with $\expect{\tau'}$

\begin{theorem}[\emph{General Markov Chain Switching Theorem} (GMCST)]\label{them_main}
        Given two absorbing homogeneous Markov chains $\{\xi_t\}_{t=0}^{+\infty}$ and  $\{\xi'_t\}_{t=0}^{+\infty}$. Let $\mathcal{X}$ and $\mathcal{Y}$ denote the state space of $\xi_t$ and $\xi'_t$, respectively. Let $\tau$ and $\tau'$ denote the hitting events of $\xi_t$ and $\xi'_t$, respectively. Let $\pi_t$ denote the distribution of $\xi_t$. Let $\{\rho_t\}_{t=0}^{+\infty}$ be a series of numbers whose sum converges to $\rho$. If there exists a mapping $\phi: \mathcal{X} \rightarrow \mathcal{Y}$, $\phi(x) \in \mathcal{Y}^{*}$ if and only if $x \in
        \mathcal{X}^*$; and it satisfies that
        \begin{align} \label{GMCST_condition}
            & \forall t: \sum\limits_{\mathclap{x\in \mathcal{X}, y \in \mathcal{Y}}} \pi_t(x) P(\xi_{t+1}\in\phi^{-1}(y) \mid \xi_t=x) \expect{\tau' \mid \xi'_{t+1} = y} \\
            & \leq(\geq) \sum\limits_{\mathclap{y_1,y_2\in \mathcal{Y}}} \pi'_t(y_1) P(\xi'_{t+1}=y_2 \mid \xi'_t=y_1) \expect{\tau' \mid \xi'_{t+1} =y_2} +\rho_t,
        \end{align}
         where $\phi^{-1}(y)=\{x\in\mathcal{X}\mid \phi(x)=y\}$ and $\pi'_t(y)=\pi_t(\phi^{-1}(y))$, and $\expect{\tau \mid \xi_{0} \sim \pi_0}$ is finite,  it will hold that
         \begin{align}
         \expect{\tau \mid \xi_{0} \sim \pi_0} \leq(\geq) \expect{\tau' \mid \xi'_{0} \sim \pi'_0} +\rho .
         \end{align}
\end{theorem}

\begin{proof}
    We prove ``$\leq$'' case, since ``$\geq$'' can be proved similarly. Denote the transition of $\xi$ as $tr$, and the transition of by $\xi'$ as $tr'$.

    We define the intermediate Markov chain $\{\xi^{k}\}_{t=0}^{+\infty}$ as a Markov chain that
    \begin{enumerate}
      \item is in the state space $\mathcal{X}$ and has the same initial state distribution with the chain $\xi$, i.e., $\pi_0^k = \pi_0$;
      \item uses $tr$ at time $\{0,1, \ldots, k-1\}$, i.e., it is identical with the chain $\xi$ at the first $k-1$ steps;
      \item switches to the state space $\mathcal{Y}$ just before time $k$, which is by mapping the distribution $\pi_k$ of states over $\mathcal{X}$ to the distribution $\pi'_k$ of states over $\mathcal{Y}$ via $\phi$;
      \item uses $tr'$ from time $k$.
    \end{enumerate}
    For any $t<k$, since the chain $\xi^k$ and $\xi$ are identical before the time $k$, we have $\pi_t=\pi_t^k$, and thus
    \begin{equation}\label{GMCST_proof_eq_2}
            \forall t <k:
            \pi_t^k(\mathcal{X}^*) = \pi_t(\mathcal{X}^*)=\pi'_t(\mathcal{Y}^*),
        \end{equation}
     since $\pi'_0(y)=\pi_0(\phi^{-1}(y))$ and $\phi(x) \in \mathcal{Y}^*$ if and only if $x \in\mathcal{X}^*$.\\
    For any $t\geq k$, since the chain $\xi^k$ and $\xi'$ are in the same space and uses the same transition, we have
        \begin{equation}\label{GMCST_proof_eq_1}
            \forall y\in \mathcal{Y}: \expect{\tau^{k} \mid \xi^{k}_{t} = y} = \expect{\tau' \mid \xi'_{t} = y},
        \end{equation}

        We prove the theorem by induction on the $k$ of the intermediate Markov chain $\xi^k$.\\
        {\bf (a) Initialization} is to prove  $
             \expect{\tau^{1} \mid \xi^{1}_{0} \sim \pi_0}
                \leq \expect{\tau' \mid \xi'_{0} \sim \pi'_0} + \rho_0
        $, i.e., $k=1$. Since $tr$ is applied only at $t=0$,
        \begin{align*}
                & \expect{ \tau^{1} \mid \xi^{1}_{0} \sim \pi_0}           = \sum\nolimits_{x\in \mathcal{X}} \pi_0(x) \expect{ \tau^{1} \mid \xi^{1}_{0} =x}\\
                & = 1 - \pi_0(\mathcal{X}^*) + \sum\limits_{\mathclap{x\in \mathcal{X},y\in \mathcal{Y}}} \pi_0(x) P(\xi^{1}_{1}\in\phi^{-1}(y) \mid \xi^{1}_{0}=x)\expect{\tau^{1} \mid \xi^{1}_{1} = y} \\
                & = 1 - \pi_0(\mathcal{X}^*) + \sum\limits_{\mathclap{x\in \mathcal{X},y\in \mathcal{Y}}} \pi_0(x)  P(\xi^{1}_{1}\in\phi^{-1}(y) \mid\xi^{1}_{0}=x)\expect{\tau' \mid  \xi'_{1}= y} \\
                & \leq 1 - \pi_0(\mathcal{X}^*) + \sum\limits_{\mathclap{y_1,y_2\in \mathcal{Y}}} \pi'_0(y_1)  P(\xi'_{1}=y_2 \mid\xi'_{0}=y_1)\expect{\tau' \mid  \xi'_{1}= y_2} +\rho_0\\
                & = 1 - \pi'_0(\mathcal{Y}^*) + \sum\limits_{y_1,y_2\in \mathcal{Y}} \pi'_0(y_1)  P(\xi'_{1}=y_2 \mid\xi'_{0}=y_1)\expect{\tau' \mid  \xi'_{1}= y_2} +\rho_0\\
                & = \expect{\tau' \mid \xi'_{0} \sim \pi'_0}+\rho_0,
        \end{align*}
        where the first, second and the last equalities are by Lemma \ref{lem_tau}, the third is by Eq.\refeq{GMCST_proof_eq_1}, the following inequality is by Eq.\refeq{GMCST_condition} and the fourth equality is by Eq.\refeq{GMCST_proof_eq_2} .\\
        {\bf (b) Inductive Hypothesis} assumes that,  $$\forall k\leq K-1(K>1), \expect{\tau^{k} \mid \xi^{k}_{0} \sim\pi_0}\leq \expect{\tau' \mid \xi'_{0} \sim \pi'_0}+\sum^{k-1}_{t=0}\rho_t.$$
        Then, for $k=K$, we have
        \begin{align*}
                & \expect{\tau^{K} \mid\xi^{K}_0\sim \pi_{0}} \\
                &= K - \sum\nolimits_{t=0}^{K-1} \pi_{t}(\mathcal{X}^*) + \sum\limits_{x \in \mathcal{X},y\in \mathcal{Y}} \pi_{K-1}(x) P(\xi^{K}_{K}\in\phi^{-1}(y)  \mid \xi^{K}_{K-1}=x) \expect{\tau^{K} \mid \xi^{K}_{K} = y}\\
                &= K - \sum\nolimits_{t=0}^{K-1} \pi_{t}(\mathcal{X}^*)+\sum\limits_{x \in \mathcal{X},y \in \mathcal{Y}} \pi_{K-1}(x) P(\xi^{K}_{K}\in\phi^{-1}(y)  \mid \xi^{K}_{K-1}=x) \expect{\tau' \mid \xi'_{K} = y}\\
                & \leq K - \sum\nolimits_{t=0}^{K-1} \pi_{t}(\mathcal{X}^*)+\rho_{K-1}+\sum\limits_{y_1,y_2 \in \mathcal{Y}} \pi'_{K-1}(y_1) P(\xi'_{K}=y_2  \mid \xi'_{K-1}=y_1) \expect{\tau' \mid \xi'_{K} = y_2} \\
                &=K - \sum\nolimits_{t=0}^{K-2} \pi_{t}(\mathcal{X}^*)-\pi'_{K-1}(\mathcal{Y}^*)+\rho_{K-1}\\
                &\quad +\sum\limits_{y_1,y_2 \in \mathcal{Y}} \pi'_{K-1}(y_1) P(\xi'_{K}=y_2  \mid \xi'_{K-1}=y_1) \expect{\tau' \mid \xi'_{K} = y_2} \\
                & = \expect{\tau^{K-1} \mid \xi^{K-1}_{0}\sim \pi_{0}}+\rho_{K-1}\\
                & \leq \expect{\tau' \mid \xi'_{0}\sim \pi'_{0}}+\sum^{K-2}_{t=0} \rho_t+\rho_{K-1}\\
                & = \expect{\tau' \mid \xi'_{0}\sim\pi'_{0}}+\sum\nolimits^{K-1}_{t=0}\rho_t,
        \end{align*}
        where the first and fourth equalities are by Lemma \ref{lem_tau}, the second is by Eq.\ref{GMCST_proof_eq_1}, the third is by Eq.\ref{GMCST_proof_eq_2}, the first inequality is by Eq.\ref{GMCST_condition}, and the second inequality is by inductive hypothesis.\\
        {\bf (c) Conclusion} from (a) and (b), it holds
        $$
                \expect{\tau^\infty \mid \xi_{0}^\infty \sim \pi_0} \leq  \expect{ \tau' \mid \xi'_{0} \sim \pi'_0} + \sum\nolimits^{\infty}_{t=0}\rho_t.
        $$
        Finally, since $\expect{\tau \mid \xi_{0} \sim \pi_0}$ is finite, $\expect{\tau^\infty \mid \xi_{0}^\infty \sim \pi_0}=\expect{\tau \mid \xi_{0} \sim \pi_0}$, and since $\{\rho_t\}_{t=0}^{+\infty}$ is a series of numbers whose sum converges to $\rho$, $\sum^{\infty}_{t=0} \rho_t=\rho<\infty$, thus $\expect{\tau \mid \xi_{0} \sim \pi_0} \leq \expect{\tau' \mid \xi'_{0} \sim \pi'_0} + \rho$.
\end{proof}

The GMCST is helpful to our analysis of EAs with crossover operators. Note that, in Eq.\ref{GMCST_condition}, there is no need of calculating the term $\expect{\tau' \mid \xi'_{t+1} = y}$, which is the long-term behavior of the chain $\xi$, but only need to compare the one-step transition behavior of the two chains, i.e.,  $P(\xi_{t+1}\in\phi^{-1}(y) \mid \xi_t=x)$ and $P(\xi'_{t+1}=y_2 \mid \xi'_t=y_1)$. Therefore, to analyze the running time of an EA with crossover operators, we can let the Markov chain $\xi$ model this EA, then let the Markov chain $\xi'$ modeling an simple (even unrealistic) algorithm that is easy to be analyzed. By the GMCST, we accomplish the analysis by comparing the one-step transition probabilities and deriving the conditional first hitting time of the simple algorithm.
%

The only requirement of the mapping function $\phi$ is that it maps optimal states to optimal states in the two state spaces. Therefore, it is possible that $\phi^{-1}(y)$ is undefined for some $y \in \mathcal{Y}$. In this case, we simply treat $\pi_t(\phi^{-1}(y))=0$.

\section{Running Time Analysis of Crossover-Enabled EAs}\label{sec:runtime}

In the first aspect, we address how to analyze running time of crossover-enabled EAs. We derive upper and lower bounds of the running time using GMCST in this section.

\subsection{On LeadingOnes problem}

According to the GMCST, we need to compare the target EA with a reference algorithm. We choose the reference algorithm to be the (1+1$_>$)-EA with one-bit mutation running on the LeadingOnes problem, for which we have the following property.
\begin{proposition}\label{1+1_EA'_LeadingOnes}
Let $\xi'$ be the Markov chain modeling the (1+1$_>$)-EA with one-bit mutation running on the LeadingOnes problem. It then satisfies that $\forall s \in \{0,1\}^n:\expect{\tau'|\xi'_t=s}=n(n-\|s\|)$, where $\|\cdot\|$ is the 1-norm, i.e., the number of 1's.
\end{proposition}

Denote $\mathbb{E}(j)$ as the CFHT $\expect{\tau'|\xi'_t=s}$ with $\|s\|=n-j$, i.e., $\mathbb{E}(j)=n\cdot j$. We then derive the running times of EAs in the following theorems. Also remember that $LO(s)$ is the number of leading ones of a solution $s$.

\begin{theorem}\label{them:nocross_LO_upper}
For the LeadingOnes problem, the EFHT of (2+2)-EA with one-point or uniform crossover and one-bit or bitwise mutation is not larger than $\frac{en}{(1-p_c)2^{2n}} \sum^n_{j=1} (2^n-2^{j-1})^2$, i.e., $O(\frac{n^2}{1-p_c})$.
\end{theorem}
\begin{myproof}
We use (1+1$_>$)-EA with one-bit mutation running on the LeadingOnes problem as the reference algorithm. Denote $\xi\in \mathcal{X}$ as the Markov chain modeling the (2+2)-EA, and $\xi'\in \mathcal{Y}$ as the Markov chain modeling the (1+1$_>$)-EA.

We construct a mapping $\phi: \mathcal{X} \rightarrow \mathcal{Y}$ as $\forall x \in \mathcal{X}: \phi(x)= 1^{\max\{LO(y)\mid y \in x\}}0^{n-\max\{LO(y)\mid y \in x\}}$. It is easy to see that the mapping $\phi$ satisfies that $\phi(x) \in \mathcal{Y}^*$ if and only if $x \in \mathcal{X}^*$.

Then, we investigate the formula, which is the condition of the GMCST,
\begin{align} \label{GMCST_condition_formula}
            & \sum\limits_{\mathclap{x\in \mathcal{X}, y \in \mathcal{Y}}} \pi_t(x) P(\xi_{t+1}\in\phi^{-1}(y) \mid \xi_t=x) \expect{\tau' \mid \xi'_{t+1} = y} \\
            & -\sum\limits_{\mathclap{y_1,y_2\in \mathcal{Y}}} \pi'_t(y_1) P(\xi'_{t+1}=y_2 \mid \xi'_t=y_1) \expect{\tau' \mid \xi'_{t+1}   =y_2}.
\end{align}
 For any $x \in \mathcal{X}^*$, since $\phi(x) \in \mathcal{Y}^*$, we have
\begin{align}
 &\sum\limits_{y \in \mathcal{Y}} P(\xi_{t+1}\in\phi^{-1}(y) \mid \xi_t=x) \expect{\tau' \mid \xi'_{t+1} = y} \\
 &= \sum\limits_{y\in \mathcal{Y}}P(\xi'_{t+1}=y \mid \xi'_t=f(x)) \expect{\tau' \mid\xi'_{t+1}=y}=0.
\end{align}
For any $x \notin \mathcal{X}^*$ with $\max_{y \in x} LO(y)=n-j(0 < j \leq n )$, when using one-bit mutation on $x$, the next population $x'$ satisfies that $\max_{y \in x'} LO(y)>n-j$ with probability at least $\frac{1}{n}$, since the best parent solution can increase the number of leading ones with probability $\frac{1}{n}$. Then, we have
\begin{align}
&\sum\limits_{y \in \mathcal{Y}} P(\xi_{t+1}\in\phi^{-1}(y) \mid \xi_t=x) \expect{\tau' \mid \xi'_{t+1} = y} \\
& \leq \frac{1}{n}\mathbb{E}(j-1)+(1-\frac{1}{n})\mathbb{E}(j);
\end{align}
when using bitwise mutation on $x$, the next population $x'$ satisfies that $\max_{y \in x'} LO(y)>n-j$ with probability at least $(\frac{n-1}{n})^{n-j}\frac{1}{n} > \frac{1}{en}$, since the best parent solution can increase the number of leading ones with probability $(\frac{n-1}{n})^{n-j}\frac{1}{n} $. Then, we have
\begin{align}
&\sum\limits_{y \in \mathcal{Y}} P(\xi_{t+1}\in\phi^{-1}(y) \mid \xi_t=x) \expect{\tau' \mid \xi'_{t+1} = y} \\
& \leq \frac{1}{en}\mathbb{E}(j-1)+(1-\frac{1}{en})\mathbb{E}(j).
\end{align}
Since $\frac{1}{n}\mathbb{E}(j-1)+(1-\frac{1}{n})\mathbb{E}(j)<\frac{1}{en}\mathbb{E}(j-1)+(1-\frac{1}{en})\mathbb{E}(j)$, we use the bound of the bitwise mutation in the following derivation for unifying the result for the two mutation operators.

When using one-point or uniform crossover on $x$, considering the worst case, the next population $x'$
satisfies that $\max_{y \in x'} LO(y)\geq n-j$, since the (2+2)-EA
keeps the best so-far solution. Then, we have
\begin{align}
&\sum\limits_{y \in \mathcal{Y}} P(\xi_{t+1}\in\phi^{-1}(y) \mid
\xi_t=x) \expect{\tau' \mid \xi'_{t+1} = y} \leq \mathbb{E}(j).
\end{align}
Since $f(x)$ is the solution $1^{n-j}0^{j}$, we have
\begin{align}
&\sum\limits_{y\in \mathcal{Y}}P(\xi'_{t+1}=y \mid \xi'_t=\phi(x)) \expect{\tau' \mid\xi'_{t+1}=y}\\
&= \frac{1}{n} \mathbb{E}(j-1) +(1-\frac{1}{n})\mathbb{E}(j).
\end{align}

Combining the mutation with crossover,
\begin{align}
&\forall x \notin \mathcal{X}^*:\sum\limits_{y \in \mathcal{Y}} P(\xi_{t+1}\in\phi^{-1}(y) \mid \xi_t=x) \expect{\tau' \mid \xi'_{t+1} = y} \\
& -\sum\limits_{y\in \mathcal{Y}}P(\xi'_{t+1}=y \mid \xi'_t=\phi(x)) \expect{\tau' \mid\xi'_{t+1}=y}\\
&\leq
p_c\mathbb{E}(j)+(1-p_c)(\frac{1}{en}\mathbb{E}(j-1)+(1-\frac{1}{en})\mathbb{E}(j))-(\frac{1}{n}
\mathbb{E}(j-1) +(1-\frac{1}{n})\mathbb{E}(j))\\
& = 1-\frac{1}{e}(1-p_c).
\end{align}
Then, considering all $x$, we have
\begin{align}
& \forall t: \sum\limits_{\mathclap{x\in \mathcal{X}, y \in
\mathcal{Y}}} \pi_t(x) P(\xi_{t+1}\in\phi^{-1}(y) \mid \xi_t=x)
\expect{\tau' \mid \xi'_{t+1} = y}  \\
& \leq \sum\limits_{\mathclap{x \in \mathcal{X},y\in \mathcal{Y}}}
\pi_t(x) P(\xi'_{t+1}=y \mid \xi'_t=\phi(x)) \expect{\tau' \mid
\xi'_{t+1}=y} +(1-\frac{1}{e}(1-p_c))(1-\pi_t(\mathcal{X^*}))\\
&=\sum\limits_{\mathclap{y_1,y_2\in \mathcal{Y}}} \pi'_t(y_1)
P(\xi'_{t+1}=y_2 \mid \xi'_t=y_1) \expect{\tau' \mid
\xi'_{t+1}=y_2}+ (1-\frac{1}{e}(1-p_c))(1-\pi_t(\mathcal{X^*})).
\end{align}
By GMCST, we get $\expect{\tau \mid \xi_{0} \sim \pi_0} \leq \expect{\tau' \mid \xi'_{0} \sim \pi'_0}+(1-\frac{1}{e}(1-p_c)) \sum^{\infty}_{t=0}(1-\pi_t(\mathcal{X^*})).$ Since $\sum^{\infty}_{t=0}(1-\pi_t(\mathcal{X^*}))=\expect{\tau \mid \xi_{0} \sim \pi_0}$, we have $\expect{\tau \mid \xi_{0} \sim \pi_0} \leq \frac{e}{1-p_c}\expect{\tau' \mid \xi'_{0} \sim \pi'_0}$.

Then, we need to derive the DCFHT $\expect{\tau'\mid \xi'_{0} \sim \pi'_0}$. Since $\pi_0$ is the uniform distribution over the population space $\{\{0,1\}^n\}^2$, we have
\begin{align}
& \forall 1 \leq j \leq n: \pi'_0(1^{n-j}0^{j})=\pi_0(\phi^{-1}(1^{n-j}0^{j}))\\
&=\pi_0(\{x \in \mathcal{X} \mid \max_{y \in x}
LO(y)=n-j\})\\
&=\frac{(2^n-2^{j-1})^2-(2^n-2^j)^2}{2^{2n}},
\end{align}
$\pi'_0(1^n)=\frac{(2^n)^2-(2^n-1)^2}{2^{2n}}$, and $\forall y \notin \{1^j0^{n-j}\mid 0 \leq j\leq n\}:\pi'_0(y)=0$, where the term $2^n-2^{j-1}$ is the number of the populations with not larger than $n-j$ number of leading ones. Thus, we have
$$
\expect{\tau' \mid \xi'_{0} \sim \pi'_0}=\sum^n_{j=1}
\pi'_0(1^{n-j}0^j) jn=\frac{n}{2^{2n}} \sum^n_{j=1} (2^n-2^{j-1})^2,
$$
and $\expect{\tau \mid \xi_{0} \sim \pi_0} \leq \frac{en}{(1-p_c)2^{2n}} \sum^n_{j=1} (2^n-2^{j-1})^2$.
\end{myproof}

\begin{theorem}\label{them:cross_LO_lower}
For the LeadingOnes problem, the EFHT of (2+2)-EA with one-point crossover and one-bit or bitwise mutation is not less than $\frac{n}{(5-2p_c)2^{2n}} \sum^n_{j=1} (2^n-2^{j-1})^2$, i.e., $\Omega(\frac{n^2}{5-2p_c})$.
\end{theorem}

\begin{proof}
The proof is similar to that of Theorem \ref{them:nocross_LO_upper}, except we derive inequalities inversely. We also use (1+1$_>$)-EA with one-bit mutation running on the LeadingOnes problem as the reference algorithm, denote $\xi\in \mathcal{X}$ as the Markov chain modeling the (2+2)-EA and $\xi'\in\mathcal{Y}$ as the Markov chain modeling the (1+1$_>$)-EA, and use the mapping $\phi(x)= 1^{\max\{LO(y)\mid y \in x\}}0^{n-\max\{LO(y)\mid y \in x\}}$.

When the current population $x=\{s_1,s_2\}$, we denote $LO_1$ and $LO_2$ be the number of leading ones of $s_1$ and $s_2$ respectively. For some $i$ and $j$ such that $0 < j\leq i \leq n$, we denote $\mathcal{X}_{i,j}$ as the population space consisting of all the populations with $LO_1=n-i$ and $LO_2=n-j$. We assume $LO_1=n-i \leq LO_2=n-j$. It is not difficult to see that, for a current solution, the bits after the first 0 bit position are randomly distributed, i.e., being 1 or 0 with 0.5 probability, since these bits are randomly initialized and the selection operator does not bias the distribution of them until now.

When using one-bit mutation on $x\in \mathcal{X}_{i,j}$, with probability $(1-\frac{1}{n})^2$, the first 0 bits of the two current solutions are not flipped, thus do not make any progress; with probability $(1-\frac{1}{n})\frac{1}{n}$ the first 0 bit of $s_2$ is flipped but the first 0 bit of $s_1$ is not, in which case we consider the probability that there are $k\in\{1,\ldots, j-1\}$ following 1 bits so that the progress is $k$; with probability $(1-\frac{1}{n})\frac{1}{n}$ the first 0 bit of $s_1$ is flipped but the first 0 bit of $s_2$ is not, in which case we consider the probability that there are $k'\in\{i-j+1,\ldots,i-1\}$ following 1 bits so that the progress is $k'$. We therefore obtain that
\begin{align}
&\sum\limits_{x \in \mathcal{X}_{i,j}, y \in \mathcal{Y}} \pi_t(x)P(\xi_{t+1}\in\phi^{-1}(y) \mid \xi_t=x) \expect{\tau' \mid \xi'_{t+1} = y} \\
& \geq \pi_t(\mathcal{X}_{i,j})\Bigg(
(1-\frac{1}{n})^2\mathbb{E}(j) + \frac{1}{n}(1-\frac{1}{n})\Big( \sum_{k=1}^{j-1} \frac{\mathbb{E}(j-k)}{2^k} + \frac{\mathbb{E}(0)}{2^{j-1}} \Big) \\
&\quad +\frac{1}{n}(1-\frac{1}{n})\Big( (1-\frac{1}{2^{i-j}})\mathbb{E}(j) + \sum_{k'=i-j+1}^{i-1} \frac{\mathbb{E}(i-k')}{2^{k'}} +\frac{\mathbb{E}(0)}{2^{i-1}} \Big) \Bigg)\\
&=\pi_t(\mathcal{X}_{i,j})\Big(\mathbb{E}(j)-(1-\frac{1}{n})(2-\frac{1}{2^{j-1}})(1+\frac{1}{2^{i-j}})-\frac{j}{n}\Big);
\end{align}
when using bitwise mutation, the situation is the same as that using one-bit mutation, except the probability that neither of the solutions have their first 0 bit flipped or the leading 1 bits are destroyed is $((1-\frac{1}{n})+(1-(\frac{n-1}{n})^{n-j})\frac{1}{n})\cdot((1-\frac{1}{n})+(1-(\frac{n-1}{n})^{n-i})\frac{1}{n}) < (1-\frac{1}{en})^2$, the probability that the first 0 bit of $s_2$ is flipped but that of $s_1$ is not is $(1-\frac{1}{n})^{1+n-j}\frac{1}{n} > \frac{n-1}{n}\frac{1}{en}$ and the probability that the first 0 bit of $s_1$ is flipped but that of $s_2$ is not is $(1-\frac{1}{n})^{1+n-i}\frac{1}{n} > \frac{n-1}{n}\frac{1}{en}$ in order to keep the leading 1 bits. So we have
\begin{align}
&\sum\limits_{x \in \mathcal{X}_{i,j}, y \in \mathcal{Y}} \pi_t(x)P(\xi_{t+1}\in\phi^{-1}(y) \mid \xi_t=x) \expect{\tau' \mid \xi'_{t+1} = y} \\
& \geq \pi_t(\mathcal{X}_{i,j})\Bigg(
(1-\frac{1}{en})^2\mathbb{E}(j) + \frac{1}{en}(1-\frac{1}{n})\Big( \sum_{k=1}^{j-1} \frac{\mathbb{E}(j-k)}{2^k} + \frac{\mathbb{E}(0)}{2^{j-1}} \Big) \\
&\quad +\frac{1}{en}(1-\frac{1}{n})\Big( (1-\frac{1}{2^{i-j}})\mathbb{E}(j) + \sum_{k'=i-j+1}^{i-1} \frac{\mathbb{E}(i-k')}{2^{k'}} +\frac{\mathbb{E}(0)}{2^{i-1}} \Big) \Bigg)\\
&=\pi_t(\mathcal{X}_{i,j})\Big(\mathbb{E}(j)-\frac{1}{e}(1-\frac{1}{n})(2-\frac{1}{2^{j-1}})(1+\frac{1}{2^{i-j}})
- (2-\frac{1}{e})\frac{1}{e}\frac{j}{n}\Big).
\end{align}
Since $\mathbb{E}(j)-\frac{1}{e}(1-\frac{1}{n})(2-\frac{1}{2^{j-1}})(1+\frac{1}{2^{i-j}})- (2-\frac{1}{e})\frac{1}{e}\frac{j}{n} >
\mathbb{E}(j)-(1-\frac{1}{n})(2-\frac{1}{2^{j-1}})(1+\frac{1}{2^{i-j}})-\frac{j}{n}$, we will use the bound for one-bit mutation in the following derivation for unifying the results for mutation operators.

When using one-point crossover, progress can only be achieved if the crossover point is between the first 0 bits of the two solutions, so that the leading 1 bits from $s_2$ joint the tailing bits from $s_2$ to make a better solution. In an optimistic case, we assume $LO_1\neq LO_2$ such that the progress can happen. Recall that $LO_1=n-i < LO_2=n-j$, there are $i-j$ positions where crossover is able to make progress, each of these positions is taken with probability $\frac{1}{n-1}$. For each of these positions, we then count the probability that the progress of $k\in\{0,\ldots,j\}$ bits is achieved.
\begin{align}
&\sum\limits_{x \in \mathcal{X}_{i,j}, y \in \mathcal{Y}} P(\xi_{t+1}\in\phi^{-1}(y) \mid \xi_t=x) \expect{\tau' \mid \xi'_{t+1} = y} \\
& =\pi_t(\mathcal{X}_{i,j})\Bigg(\frac{LO_1}{n-1} \mathbb{E}(j)+\frac{n-1-LO_2}{n-1} \mathbb{E}(j)\\
& \quad + \sum_{q=1}^{i-j} \frac{1}{n-1}\left((1-\frac{1}{2^q})\mathbb{E}(j)+\sum_{k=1}^{j-1}\frac{\mathbb{E}(j-k)}{2^{q+k}}+\frac{1}{2^{j-1+q}}\mathbb{E}(0)\right)\Bigg)\\
&=\pi_t(\mathcal{X}_{i,j})\left(\mathbb{E}(j)-\frac{n}{n-1}(2-\frac{1}{2^{j-1}})(1-\frac{1}{2^{i-j}})\right).
\end{align}

 Since $\phi(x)$ is the solution $1^{n-j}0^{j}$, we have
\begin{align}
&\sum\limits_{y\in \mathcal{Y}}P(\xi'_{t+1}=y \mid \xi'_t=f(x)) \expect{\tau' \mid\xi'_{t+1}=y}\\
&= \frac{1}{n} \mathbb{E}(j-1) +(1-\frac{1}{n})\mathbb{E}(j) = \mathbb{E}(j)-1.
\end{align}
Thus, we have
\begin{aligna}
&\sum\limits_{x \in \mathcal{X}_{i,j}, y \in \mathcal{Y}} P(\xi_{t+1}\in\phi^{-1}(y) \mid \xi_t=x) \expect{\tau' \mid \xi'_{t+1} = y} \\
& -\sum\limits_{x \in \mathcal{X}_{i,j}, y\in \mathcal{Y}}P(\xi'_{t+1}=y \mid \xi'_t=\phi(x)) \expect{\tau'\mid\xi'_{t+1}=y}\\
&\geq p_c\pi_t(\mathcal{X}_{i,j})(\mathbb{E}(j)-\frac{n}{n-1}(1-\frac{1}{2^{i-j}})(2-\frac{1}{2^{j-1}}))\\
&\quad +(1-p_c)\pi_t(\mathcal{X}_{i,j})(\mathbb{E}(j) - (1-\frac{1}{n})(2-\frac{1}{2^{j-1}})(1+\frac{1}{2^{i-j}})-\frac{j}{n})\\
& \quad -\pi_t(\mathcal{X}_{i,j})(\mathbb{E}(j)-1)\\
&\geq\pi_t(\mathcal{X}_{i,j})(2p_c-4).
\end{aligna}
Then, we have
\begin{align}
& \forall t: \sum\limits_{\mathclap{x\in \mathcal{X}, y \in
\mathcal{Y}}} \pi_t(x) P(\xi_{t+1}\in\phi^{-1}(y) \mid \xi_t=x)
\expect{\tau' \mid \xi'_{t+1} = y}  \\
& \geq \sum\limits_{\mathclap{x \in \mathcal{X},y\in \mathcal{Y}}}
\pi_t(x) P(\xi'_{t+1}=y \mid \xi'_t=\phi(x)) \expect{\tau' \mid
\xi'_{t+1}=y} + (2p_c-4)(1-\pi_t(\mathcal{X^*}))\\
&=\sum\limits_{\mathclap{y_1,y_2\in \mathcal{Y}}} \pi'_t(y_1)
P(\xi'_{t+1}=y_2 \mid \xi'_t=y_1) \expect{\tau' \mid
\xi'_{t+1}=y_2}+ (2p_c-4)(1-\pi_t(\mathcal{X^*})).
\end{align}
By GMCST, we get $\expect{\tau \mid \xi_{0} \sim \pi_0} \geq
\expect{\tau' \mid \xi'_{0} \sim \pi'_0}+(2p_c-4)
\sum^{\infty}_{t=0}(1-\pi_t(\mathcal{X^*})).$ Thus, $\expect{\tau
\mid \xi_{0} \sim \pi_0} \geq \frac{1}{5-2p_c}\expect{\tau' \mid
\xi'_{0} \sim \pi'_0}$.

From the proof of Theorem \ref{them:nocross_LO_upper}, we know $
\expect{\tau' \mid \xi'_{0} \sim \pi'_0}=\frac{n}{2^{2n}}
\sum^n_{j=1} (2^n-2^{j-1})^2$. Thus, we get $\expect{\tau \mid
\xi_{0} \sim \pi_0} \geq \frac{n}{(5-2p_c)2^{2n}} \sum^n_{j=1}
(2^n-2^{j-1})^2$.
\end{proof}

\begin{theorem}\label{LO_lower_uniform}
For the LeadingOnes problem, the EFHT of (2+2)-EA with uniform
crossover and one-bit or bitwise mutation is not less than $\frac{n}{(2n-1)2^{2n}}
\sum^n_{j=1} (2^n-2^{j-1})^2$, i.e., $\Omega(n)$.
\end{theorem}

\begin{proof}
We use (1+1$_>$)-EA with one-bit mutation running on the LeadingOnes problem as the reference algorithm. Denote $\xi\in \mathcal{X}$ as the Markov chain modeling the (2+2)-EA and $\xi'\in\mathcal{Y}$ as the Markov chain modeling the (1+1$_>$)-EA, and use the mapping $\phi: \mathcal{X} \rightarrow \mathcal{Y}$ which satisfies that $\forall x \in \mathcal{X}: \phi(x)= 1^{\max\{LO(y)\mid y \in x\}}0^{n-\max\{LO(y)\mid y \in x\}}$.

For any $x \notin \mathcal{X}^*$, suppose that $\max_{y \in x}
LO(y)=n-j(0 < j \leq n )$. If using one-bit mutation or bitwise mutation on $x$, the
next population $x'$ satisfies that $\max_{y \in x'} LO(y)=n-j$ with
probability $(1-\frac{1}{n})^2$ when the first 0 bits of the two
solutions do not flip. Then, we have
\begin{align}
&\sum\limits_{y \in \mathcal{Y}} P(\xi_{t+1}\in\phi^{-1}(y) \mid
\xi_t=x) \expect{\tau' \mid \xi'_{t+1} = y}  \geq (1-\frac{1}{n})^2
\mathbb{E}(j).
\end{align}
If using uniform crossover on $x$, the next population $x'$
satisfies that $\max_{y \in x'} LO(y)= n-j$ with probability
$(1-\frac{1}{n})^2$ when the bits on the position $LO2+1$ and $LO1+1$ do not
exchange. Then, we have
\begin{align}
&\sum\limits_{y \in \mathcal{Y}} P(\xi_{t+1}\in\phi^{-1}(y) \mid \xi_t=x)
\expect{\tau' \mid \xi'_{t+1} = y} \geq
(1-\frac{1}{n})^2\mathbb{E}(j).
\end{align}
Since $\phi(x)$ is the solution $1^{n-j}0^{j}$, we have
\begin{align}
&\sum\limits_{y\in \mathcal{Y}}P(\xi'_{t+1}=y \mid \xi'_t=\phi(x)) \expect{\tau' \mid\xi'_{t+1}=y}\\
&= \frac{1}{n} \mathbb{E}(j-1) +(1-\frac{1}{n})\mathbb{E}(j).
\end{align}
Thus, we have
\begin{align}
&\forall x \notin \mathcal{X}^*:\sum\limits_{y \in \mathcal{Y}} P(\xi_{t+1}\in\phi^{-1}(y) \mid \xi_t=x) \expect{\tau' \mid \xi'_{t+1} = y} \\
& -\sum\limits_{y\in \mathcal{Y}}P(\xi'_{t+1}=y \mid \xi'_t=\phi(x)) \expect{\tau' \mid\xi'_{t+1}=y}\\
&\geq
p_c(1-\frac{1}{n})^2\mathbb{E}(j)+(1-p_c)(1-\frac{1}{n})^2\mathbb{E}(j)-(\frac{1}{n}
\mathbb{E}(j-1) +(1-\frac{1}{n})\mathbb{E}(j))\\
&= 1-2j+\frac{j}{n} \geq 2-2n.
\end{align}
Then, we have
\begin{align}
& \forall t: \sum\limits_{\mathclap{x\in \mathcal{X}, y \in \mathcal{Y}}} \pi_t(x) P(\xi_{t+1}\in\phi^{-1}(y) \mid \xi_t=x) \expect{\tau' \mid \xi'_{t+1} = y}  \\
& \geq \sum\limits_{\mathclap{x \in \mathcal{X},y\in \mathcal{Y}}}
\pi_t(x) P(\xi'_{t+1}=y \mid \xi'_t=\phi(x)) \expect{\tau' \mid \xi'_{t+1}=y} + (2-2n)(1-\pi_t(\mathcal{X^*}))\\
&=\sum\limits_{\mathclap{y_1,y_2\in \mathcal{Y}}} \pi'_t(y_1) P(\xi'_{t+1}=y_2 \mid \xi'_t=y_1) \expect{\tau' \mid \xi'_{t+1}=y_2}+ (2-2n)(1-\pi_t(\mathcal{X^*})).
\end{align}
By GMCST, we get $\expect{\tau \mid \xi_{0} \sim \pi_0} \geq
\expect{\tau' \mid \xi'_{0} \sim \pi'_0}+(2-2n)
\sum^{\infty}_{t=0}(1-\pi_t(\mathcal{X^*})).$ Thus, $\expect{\tau
\mid \xi_{0} \sim \pi_0} \geq \frac{1}{2n-1}\expect{\tau' \mid
\xi'_{0} \sim \pi'_0}$.

From the proof of Theorem \ref{them:nocross_LO_upper} we know
$
\expect{\tau' \mid \xi'_{0} \sim \pi'_0} =\frac{n}{2^{2n}}
\sum^n_{j=1} (2^n-2^{j-1})^2.
$
Thus, we get $\expect{\tau \mid \xi_{0} \sim \pi_0} \geq
\frac{n}{(2n-1)2^{2n}} \sum^n_{j=1} (2^n-2^{j-1})^2$.
\end{proof}

\subsection{On OneMax problem}

\begin{proposition}\label{1+1_EA_OneMax}
Let $\xi'$ be the Markov chain modeling the (1+1)-EA with one-bit mutation running on the OneMax problem. It then satisfies that $\forall s \in \{0,1\}^n:\expect{\tau'|\xi'_t=s}=nH_{n-\|s\|}$, where $\|\cdot\|$ is the 1-norm, i.e., the number of 1's, and $H_k=\sum_{i=1}^k \frac{1}{i}$ is the $k$-th harmonic number.
\end{proposition}

We reuse the notation $\mathbb{E}(j)$ in this subsection as the CFHT of the (1+1)-EA with one-bit mutation running on the OneMax problem given $\|s\|=n-j$, so that $\mathbb{E}(j)=nH_j$.

\begin{theorem}\label{2+2EA_uniform_onemax_upper}
For the OneMax problem, the EFHT of (2+2)-EA with one-point or uniform crossover and one-bit or bitwise mutation is not larger than $\frac{en}{(1-p_c)2^{2n}} \sum^n_{j=1} \frac{1}{j}(\sum^n_{k=j} \binom{n}{k})^2$, i.e., $O(\frac{n\log n}{1-p_c})$.
\end{theorem}

\begin{proof}
We use the (1+1)-EA with one-bit mutation running on the OneMax problem as the reference algorithm. We use the mapping function $\phi(x)= \arg \max_{y \in x} \|y\|$. It is easy to see that the mapping satisfies that $\phi(x) \in \mathcal{Y}^*$ if and only if $x \in \mathcal{X}^*$.

For any $x \notin \mathcal{X}^*$ with $\max_{y \in x} \|y\|=n-j(0 < j \leq n )$. When using one-bit mutation, since the probability of flipping a 0 bit of $j$ 0's in the solution is $\frac{j}{n}$, we have
\begin{align}
&\sum\limits_{y \in \mathcal{Y}} P(\xi_{t+1}\in\phi^{-1}(y) \mid \xi_t=x) \expect{\tau' \mid \xi'_{t+1} = y} \\
& \leq \frac{j}{n} \cexpect{j-1} +(1-\frac{j}{n})\cexpect{j};
\end{align}
when using bitwise mutation, the probability that the solution will have less 0's after the mutation is at least $\frac{j}{n}(\frac{n-1}{n})^{n-j} > \frac{j}{en} $, thus
\begin{align}
&\sum\limits_{y \in \mathcal{Y}} P(\xi_{t+1}\in\phi^{-1}(y) \mid \xi_t=x) \expect{\tau' \mid \xi'_{t+1} = y} \\
& \leq \frac{j}{en} \cexpect{j-1} +(1-\frac{j}{en})\cexpect{j}.
\end{align}
Since $\frac{j}{en} \cexpect{j-1} +(1-\frac{j}{en})\cexpect{j}>\frac{j}{n} \cexpect{j-1} +(1-\frac{j}{n})\cexpect{j}$, we use the bound for the bitwise mutation in the following derivation for unifying the results for the two mutation operators.\\
When using one-point or uniform crossover, considering the worst case and that the (2+2)-EA keeps the best so-far
solution, we have
\begin{align}
&\sum\limits_{y \in \mathcal{Y}} P(\xi_{t+1}\in\phi^{-1}(y) \mid
\xi_t=x) \expect{\tau' \mid \xi'_{t+1} = y} \leq \cexpect{j}.
\end{align}

 Since $\phi(x)$ is a solution with $j$ 0's, we have
\begin{align}\label{N+N-EA_one_onemax_onestep2}
&\sum\limits_{y\in \mathcal{Y}}P(\xi'_{t+1}=y \mid \xi'_t=\phi(x)) \expect{\tau' \mid\xi'_{t+1}=y}\\
&= \frac{j}{n} \cexpect{j-1} +(1-\frac{j}{n})\cexpect{j}.
\end{align}

Combining the mutation and the crossover operator,
\begin{align}
      &\forall x \notin \mathcal{X}^*:\sum\limits_{y \in \mathcal{Y}} P(\xi_{t+1}\in\phi^{-1}(y) \mid \xi_t=x) \expect{\tau' \mid \xi'_{t+1} = y} \\
      & -\sum\limits_{y\in \mathcal{Y}}P(\xi'_{t+1}=y \mid \xi'_t=\phi(x))\expect{\tau'\mid\xi'_{t+1}=y}\\
      & \leq p_c\cexpect{j}+(1-p_c)(\frac{j}{en} \cexpect{j-1}
      +(1-\frac{j}{en})\cexpect{j})-(\frac{j}{n} \cexpect{j-1}
      +(1-\frac{j}{n})\cexpect{j})\\
      &=1-\frac{1}{e}(1-p_c).
\end{align}
Then, considering all $x$, we have
\begin{align}
    & \forall t: \sum\limits_{\mathclap{x\in \mathcal{X}, y \in
    \mathcal{Y}}} \pi_t(x) P(\xi_{t+1}\in\phi^{-1}(y) \mid \xi_t=x)
    \expect{\tau' \mid \xi'_{t+1} = y}  \\
    & \leq \sum\limits_{\mathclap{x \in \mathcal{X},y\in \mathcal{Y}}}
    \pi_t(x) P(\xi'_{t+1}=y \mid \xi'_t=\phi(x)) \expect{\tau' \mid
    \xi'_{t+1}=y}+(1-\frac{1}{e}(1-p_c))(1-\pi_t(\mathcal{X}^*))\\
    &=\sum\limits_{\mathclap{y_1,y_2\in \mathcal{Y}}} \pi'_t(y_1)
    P(\xi'_{t+1}=y_2 \mid \xi'_t=y_1) \expect{\tau' \mid
    \xi'_{t+1}=y_2}+(1-\frac{1}{e}(1-p_c))(1-\pi_t(\mathcal{X}^*)).
\end{align}
By GMCST, we get $\expect{\tau \mid \xi_{0} \sim \pi_0} \leq
\expect{\tau' \mid \xi'_{0} \sim
\pi'_0}+(1-\frac{1}{e}(1-p_c))\sum^{\infty}_{t=0}(1-\pi_t(\mathcal{X}^*)).$ Thus,
$\expect{\tau \mid \xi_{0} \sim \pi_0} \leq
\frac{e}{1-p_c}\expect{\tau' \mid \xi'_{0} \sim \pi'_0}$.

Then, we derive  the DCFHT $\expect{\tau' \mid \xi'_{0} \sim
\pi'_0}$. Since $\pi_0$ is the uniform distribution over the
population space $\{\{0,1\}^n\}^2$, we have
\begin{align}
& \forall 0 \leq j \leq n: \pi'_0(\{y \in \mathcal{Y}\mid
n-\|y\|=j\})=\pi_0(\{\phi^{-1}(y) \mid n-\|y\|=j\})\\
&=\pi_0(\{x \in \mathcal{X} \mid \max_{y \in x}
\|y\|=n-j\})\\
&=\frac{(\sum^n_{k=j} \binom{n}{k})^2-(\sum^n_{k=j+1}
\binom{n}{k})^2}{2^{2n}},
\end{align}
where the term $\sum^n_{k=j} \binom{n}{k}$ is the number of
populations with at least $j$ number of 0's.
 Thus, we have
$$
\expect{\tau' \mid \xi'_{0} \sim \pi'_0}=\sum^n_{j=1} \pi'_0(\{y \in
\mathcal{Y}\mid \|y\|=n-j\}) nH_j=\frac{n}{2^{2n}} \sum^n_{j=1}
\frac{1}{j}(\sum^n_{k=j} \binom{n}{k})^2,
$$
and $\expect{\tau \mid \xi_{0} \sim \pi_0} \leq
\frac{en}{(1-p_c)2^{2n}} \sum^n_{j=1} \frac{1}{j}(\sum^n_{k=j}
\binom{n}{k})^2$.
\end{proof}

\begin{theorem}\label{OneMax_lower_uniform}
For the OneMax problem, the EFHT of (2+2)-EA with uniform
crossover and one-bit or bitwise mutation is not less than $\frac{n}{(2n-1)2^{2n}}
\sum^n_{j=1} (2^n-2^{j-1})^2$, i.e., $\Omega(n)$.
\end{theorem}
\begin{proof}
Instead of using (1+1)-EA with one-bit mutation running on the OneMax problem, here we use (1+1$_>$)-EA with one-bit mutation running on the LeadingOnes problem as the reference algorithm. Denote $\xi\in \mathcal{X}$ as the Markov chain modeling the (2+2)-EA and $\xi'\in\mathcal{Y}$ as the Markov chain modeling the (1+1$_>$)-EA, and use the mapping $\phi: \mathcal{X} \rightarrow \mathcal{Y}$ which satisfies that $\forall x \in \mathcal{X}: \phi(x)= 1^{\max\{LO(y)\mid y \in x\}}0^{n-\max\{LO(y)\mid y \in x\}}$. Note that this mapping allows us to compare the two EAs running on different problems. The remaining proof is therefore identical to that of Theorem \ref{LO_lower_uniform}. We repeat the proof below.

For any $x \notin \mathcal{X}^*$, suppose that $\max_{y \in x}
LO(y)=n-j(0 < j \leq n )$. If using one-bit mutation or bitwise mutation on $x$, the
next population $x'$ satisfies that $\max_{y \in x'} LO(y)=n-j$ with
probability $(1-\frac{1}{n})^2$ when the first 0 bit of the two
solutions do not flip. Then, we have
\begin{align}
&\sum\limits_{y \in \mathcal{Y}} P(\xi_{t+1}\in\phi^{-1}(y) \mid
\xi_t=x) \expect{\tau' \mid \xi'_{t+1} = y}  \geq (1-\frac{1}{n})^2
\mathbb{E}(j).
\end{align}
If using uniform crossover on $x$, the next population $x'$
satisfies that $\max_{y \in x'} LO(y)= n-j$ with probability
$(1-\frac{1}{n})^2$ when the bits on the position $LO2+1$ and $LO1+1$ do not
exchange. Then, we have
\begin{align}
&\sum\limits_{y \in \mathcal{Y}} P(\xi_{t+1}\in\phi^{-1}(y) \mid \xi_t=x)
\expect{\tau' \mid \xi'_{t+1} = y} \geq
(1-\frac{1}{n})^2\mathbb{E}(j).
\end{align}
Since $\phi(x)$ is the solution $1^{n-j}0^{j}$, we have
\begin{align}
&\sum\limits_{y\in \mathcal{Y}}P(\xi'_{t+1}=y \mid \xi'_t=\phi(x)) \expect{\tau' \mid\xi'_{t+1}=y}\\
&= \frac{1}{n} \mathbb{E}(j-1) +(1-\frac{1}{n})\mathbb{E}(j).
\end{align}
Thus, we have
\begin{align}
&\forall x \notin \mathcal{X}^*:\sum\limits_{y \in \mathcal{Y}} P(\xi_{t+1}\in\phi^{-1}(y) \mid \xi_t=x) \expect{\tau' \mid \xi'_{t+1} = y} \\
& -\sum\limits_{y\in \mathcal{Y}}P(\xi'_{t+1}=y \mid \xi'_t=\phi(x)) \expect{\tau' \mid\xi'_{t+1}=y}\\
&\geq
p_c(1-\frac{1}{n})^2\mathbb{E}(j)+(1-p_c)(1-\frac{1}{n})^2\mathbb{E}(j)-(\frac{1}{n}
\mathbb{E}(j-1) +(1-\frac{1}{n})\mathbb{E}(j))\\
&= 1-2j+\frac{j}{n} \geq 2-2n.
\end{align}
Then, we have
\begin{align}
& \forall t: \sum\limits_{\mathclap{x\in \mathcal{X}, y \in \mathcal{Y}}} \pi_t(x) P(\xi_{t+1}\in\phi^{-1}(y) \mid \xi_t=x) \expect{\tau' \mid \xi'_{t+1} = y}  \\
& \geq \sum\limits_{\mathclap{x \in \mathcal{X},y\in \mathcal{Y}}}
\pi_t(x) P(\xi'_{t+1}=y \mid \xi'_t=\phi(x)) \expect{\tau' \mid \xi'_{t+1}=y} + (2-2n)(1-\pi_t(\mathcal{X^*}))\\
&=\sum\limits_{\mathclap{y_1,y_2\in \mathcal{Y}}} \pi'_t(y_1) P(\xi'_{t+1}=y_2 \mid \xi'_t=y_1) \expect{\tau' \mid \xi'_{t+1}=y_2}+ (2-2n)(1-\pi_t(\mathcal{X^*})).
\end{align}
By GMCST, we get $\expect{\tau \mid \xi_{0} \sim \pi_0} \geq
\expect{\tau' \mid \xi'_{0} \sim \pi'_0}+(2-2n)
\sum^{\infty}_{t=0}(1-\pi_t(\mathcal{X^*})).$ Thus, $\expect{\tau
\mid \xi_{0} \sim \pi_0} \geq \frac{1}{2n-1}\expect{\tau' \mid
\xi'_{0} \sim \pi'_0}$.

From the proof of Theorem \ref{them:nocross_LO_upper} we know
$
\expect{\tau' \mid \xi'_{0} \sim \pi'_0} =\frac{n}{2^{2n}}
\sum^n_{j=1} (2^n-2^{j-1})^2.
$
Thus, we get $\expect{\tau \mid \xi_{0} \sim \pi_0} \geq
\frac{n}{(2n-1)2^{2n}} \sum^n_{j=1} (2^n-2^{j-1})^2$.
\end{proof}

\section{Running Time Comparison of Crossover On and Off}\label{sec:compare}

In the second aspect, we would like to know if it is better to use an crossover operator in an EA. We need to compare the running time of an EA  turning crossover on and off. In this section, we study the (2:2)-EA on the model problems to address that if enabling one-bit crossover operator is good or not.

\subsection{On LeadingOnes problem}

Let the Markov chain $\xi$ model the ($2:2$)-EA with one-bit crossover and one-bit mutation running on the problem, and $\xi'$ model the ($2:2$)-EA with one-bit mutation only.

First, we analyze the one-step transition behavior of the ($2:2$)-EA with and without the crossover on the LeadingOnes problem, i.e., to figure out $P(\xi_{t+1} \mid \xi_{t})$ and $P(\xi'_{t+1} \mid \xi'_t)$.

\begin{proposition}\label{prop_transition_LO}
For any non-optimal population $x=\{s_1,s_2\}$, denote $s'_1$ and $s'_2$ as the solution which flips the first 0 bit of $s_1$ and the solution which flips the first 0 bit of $s_2$, respectively. Let $LO_1$ and $LO_2$ be the number of leading ones of $s_1$ and $s_2$, respectively. Then, for $\{\xi'_t\}^{\infty}_{t=0}$, we have
    \begin{align}
         & P(\xi'_{t+1}=\{s'_1,s_2\} \mid \xi'_t=x) =\frac{n-1}{n^2},\\
         & P(\xi'_{t+1}=\{s_1,s'_2\} \mid \xi'_t=x) =\frac{n-1}{n^2}, \\
         & P(\xi'_{t+1}=\{s'_1,s'_2\} \mid \xi'_t=x) =\frac{1}{n^2},\\
         &P(\xi'_{t+1}=x \mid \xi'_t=x) =\frac{(n-1)^2}{n^2},\\
         & P(\xi'_{t+1}\in \mathcal{X}-\{x,\{s'_1,s_2\},\{s_1,s'_2\},\{s'_1,s'_2\}\} \mid\xi'_t=x) =0;
    \end{align}
and for $\{\xi_t\}^{\infty}_{t=0}$, if $LO_1=LO_2$, we have
    \begin{align}
         & P(\xi_{t+1}=x \mid \xi_t=x) =p_c+\frac{(1-p_c)(n-1)^2}{n^2}, \\
         & \forall y \in \mathcal{X}-\{x\}: P(\xi_{t+1}=y \mid \xi_t=x) =(1-p_c)P(\xi'_{t+1}=y \mid\xi'_t=x);
    \end{align}
if $LO_1<LO_2$, we have
    \begin{align}
            & P(\xi_{t+1}=\{s'_1,s_2\} \mid \xi_t=x) =\frac{p_c}{n}+\frac{(1-p_c)(n-1)}{n^2}, \\
            & P(\xi_{t+1}=\{s_1,s'_2\} \mid \xi_t=x) =
            \begin{cases}
            \frac{p_c}{n}+\frac{(1-p_c)(n-1)}{n^2}, &\text{if } s_1(LO_2+1)=1, \\
            \frac{(1-p_c)(n-1)}{n^2}, &\text{otherwise},\\
            \end{cases}\\
            & P(\xi_{t+1}=x \mid \xi_t=x)=
            \begin{cases}
            \frac{p_c(n-2)}{n}+\frac{(1-p_c)(n-1)^2}{n^2}, &\text{if } s_1(LO_2+1)=1, \\
            \frac{p_c(n-1)}{n}+\frac{(1-p_c)(n-1)^2}{n^2}, &\text{otherwise},\\
            \end{cases}\\
            & \forall y \in \mathcal{X}-\{x,\{s'_1,s_2\},\{s_1,s'_2\}\}:
            P(\xi_{t+1}=y \mid \xi_t=x) =(1-p_c)P(\xi'_{t+1}=y \mid \xi'_t=x).
    \end{align}
\end{proposition}

For the Markov chain $\xi'$, it is obvious that the CFHT $\expect{\tau' \mid \xi'_t=\{s_1,s_2\}}$ only depends on the number of 1's of the two solutions, i.e., $\{\|s_1\|,\|s_2\|\}$, where $\|\cdot\|$ denotes the 1-norm, i.e., the number of 1 bits. Given a population $\{s_1,s_2\}$ with $\|s_1\|=n-i$ and $\|s_2\|=n-j$, we denote $\mathbb{E}(i,j)$ as the CFHT $\expect{\tau' \mid \xi'_t=\{s_1,s_2\}}$ of EA using mutation only. It is easy to see that  $\mathbb{E}(i,j)=\mathbb{E}(j,i)$.

Then, we give two propositions on the properties of the CFHT $\expect{\tau' \mid \xi'_t=\{s_1,s_2\}}$. Proposition \ref{prop_cfht1_LO} is the instantiation of Lemma \ref{lem_tau} for (2:2)-EA with mutation only on the LeadingOnes problem. Proposition \ref{prop_cfht2_LO} compares the CFHT of two adjacent populations, i.e., there is only one bit difference between the solutions of the two populations. Then Proposition \ref{prop_equalLO} deals with the probability distribution when the two solutions have the same leading ones.

\begin{proposition}\label{prop_cfht1_LO}
The CFHT of $\{\xi'_t\}_{t=0}^{+\infty}$ on the LeadingOnes problem
satisfies that
\begin{align}
& \forall i\geq 0, \mathbb{E}(i,0)=\mathbb{E}(0,i)=0;\\
&    \forall i,j \geq 1,
    \mathbb{E}(i,j)=\frac{n^2}{2n-1}+\frac{n-1}{2n-1}\mathbb{E}(i-1,j)+\frac{n-1}{2n-1}\mathbb{E}(i,j-1)+\frac{1}{2n-1}\mathbb{E}(i-1,j-1).
\end{align}
\end{proposition}

\begin{proposition}\label{prop_cfht2_LO}
    The CFHT of $\{\xi'_t\}_{t=0}^{+\infty}$ satisfies that
    \begin{align}
        & \forall i \geq 1, \delta \geq 1: \quad \mathbb{E}(i,i+\delta)-\mathbb{E}(i,i+\delta-1) \geq \frac{n}{2^{\delta+2}},\\
        & \forall i \geq 1, \delta \geq 0: \quad \mathbb{E}(i,i+\delta)-\mathbb{E}(i-1,i+\delta) \leq n-\frac{(3n-1)}{2^{\delta+3}}.
    \end{align}
\end{proposition}


\begin{proposition}\label{prop_equalLO}
 For the Markov chain $\{\xi_t\}_{t=0}^{+\infty}$, the probability that the two solutions in the population $\{s_1,s_2\}$ after $t$ steps have the same number of leading ones and both of them are not optimal satisfies that
 $$
 \pi_t(LO_1=LO_2<n) \geq
 (\frac{1}{3}-\frac{1}{3 \cdot 4^n})(p_c+(1-p_c)(1-\frac{1}{n})^2)^t,
 $$
 where $LO_1$ and $LO_2$ are the number of leading ones of $s_1$ and $s_2$, respectively.
\end{proposition}

\begin{theorem}\label{theorem_LO_compare}
 For the LeadingOnes problem, let the Markov chains $\{\xi_t\}_{t=0}^{+\infty}$ and  $\{\xi'_t\}_{t=0}^{+\infty}$ model the (2:2)-EA with one-bit mutation and one-bit crossover and that with one-bit mutation only, respectively. When $n \geq 2$, we have $\expect{\tau'}<\expect{\tau}\leq\expect{\tau'}/(1-p_c)$ and $\expect{\tau}-\expect{\tau'} \in \Omega(n\frac{p_c}{1-p_c})$.
\end{theorem}
\begin{myproof}
For any population $x=\{s_1,s_2\}$, we denote $LO_1$ and $LO_2$ as the number of leading ones of the solution $s_1$ and that of $s_2$, respectively. Without loss of generality, we suppose that $LO_1 \leq LO_2$. We denote $i$ and $j$ as the number of 0's of the solution $s_1$ and that of $s_2$, respectively, i.e., $i=n-\|s_1\|, j=n-\|s_2\|$.


We split the state space $\mathcal{X}$ into several cases in terms of the relation between $LO_1$ and $LO_2$, and then consider this formula in each case.

If $LO_2=n$, the optimal solution has been reached. We then have
\begin{align}
        & \sum\nolimits_{y \in \mathcal{X}} P(\xi_{t+1}=y \mid \xi_{t}=x) \expect{\tau' \mid \xi'_{t+1}=y}\\
        &=\sum\nolimits_{y \in \mathcal{X}} P(\xi'_{t+1}=y \mid \xi'_{t}=x)
        \expect{\tau' \mid \xi'_{t+1}=y}=0,
\end{align}
since the (2:2)-EA keeps the best so-far solution.

Otherwise, we consider two cases:\\
case 1: $LO_1=LO_2$. We have
\begin{align}
    & \sum\nolimits_{y \in \mathcal{X}} P(\xi_{t+1}=y \mid \xi_{t}=x) \expect{\tau' \mid \xi'_{t+1}=y}\\
    & =p_c\mathbb{E}(i,j)+(1-p_c)(\mathbb{E}(i,j)-1)\\
    & =p_c+\mathbb{E}(i,j)-1\\
    & =p_c+\sum\nolimits_{y \in \mathcal{X}} P(\xi'_{t+1}=y \mid \xi'_{t}=x) \expect{\tau' \mid\xi'_{t+1}=y},
\end{align}
where the first inequality is by Proposition \ref{prop_transition_LO} and Lemma \ref{lem_tau}, and the last inequality is by Lemma \ref{lem_tau}.

case 2: $LO_1 < LO_2$. It is not difficult to see that the part of a solution after the first 0 bit is randomly distributed, since the reproduction and the selection of the (2:2)-EA does not affect them. We denote $\mathcal{X}_<$ as the set of the populations in this case. Then, the probability that $\|s_1\|=n-i \wedge \|s_2\|=n-j (0<i\leq n-LO_1, 0<j\leq n-LO_2)$ is $\frac{\binom{n-LO_1-1}{i-1}}{2^{n-LO_1-1}}\cdot\frac{\binom{n-LO_2-1}{j-1}}{2^{n-LO_2-1}}\cdot\pi_t(\mathcal{X}_<)$. The set of the populations with $\|s_1\|=n-i \wedge \|s_2\|=n-j$ in this case is denoted as $\mathcal{X}_{<,i,j}$.
Then, we consider two subcases. \\
case 2a: $\min \{i,j\}>1$. Then, we have
\begin{align}
     & \sum\nolimits_{x \in  \mathcal{X}_{<,i,j}} \pi_t(x) \sum\nolimits_{y \in \mathcal{X}} P(\xi_{t+1}=y \mid \xi_{t}=x) \expect{\tau' \mid \xi'_{t+1}=y}\\
     & =\pi_t(\mathcal{X}_{<,i,j}) \sum\nolimits_{x \in  \mathcal{X}_{<,i,j}} \frac{1}{\mid\mathcal{X}_{<,i,j}\mid} \sum\nolimits_{y \in \mathcal{X}} P(\xi_{t+1}=y \mid \xi_{t}=x) \expect{\tau' \mid \xi'_{t+1}=y}\\
     & =\pi_t(\mathcal{X}_{<,i,j})(p_c(\frac{1}{n} \mathbb{E}(i-1,j)+ \frac{n-2}{n}\mathbb{E}(i,j)+\frac{1}{n}(\frac{i-1}{n-LO_1-1}\mathbb{E}(i,j)\\
     &\quad +(\frac{n-LO_1-i}{n-LO_1-1})\mathbb{E}(i,j-1)))+(1-p_c)(\mathbb{E}(i,j)-1))\\
     &=\pi_t(\mathcal{X}_{<,i,j})(\mathbb{E}(i,j)-1+p_c(\frac{1}{n^2}(\mathbb{E}(i-1,j)-\mathbb{E}(i-1,j-1))\\
     &\quad +(\frac{n-1}{n^2}-\frac{n-LO_1-i}{n(n-LO_1-1)})(\mathbb{E}(i,j)-\mathbb{E}(i,j-1))))\\
     & > \pi_t(\mathcal{X}_{<,i,j})(\mathbb{E}(i,j)-1) \\
     & =\sum\nolimits_{x \in  \mathcal{X}_{<,i,j}} \pi_t(x)\sum\nolimits_{y \in \mathcal{X}} P(\xi'_{t+1}=y \mid \xi'_{t}=x) \expect{\tau' \mid\xi'_{t+1}=y},
\end{align}
where the first equality is by the random distribution of the part of a solution after the first 0 bit, the second equality is by $P(s_1(LO_2+1)=0)=\frac{i-1}{n-LO_1-1}$, Proposition \ref{prop_transition_LO} and Lemma \ref{lem_tau}, the first inequality is by Proposition \ref{prop_cfht2_LO} and $\frac{n-LO_1-i}{n-LO_1-1} < \frac{n-1}{n}$, and the last equality is by Lemma \ref{lem_tau}.

case 2b: $\min \{i,j\}=1$. Then, there are three cases, the analysis of which below is similar to that of case 2a.\\
case 2b1: $i=j=1$. Then, we have
\begin{align}
    & \sum\nolimits_{x \in \mathcal{X}_{<,1,1}} \pi_t(x)\sum\nolimits_{y \in \mathcal{X}}P(\xi_{t+1}=y \mid \xi_{t}=x) \expect{\tau' \mid \xi'_{t+1}=y}\\
    &=\pi_t(\mathcal{X}_{<,1,1})(p_c\frac{n-2}{n}\mathbb{E}(1,1)+(1-p_c)(\mathbb{E}(1,1)-1))\\
    &=\pi_t(\mathcal{X}_{<,1,1})(\mathbb{E}(1,1)-1+p_c(1-\frac{2}{n}\mathbb{E}(1,1)))\\
    &=\pi_t(\mathcal{X}_{<,1,1})(\mathbb{E}(1,1)-1-\frac{p_c}{2n-1})(\text{since $\mathbb{E}(1,1)=\frac{n^2}{2n-1}$, easily derived from Proposition \ref{prop_cfht1_LO}})\\
    &=\sum\nolimits_{x \in \mathcal{X}_{<,1,1}} \pi_t(x)\sum\nolimits_{y \in \mathcal{X}}P(\xi'_{t+1}=y \mid \xi'_{t}=x) \expect{\tau' \mid \xi'_{t+1}=y}-\frac{p_c}{2n-1}\pi_t(\mathcal{X}_{<,1,1}).
\end{align}
case 2b2: $i>1 \wedge j=1$. We have
\begin{align}
    & \sum\nolimits_{x \in \mathcal{X}_{<,i,1}} \pi_t(x) \sum\nolimits_{y \in \mathcal{X}} P(\xi_{t+1}=y \mid \xi_{t}=x) \expect{\tau' \mid \xi'_{t+1}=y}\\
    & =\pi_t(\mathcal{X}_{<,i,1})\sum\nolimits_{x \in \mathcal{X}_{<,i,1}} \frac{1}{\mid\mathcal{X}_{<,i,1}\mid} \sum\nolimits_{y \in \mathcal{X}} P(\xi_{t+1}=y \mid \xi_{t}=x) \expect{\tau' \mid \xi'_{t+1}=y}\\
    &=\pi_t(\mathcal{X}_{<,i,1})(p_c(\frac{n-2}{n}\mathbb{E}(i,1)+\frac{1}{n}\mathbb{E}(i-1,1)+\frac{1}{n}(\frac{i-1}{n-LO_1-1}\mathbb{E}(i,1))) +(1-p_c)(\mathbb{E}(i,1)-1))\\
    &=\pi_t(\mathcal{X}_{<,i,1})(\mathbb{E}(i,1)-1+p_c(\frac{1}{n^2}\mathbb{E}(i-1,1)+\frac{in-2n+LO_1+1}{n^2(n-LO_1-1)}\mathbb{E}(i,1)))\\
    &> \sum\nolimits_{x \in\mathcal{X}_{<,i,1}}\pi_t(x)\sum\nolimits_{y\in \mathcal{X}}P(\xi'_{t+1}=y \mid \xi'_{t}=x) \expect{\tau'\mid\xi'_{t+1}=y}+\frac{p_c}{2n-1}\pi_t(\mathcal{X}_{<,i,1}).\\
    &\quad (\text{since $\mathbb{E}(i,1) \geq \frac{n^2}{2n-1}$, easily
    derived from Proposition \ref{prop_cfht1_LO}})
\end{align}
case 2b3: $i=1 \wedge j>1$. We have
\begin{align}
    & \sum\nolimits_{x \in \mathcal{X}_{<,1,j}} \pi_t(x) \sum\nolimits_{y \in \mathcal{X}} P(\xi_{t+1}=y \mid \xi_{t}=x) \expect{\tau'\mid \xi'_{t+1}=y}\\
    & =\pi_t(\mathcal{X}_{<,1,j})\sum\nolimits_{x \in \mathcal{X}_{<,1,j}} \frac{1}{\mid\mathcal{X}_{<,1,j}\mid} \sum\nolimits_{y \in \mathcal{X}} P(\xi_{t+1}=y \mid \xi_{t}=x) \expect{\tau' \mid \xi'_{t+1}=y})\\
    &=\pi_t(\mathcal{X}_{<,1,j})(p_c(\frac{n-2}{n}\mathbb{E}(1,j)+\frac{1}{n}\mathbb{E}(1,j-1))+(1-p_c)(\mathbb{E}(1,j)-1))\\
    &=\pi_t(\mathcal{X}_{<,1,j})(\mathbb{E}(1,j)-1+\frac{p_c}{n^2}(\mathbb{E}(1,j-1)-\mathbb{E}(1,j)))\\
    &\geq \pi_t(\mathcal{X}_{<,1,j})(\mathbb{E}(1,j)-1-\frac{p_c}{2n-1}) \\
    & \quad (\text{since $\mathbb{E}(1,j)-\mathbb{E}(1,j-1) \leq \frac{n^2}{2n-1}$, easily derived from Proposition \ref{prop_cfht1_LO}})\\
    &=\sum\nolimits_{x \in \mathcal{X}_{<,1,j}} \pi_t(x)\sum\nolimits_{y
    \in \mathcal{X}} P(\xi'_{t+1}=y \mid \xi'_{t}=x) \expect{\tau' \mid
    \xi'_{t+1}=y}-\frac{p_c}{2n-1}\pi_t(\mathcal{X}_{<,1,j}).
\end{align}
Combining the above three cases in the case 2b, we have
\begin{align}
    & \sum\nolimits_{x \in \mathcal{X}_{<,i,j}, \min \{i,j\}=1} \pi_t(x) \sum\nolimits_{y \in \mathcal{X}} P(\xi_{t+1}=y \mid \xi_{t}=x) \expect{\tau' \mid \xi'_{t+1}=y}\\
    & =(\sum\limits_{x \in
    \mathcal{X}_{<,1,1}}+\sum^{n-LO_1}_{i=2}\sum\limits_{x \in
    \mathcal{X}_{<,i,1}}+\sum^{n-LO_2}_{j=2}\sum\limits_{x \in
    \mathcal{X}_{<,1,j}}) \pi_t(x) \sum\nolimits_{y \in \mathcal{X}}
    P(\xi_{t+1}=y \mid \xi_{t}=x) \expect{\tau' \mid \xi'_{t+1}=y}\\
    & >\sum\nolimits_{x \in \mathcal{X}_{<,i,j}, \min \{i,j\}=1}\pi_t(x) \sum\nolimits_{y \in \mathcal{X}} P(\xi'_{t+1}=y \mid\xi'_{t}=x) \expect{\tau' \mid\xi'_{t+1}=y}\\
    & \quad+\frac{p_c}{2n-1}(\sum^{n-LO_1}_{i=2}\pi_t(\mathcal{X}_{<,i,1})-\sum^{n-LO_2}_{j=1}\pi_t(\mathcal{X}_{<,1,j}))\\
    & > \sum\nolimits_{x \in \mathcal{X}_{<,i,j}, \min
    \{i,j\}=1}\pi_t(x) \sum\nolimits_{y \in \mathcal{X}} P(\xi'_{t+1}=y
    \mid\xi'_{t}=x) \expect{\tau' \mid\xi'_{t+1}=y},
\end{align}
where the last inequality is by $\forall 2 \leq i \leq n-LO_2:\pi_t(\mathcal{X}_{<,i,1})=\frac{\binom{n-LO_1-1}{i-1}\pi_t(\mathcal{X}_<)}{2^{n-LO_1-1}\cdot 2^{n-LO_2-1}}>\pi_t(\mathcal{X}_{<,1,i})=\frac{\binom{n-LO_2-1}{i-1}\pi_t(\mathcal{X}_<)}{2^{n-LO_1-1}\cdot 2^{n-LO_2-1}}$ and $LO_1<LO_2$.

Combining all the above cases, we have
\begin{align}
     &\forall t: \sum\nolimits_{x,y\in \mathcal{X}} \pi_t(x) P(\xi_{t+1}=y \mid \xi_t=x) \expect{\tau'
     \mid \xi'_{t+1} = y} \\
     & > \sum\nolimits_{x,y\in \mathcal{X}} \pi_t(x) P(\xi'_{t+1}=y \mid \xi'_t=x) \expect{\tau' \mid \xi'_{t+1} = y}+p_c\pi_t(LO_1=LO_2<n).
\end{align}
Therefore, by MCST, we have
\begin{align}
    & \expect{\tau} > \expect{\tau'}+ p_c\sum^{\infty}_{t=0}\pi_t(LO_1=LO_2<n)\\
    &\geq \expect{\tau'}+ p_c\sum^{\infty}_{t=0}(\frac{1}{3}-\frac{1}{3 \cdot 4^n})(p_c+(1-p_c)(1-\frac{1}{n})^2)^t\\
    & = \expect{\tau'} +\frac{p_cn^2}{(1-p_c)(2n-1)}(\frac{1}{3}-\frac{1}{3 \cdot
    4^n}),
\end{align}
where the second inequality is by Proposition \ref{prop_equalLO}.\\
Thus, we have $\expect{\tau} > \expect{\tau'}$, and the expected running time increment satisfies that $\expect{\tau}-\expect{\tau'} \in \Omega(n\frac{p_c}{1-p_c})$.

Then we prove $\expect{\tau} < \expect{\tau'}/(1-p_c)$. For any non-optimal population $x=\{s_1,s_2\}$ with $\|s_1\|=n-i \wedge \|s_2\|=n-j$, we have
\begin{align}
     &\forall t: \sum\nolimits_{y\in \mathcal{X}} P(\xi_{t+1}=y \mid \xi_t=x) \expect{\tau'\mid \xi'_{t+1} = y} \\
     & \leq p_c \mathbb{E}(i,j)+(1-p_c)(\mathbb{E}(i,j)-1) \\
     &=\mathbb{E}(i,j)-1+p_c\\
     &=\sum\nolimits_{y\in \mathcal{X}} P(\xi'_{t+1}=y \mid \xi'_t=x) \expect{\tau' \mid \xi'_{t+1} =y}+p_c,
\end{align}
where the first inequality is by Proposition
\ref{prop_transition_LO} and Lemma \ref{lem_tau}, and the last equality is by Lemma \ref{lem_tau}. \\
Thus, we have
\begin{align}
 &\forall t: \sum\nolimits_{x,y\in \mathcal{X}} \pi_t(x)P(\xi_{t+1}=y \mid \xi_t=x) \expect{\tau'\mid \xi'_{t+1} = y} \\
 & \leq \sum\nolimits_{x,y\in \mathcal{X}} \pi_t(x)P(\xi'_{t+1}=y \mid \xi'_t=x) \expect{\tau'\mid \xi'_{t+1} =y}+p_c(1-\pi_t(\mathcal{X}^*)).
\end{align}
Then, by MCST, we have $\expect{\tau} \leq \expect{\tau'}+p_c\sum^{\infty}_{t=0}(1-\pi_t(\mathcal{X^*}))=\expect{\tau'}+p_c\expect{\tau}$.
Thus, $\expect{\tau}\leq\expect{\tau'}/(1-p_c)$.
\end{myproof}

\subsection{On OneMax problem}

%
%

First, we will analyze the one-step transition behavior of the (2:2)-EA with crossover and that without crossover on the OneMax problem, i.e., $P(\xi_{t+1} \mid \xi_{t})$ and $P(\xi'_{t+1} \mid \xi'_t)$.

\begin{proposition}\label{prop_transition_OneMax}
    For any non-optimal population $x=\{s_1,s_2\}$, denote $S'_1$ as the set of solutions which flip one 0 bit of $s_1$ and $S'_2$ as the set of solutions which flip one 0 bit of $s_2$. Let $i$ and $j$ be the number of 0s of $s_1$ and $s_2$, respectively. Let k be the number of positions where the bit of $s_1$ is 0 and the bit of $s_2$ is 1. Then, for $\{\xi'_t\}^{\infty}_{t=0}$, we have
    \begin{align}
     & P(\xi'_{t+1} \in \mathcal{X}_1=\{\{s'_1,s'_2\} \mid s'_1 \in S'_1, s'_2 \in S'_2\} \mid \xi'_t=x) =\frac{ij}{n^2}, \\
     & P(\xi'_{t+1} \in \mathcal{X}_2=\{\{s_1,s'_2\} \mid s'_2 \in S'_2\} \mid \xi'_t=x) =\frac{(n-i)j}{n^2}, \\
     & P(\xi'_{t+1} \in \mathcal{X}_3=\{\{s'_1,s_2\} \mid s'_1 \in S'_1\} \mid \xi'_t=x) =\frac{(n-j)i}{n^2}, \\
     & P(\xi'_{t+1}=x \mid \xi'_t=x) =\frac{(n-i)(n-j)}{n^2},\\
     & P(\xi'_{t+1}\in \mathcal{X}-\mathcal{X}_1 \cup \mathcal{X}_2 \cup \mathcal{X}_3 \cup \{x\} \mid\xi'_t=x) =0;
    \end{align}
    and for $\{\xi_t\}^{\infty}_{t=0}$, we have
    \begin{align}
    & P(\xi'_{t+1} \in \mathcal{X}_2 \mid \xi'_t=x) =\frac{p_c(j-i+k)}{n}+\frac{(1-p_c)(n-i)j}{n^2}, \\
    & P(\xi'_{t+1} \in \mathcal{X}_3 \mid \xi'_t=x) =\frac{p_ck}{n}+\frac{(1-p_c)(n-j)i}{n^2}, \\
    & P(\xi'_{t+1}=x \mid \xi'_t=x) =\frac{p_c(n+i-j-2k)}{n}+\frac{(1-p_c)(n-i)(n-j)}{n^2},\\
    & \forall y \in \mathcal{X}-\mathcal{X}_2 \cup \mathcal{X}_3 \cup
    \{x\}: P(\xi_{t+1}=y \mid \xi_t=x) =(1-p_c)P(\xi'_{t+1}=y \mid
    \xi'_t=x).
    \end{align}
\end{proposition}

For the OneMax problem, it is easy to see that the optimal solution is $11\ldots1$. For the Markov chain $\{\xi'_t\}_{t=0}^{+\infty}$, it is obvious that the CFHT $\expect{\tau'_t \mid \xi'_t=\{s_1,s_2\}}$ only depends on the number of 1's of the two solutions, i.e., $\{\|s_1\|,\|s_2\|\}$. Given a population $\{s_1,s_2\}$ with $\|s_1\|=n-i$ and $\|s_2\|=n-j$, we denote $\mathbb{E}(i,j)$ as the CFHT $\expect{\tau' \mid \xi'_t=\{s_1,s_2\}}$ of EA using mutation only.

Then, we give three propositions on the property of the CFHT $\expect{\tau' \mid \xi'_t=\{s_1,s_2\}}$. Proposition \ref{prop_cfht1_OneMax} is the instantiation of Lemma \ref{lem_tau} for (2:2)-EA with mutation only on the OneMax problem. Propositions \ref{prop_cfht2_OneMax} and \ref{prop_cfht3_OneMax} compare the CFHT from adjacent populations, where the number of 1 bits for one solution is same, and the difference of the number of 1 bits for the other solution is 1.

\begin{proposition}\label{prop_cfht1_OneMax}
    The CFHT of $\{\xi'_t\}_{t=0}^{+\infty}$ on the OneMax problem
    satisfies that
    \begin{equation}
    \begin{aligned}
    & \forall i\geq 0, \mathbb{E}(i,0)=\mathbb{E}(0,i)=0;\\
    & \forall i,j \geq 1,\mathbb{E}(i,j)=\\
    &\frac{n^2}{(i+j)n-ij}+\frac{ij}{(i+j)n-ij}\mathbb{E}(i-1,j-1)+\frac{i(n-j)}{(i+j)n-ij}\mathbb{E}(i-1,j)+\frac{(n-i)j}{(i+j)n-ij}\mathbb{E}(i,j-1).
    \end{aligned}
    \end{equation}
\end{proposition}

\begin{proposition}\label{prop_cfht2_OneMax}
 The CFHT of $\{\xi'_t\}_{t=0}^{+\infty}$ satisfies that
    \begin{align}
    &\forall i \geq 1, \delta \geq 1:  \cexpect{i,i+\delta}-\cexpect{i,i+\delta-1}>\frac{n}{2^{\delta+1}(i+\delta)} ,  \\
    &\forall i \geq 1, \delta \geq 0:  \cexpect{i,i+\delta}-\cexpect{i-1,i+\delta} <(1-\frac{3}{2^{\delta+3}})\frac{n}{i}+\frac{1}{2^{\delta+3}}.
    \end{align}
\end{proposition}

\begin{proposition}\label{prop_cfht3_OneMax}
 The CFHT of $\{\xi'_t\}_{t=0}^{+\infty}$ satisfies that
    \begin{align}
    &\forall i \geq 0, \delta \geq 1: \cexpect{i,i+\delta}-\cexpect{i,i+\delta-1} < \frac{n}{2(i+\delta)},  \\
    &\forall i \geq 1, \delta \geq 0:   \cexpect{i,i+\delta}-\cexpect{i-1,i+\delta} >\frac{n}{2i}.
    \end{align}
\end{proposition}

Then, we will analyze the distribution $\pi_t$. First, we define several notations for the simplicity of analysis. For population $x=\{s_1,s_2\}$, we define $N_x(0,1)=\sum^n_{i=1} (1-s_1(i))s_2(i)$, i.e., the number of positions where the bit of $s_1$ is 0 and the bit of $s_2$ is 1. Similarly, we can define $N_x(0,0)$ and $N_x(1,0)$. In the evolutionary process, it is obvious that the change of the bit-pair on each position of the population is independent. For the Markov chain $\{\xi_t\}^{\infty}_{t=0}$, given the distribution $\pi_t$ of $\xi_t$, we define $p_t(0,1)$ as the probability that the bit of the first solution on one position is 0 and the bit of the second solution on the same position is 1. Similarly, we can define $p_t(0,0)$. In Proposition \ref{prop_onemax_dist1}, we calculate $p_t(0,0)$ and $p_t(0,1)$. In Proposition \ref{prop_onemax_dist2}, we derive a relation of the expected value of $N_x(0,1)/(N_x(0,1)+N_x(0,0))$ between adjacent time steps. From these two propositions, we establish a relation among $N_x(0,1)/(N_x(0,1)+N_x(0,0))$, $p_t(0,0)$ and $p_t(0,1)$, which is given in Proposition \ref{prop_onemax_dist3}.

\begin{proposition}\label{prop_onemax_dist1}
    For the Markov chain $\{\xi_t\}^{\infty}_{t=0}$, we have
    \begin{align}
    & p_t(0,0)=\frac{1}{4}(1-\frac{2(1-p_c)}{n}+\frac{1-p_c}{n^2})^t;\\
    & p_t(0,1)=
    \begin{cases}
    \frac{1}{4}(1+\frac{t}{2n-1})(1-\frac{1}{n})^t& \text{if $p_c=\frac{n-1}{2n-1}$,}\\
    \frac{n-1}{4(1-2n+\frac{n}{1-p_c})}(1-\frac{2(1-p_c)}{n}+\frac{1-p_c}{n^2})^t+\frac{2-3n+\frac{n}{1-p_c}}{4(1-2n+\frac{n}{1-p_c})}(1-\frac{1}{n})^t&\text{otherwise.}
    \end{cases}
    \end{align}
\end{proposition}

\begin{proposition}\label{prop_onemax_dist2}
    For the Markov chain $\{\xi_t\}^{\infty}_{t=0}$, we have
    $$
    \expect{\frac{N_x(0,1)}{N_x(0,1)+N_x(0,0)} \mid x \sim \pi_{t+1}}
    \leq (1-\frac{1-p_c}{n}) \expect{\frac{N_x(0,1)}{N_x(0,1)+N_x(0,0)}
    \mid x \sim \pi_{t}}+\frac{1-p_c}{n}.
    $$
\end{proposition}

\begin{proposition}\label{prop_onemax_dist3}
    For the Markov chain $\{\xi_t\}^{\infty}_{t=0}$, we have
    $$
    \sum_{x \in \mathcal{X}} \pi_t(x) \frac{N_x(0,1)}{N_x(0,1)+N_x(0,0)}
    \leq 1-p_t(0,1)-p_t(0,0)-(1-\frac{1-p_c}{n})^t.
    $$
\end{proposition}

\begin{theorem}\label{theorem_OM_compare}
For the OneMax problem, let the Markov chains $\{\xi_t\}_{t=0}^{+\infty}$ and $\{\xi'_t\}_{t=0}^{+\infty}$. When $n \geq 2$, we have $\expect{\tau'}<\expect{\tau}\leq \expect{\tau'}/(1-p_c)$ and $\expect{\tau} - \expect{\tau'} \in
\Omega(n\frac{p_c}{1-p_c})$.
\end{theorem}
\begin{myproof}
     For any population $x=(s_1,s_2)$ such that $\max\{\|s_1\|,\|s_2\|\}=n$, we have
    \begin{align}
    & \sum\nolimits_{y \in \mathcal{X}} P(\xi_{t+1}=y \mid \xi_{t}=x) \expect{\tau' \mid \xi'_{t+1}=y}\\
    & =\sum\nolimits_{y \in \mathcal{X}} P(\xi'_{t+1}=y \mid \xi'_{t}=x)
    \expect{\tau' \mid \xi'_{t+1}=y}=0,
    \end{align}
    since the (2:2)-EA tracks the best so-far solution.

    Let $(i,j)$ be the set of populations such that the first solution has $i$ number of 0's and the second solution has $j$ number of 0's,  $\pi_t((i,j))=\sum_{x \in (i,j)} \pi_t(x)$, and $\mathbb{E}_{t,i,j}(N(0,1))=\frac{1}{\pi_t((i,j))}\sum_{x \in (i,j)}
    \pi_t(x) N_x(0,1)$ be the expected value of $N_x(0,1)$ of populations in the set $(i,j)$ under the distribution $\pi_t$. Then,
    we have
    \begin{align}
    & \forall t: \sum\nolimits_{x,y\in \mathcal{X}} \pi_t(x)P(\xi_{t+1}=y \mid \xi_t=x) \expect{\tau' \mid \xi'_{t+1} =y}\\
    &-\sum\nolimits_{x,y\in \mathcal{X}} \pi_t(x)P(\xi'_{t+1}=y \mid \xi'_t=x) \expect{\tau' \mid \xi'_{t+1} =y}\\
    &=\sum^{n}_{i=1}
    \sum^{n}_{j=1}\pi_t(i,j)(p_c(\frac{\mathbb{E}_{t,i,j}(N(0,1))}{n}\mathbb{E}(i-1,j)+\frac{j-i+\mathbb{E}_{t,i,j}(N(0,1))}{n}\mathbb{E}(i,j-1)\\
    &\quad +\frac{n+i-j-2\mathbb{E}_{t,i,j}(N(0,1))}{n}\mathbb{E}(i,j))+(1-p_c)(\mathbb{E}(i,j)-1))\\
    &\quad -\sum^{n}_{i=1} \sum^{n}_{j=1}\pi_t(i,j)(\frac{ij}{n^2}\mathbb{E}(i-1,j-1)+\frac{i(n-j)}{n^2}\mathbb{E}(i-1,j)+\frac{(n-i)j}{n^2}\mathbb{E}(i,j-1) \\
    &\quad +\frac{(n-i)(n-j)}{n^2}\mathbb{E}(i,j))\\
    &=\sum^{n}_{i=1} \sum^{n}_{j=1}\pi_t(i,j)p_c(\frac{ij}{n^2}(\mathbb{E}(i-1,j)+\mathbb{E}(i,j-1)-\mathbb{E}(i-1,j-1)-\mathbb{E}(i,j))\\
    &\quad +\frac{i-\mathbb{E}_{t,i,j}(N(0,1))}{n}(2\mathbb{E}(i,j)-\mathbb{E}(i-1,j)-\mathbb{E}(i,j-1)))\\
    &>\sum^{n}_{i=1} \sum^{n}_{j=1}\pi_t(i,j)p_c(\frac{ij}{n^2}\frac{-n}{2\max\{i,j\}}+\frac{i-\mathbb{E}_{t,i,j}(N(0,1))}{n}\frac{n}{2\min\{i,j\}})\\
    & \geq \sum^{n}_{i=1} \sum^{n}_{j=1}\pi_t(i,j)\frac{p_c}{2}(1-\frac{\mathbb{E}_{t,i,j}(N(0,1))}{i}-\frac{i}{n})\\
    & \geq \sum^{n}_{i=0} \sum^{n}_{j=0}\pi_t(i,j)\frac{p_c}{2}(1-\frac{\mathbb{E}_{t,i,j}(N(0,1))}{i}-\frac{i}{n})\\
    & =\frac{p_c}{2}(1-p_t(0,1)-p_t(0,0)-\sum_{\mathclap{x\in \mathcal{X}}} \pi_t(x)\frac{N_x(0,1)}{N_x(0,1)+N_x(0,0)})\\
    & \geq \frac{p_c}{2}(1-\frac{1-p_c}{n})^t,
    \end{align}
    where the first equality is by analyzing the one-step transition behaviors of $\{\xi_t\}_{t=0}^{+\infty}$ and $\{\xi'_t\}_{t=0}^{+\infty}$ in Proposition \ref{prop_transition_OneMax}, the first inequality is by analyzing the CFHT of $\{\xi'_t\}_{t=0}^{+\infty}$ presented in Proposition \ref{prop_cfht2_OneMax} and Proposition \ref{prop_cfht3_OneMax}, the second inequality is by $\min\{i,j\} \leq i$ and $\max\{i,j\} \geq j$, the third inequality is by $\frac{0}{0}=1$ in computation, the following equality is by $\sum^{n}_{i=0} \sum^{n}_{j=0}\pi_t(i,j)\frac{i}{n}=p_t(0,1)+p_t(0,0)$ and
    $$
        \sum^{n}_{\mathclap{i=0}}\sum^{n}_{\mathclap{j=0}}\pi_t(i,j)\frac{\mathbb{E}_{t,i,j}(N(0,1))}{i}\!=\!\sum_{\mathclap{x
    \in \mathcal{X}}} \pi_t(x) \frac{N_x(0,1)}{N_x(0,1)+N_x(0,0)},
    $$
    and the last inequality is by Proposition \ref{prop_onemax_dist3}.

    By MCST, we get $\expect{\tau}  > \expect{\tau'}+n\frac{p_c}{2(1-p_c)}$. Thus, $\expect{\tau}-\expect{\tau'} \in \Omega(n\frac{p_c}{1-p_c}).$ Then, same as the proof of $\expect{\tau} \leq \expect{\tau'}/(1-p_c)$ for LeadingOnes problem, we can also prove that $\expect{\tau} \leq \expect{\tau'}/(1-p_c)$ for OneMax problem.
\end{myproof}

\section{Crossover Strategies}\label{sec:condition}
We have shown that using the crossover throughout the evolutionary process does not help the search. In the third aspect, instead, we try to use the crossover only when necessary. We replace the Reproduction step of the (2:2)-EA by a \emph{mutation and crossover strategy} (M\&R), which determines when to use the mutation and the crossover. We derive several strategies that apply the crossover only when necessary. The strategies involve crossover operators below.
\begin{description}
    \item[one-diff-bit crossover] for the current two solutions, randomly choose one of the bit positions where the two solutions have different bits, and exchange the bit on that position.
    \item[first-diff-bit crossover] for the current two solutions, scan the solutions left-to-right, exchange the first different bit.
    \item[first-diff-point crossover] for the current two solutions, scan the solutions left-to-right, exchange bits starting from the position of the first different bit.
\end{description}

\subsection{On LeadingOnes Problem}

Let the Markov chain $\xi$ model the EA with a crossover strategy, $\xi'$ model the EA with one-bit mutation only, and  $\expect{\tau}$ and $\expect{\tau'}$ respectively denote the EFHT of (2:2)-EA with a crossover strategy and one-bit mutation only. Given two solutions $(s_1,s_2)$ such that $\|s_1\|=n-i$ and $\|s_2\|=n-j$, we reuse the notation $\cexpect{i,j}$ as the CFHT $\expect{\tau'_t | \xi'_t = \{s_1,s_2\} }$ of (2:2)-EA with one-bit mutation only. We also denote $LO_1=LO(s_1)$, $LO_2=LO(s_2)$, and $\delta = \delta(s_1,s_2)$.

We use Proposition \ref{prop_cfht2_LO} together with \ref{prop_cfht3_LO} to characterize the CFHT of the (2:2)-EA with one-bit mutation only.

\begin{Prop}\label{prop_cfht3_LO}
    The CFHT of $\{\xi'_t\}_{t=0}^{+\infty}$ satisfies that
    \begin{align}
    &\forall i \geq 0, \delta \geq 1:\cexpect{i,i+\delta}-\cexpect{i,i+\delta-1} < \frac{n}{2},  \\
    &\forall i \geq 1, \delta \geq 0: \cexpect{i,i+\delta}-\cexpect{i-1,i+\delta} > \frac{n}{2}.
    \end{align}
    \end{Prop}

We start from designing a simple strategy, M\&R1a, defined below, which uses the first-diff-bit crossover. Theorem \ref{them_LO_mr1a} proves that the EA using M\&R1a has a smaller runtime than that using one-bit mutation.

\noindent{\bf M\&R1a:} Use the first-diff-bit crossover if either $LO_1<LO_2$ or (both $\delta=0$ and $LO_1 \neq LO_2$); Use the one-bit mutation otherwise.

\begin{theorem}\label{them_LO_mr1a}
        For the LeadingOnes problem, given the crossover strategy of the (2:2)-EA being implemented by M\&R1a, when $n \geq
        2$, we have $\expect{\tau}\leq \expect{\tau'}$.
\end{theorem}
    \begin{myproof}
        For any population $x=\{s_1,s_2\}$ such that $\|s_1\|=n-i$ and $\|s_2\|=n-i-\delta(\delta \geq 0)$.\\
        a) in the case $LO_1<LO_2$, the first different bit of the two solutions is after the leading 1 bits of $s_1$, where the bit of $s_1$ is 0 and that of $s_2$ is 1. Then, the first-diff-bit crossover exchanges the bit on that position and reproduces two solutions $\{s^R_1,s^R_2\}$, this results that $\|s^R_1\|=n-i+1$ and $\|s^R_2\|=n-i-\delta-1$. Note that the fitness of  $s^R_1$ is larger than that of $s_1$ and the fitness of $s^R_2$ is lower than that of $s_2$, therefore, after the selection, two solutions $\{s'_1,s'_2\}$ are selected into the next iteration, where $s'_1=s^R_1$ and $s'_2=s_2$. Therefore, we have that $\|s'_1\|=n-i+1$ and $\|s'_2\|=n-i-\delta$.\\
        Similarly, in the case $\delta=0 \wedge LO_1 \neq LO_2$, w.o.l.g., let $LO_1<LO_2$, we also have $\|s'_1\|=n-i+1$ and $\|s'_2\|=n-i-\delta$.\\
        Thus we have
        \begin{align}
            & \sum_{y \in \mathcal{X}} P(\xi_{t+1}=y \mid \xi_{t}=x) \expect{\tau'_{t+1} \mid \xi'_{t+1}=y} \\
            & =\cexpect{i-1,i+\delta} \\
            & < \cexpect{i,i+\delta} -\frac{n}{2} \leq \cexpect{i,i+\delta}-1.
        \end{align}
        b) otherwise, it uses the one-bit mutation, which yields
        \begin{align}
            & \sum_{y \in \mathcal{X}} P(\xi_{t+1}=y \mid \xi_{t}=x) \cexpect{\tau'_{t+1} \mid \xi'_{t+1}=y} \\
            & =\sum_{y \in \mathcal{X}} P(\xi'_{t+1}=y \mid \xi'_{t}=x) \cexpect{\tau'_{t+1} \mid \xi'_{t+1}=y} \\
            &= \cexpect{i,i+\delta}-1.
        \end{align}
         Thus, for any population $x$, it satisfies that
        \begin{align}
            & \sum_{y \in \mathcal{X}} P(\xi_{t+1}=y \mid \xi_{t}=x) \expect{\tau'_{t+1} \mid \xi'_{t+1}=y} \\
            & \leq \cexpect{i,i+\delta}-1 \\
            & = \sum_{y \in \mathcal{X}} P(\xi'_{t+1}=y \mid \xi'_{t}=x) \expect{\tau'_{t+1} \mid \xi'_{t+1}=y}.
        \end{align}
        By Theorem \ref{them_main}, we immediately have $\expect{\tau}\leq \expect{\tau'}$.
    \end{myproof}

Under a different condition, we can design another strategy, M\&R1b. Theorem \ref{them_LO_mr1b} proves that M\&R1b is superior than one-bit mutation.

\noindent{\bf M\&R1b:} Use the first-diff-bit crossover if both $LO_1>LO_2$ and $0< \delta \leq 2$; Use the one-bit mutation otherwise.

    \begin{theorem}\label{them_LO_mr1b}
        For the LeadingOnes problem, given the crossover strategy of the (2:2)-EA being implemented by M\&R1b, when $n \geq
        16$, we have $\expect{\tau}\leq \expect{\tau'}$.
    \end{theorem}
    \begin{myproof}
        For any population $x=\{s_1,s_2\}$ such that $\|s_1\|=n-i$ and $\|s_2\|=n-i-\delta (\delta \geq 0)$.\\
        a) in the case $0< \delta \leq 2 \wedge LO_1>LO_2$, using the first-diff-bit crossover, together with the selection, reproduces two solutions $\{s'_1,s'_2\}$ such that $\|s'_1\|=\|s_1\|=n-i$ and $\|s'_2\|=n-i-\delta+1$. \\
        Thus we have
        \begin{align}
            & \sum_{y \in \mathcal{X}} P(\xi_{t+1}=y \mid \xi_{t}=x) \expect{\tau'_{t+1} \mid \xi'_{t+1}=y} \\
            &=\cexpect{i,i+\delta-1} \\
            & \leq \cexpect{i,i+\delta} - \frac{n}{2^{\delta+2}} \leq \cexpect{i,i+\delta}-1.
        \end{align}
        b)otherwise, it uses the one-bit mutation, which yields
        \begin{align}
            & \sum_{y \in \mathcal{X}} P(\xi_{t+1}=y \mid \xi_{t}=x) \cexpect{\tau'_{t+1} \mid \xi'_{t+1}=y} \\
            & =\sum_{y \in \mathcal{X}} P(\xi'_{t+1}=y \mid \xi'_{t}=x) \cexpect{\tau'_{t+1} \mid \xi'_{t+1}=y} \\
            &= \cexpect{i,i+\delta}-1.
        \end{align}
         Thus, for any population $x$, it satisfies that
        \begin{align}
            & \sum_{y \in \mathcal{X}} P(\xi_{t+1}=y \mid \xi_{t}=x) \expect{\tau'_{t+1} \mid \xi'_{t+1}=y} \\
            & \leq \cexpect{i,i+\delta}-1 \\
            & = \sum_{y \in \mathcal{X}} P(\xi'_{t+1}=y \mid \xi'_{t}=x) \expect{\tau'_{t+1} \mid \xi'_{t+1}=y}.
        \end{align}
        By Theorem \ref{them_main}, we immediately have $\expect{\tau}\leq \expect{\tau'}$.
    \end{myproof}

From the proofs of Theorems \ref{them_LO_mr1a} and \ref{them_LO_mr1b}, we can find that the first-diff-bit crossover operator works in different ways under different situations. For the two solutions in the current population, under the condition of M\&R1a, the crossover helps the solution with better fitness but farther to the optimal solution (i.e., smaller number of 1 bits), while under the condition of M\&R1b, the crossover helps the solution with worse fitness and farther to the optimal solution.

Since the strategies M\&R1a and M\&R1b work under non-overlap conditions, they can be combined into one strategy, M\&R1. Corollary \ref{cor_LO_mr1} shows that M\&R1 is better than the mutation operator, which is proved by combining the proofs of Theorems \ref{them_LO_mr1a} and \ref{them_LO_mr1b}.

\noindent{\bf M\&R1} Use the first-diff-bit crossover if either M\&R1a or M\&R1b would be triggered to use the first-diff-bit crossover; Use the one-bit mutation otherwise.

\begin{corollary}\label{cor_LO_mr1}
        For the LeadingOnes problem, given the crossover strategy of the (2:2)-EA being implemented by M\&R1, when $n \geq
        16$, we have $\expect{\tau}\leq \expect{\tau'}$.
\end{corollary}

Since M\&R1b improve the worst solution by exchanging the first different bit, we further designed M\&R2 which uses first-diff point crossover to exchange all the bits following the first different bit. It can also be proved that M\&R2 is better than the mutation.

\noindent{\bf M\&R2:} Use the first-diff-bit crossover if either $LO_1<LO_2$ or (both $\delta=0$ and $LO_1 \neq LO_2$); Use the first-diff-point crossover if both $LO_1>LO_2$ and $\delta \neq 0$; Use the one-bit mutation otherwise.

\begin{theorem}\label{prop_LO_mr2}
        For the LeadingOnes problem, given the crossover strategy of the (2:2)-EA being implemented by M\&R3, when $n \geq
        8$, we have $\expect{\tau}\leq \expect{\tau'}$.
\end{theorem}
\begin{myproof}
      For any population $x=\{s_1,s_2\}$ such that $\|s_1\|=n-i$ and $\|s_2\|=n-i-\delta(\delta \geq 0)$.\\
        a) in the case $LO_1>LO_2 \wedge \delta>0$, it uses the first-diff-point crossover, which reproduces two solutions $\{s'_1,s'_2\}$ such that $\|s'_1\|=n-i$ and $\|s'_2\|=n-i$. Thus we have
        \begin{align}
            & \sum_{y \in \mathcal{X}} P(\xi_{t+1}=y \mid \xi_{t}=x) \expect{\tau'_{t+1} \mid \xi'_{t+1}=y} \\
            & =\expect{i,i} \\
            & \leq \expect{i,i+\delta} -\frac{n}{4}(1-\frac{1}{2^{\delta}}) \leq \expect{i,i+\delta} -\frac{n}{8} \leq \expect{i,i+\delta}-1
        \end{align}
        b) in the case $LO_1<LO_2$ or $LO_1 \neq LO_2 \wedge \delta=0$, it uses the first-diff-bit crossover, so we have already known that
        \begin{align}
            & \sum_{y \in \mathcal{X}} P(\xi_{t+1}=y \mid \xi_{t}=x) \expect{\tau'_{t+1} \mid \xi'_{t+1}=y} \\
            & = \expect{i-1,i+\delta} < \expect{i,i+\delta} -\frac{n}{2} < \expect{i,i+\delta}-1.
        \end{align}
        c) Otherwise(in the case $LO_1=LO_2$), it uses the one-bit mutation only,
        \begin{align}
            & \sum_{y \in \mathcal{X}} P(\xi_{t+1}=y \mid \xi_{t}=x) \expect{\tau'_{t+1} \mid \xi'_{t+1}=y} \\
            & =\sum_{y \in \mathcal{X}} P(\xi'_{t+1}=y \mid \xi'_{t}=x) \expect{\tau'_{t+1} \mid \xi'_{t+1}=y}.\\
            & =\cexpect{i,i+\delta}-1.
        \end{align}
        Thus, for any population $x$, it satisfies that
        \begin{align}
            & \sum_{y \in \mathcal{X}} P(\xi_{t+1}=y \mid \xi_{t}=x) \expect{\tau'_{t+1} \mid \xi'_{t+1}=y} \\
            & \leq \cexpect{i,i+\delta}-1 \\
            & = \sum_{y \in \mathcal{X}} P(\xi'_{t+1}=y \mid \xi'_{t}=x) \expect{\tau'_{t+1} \mid \xi'_{t+1}=y}.
        \end{align}
        By Theorem \ref{them_main}, we immediately have $\expect{\tau}\leq \expect{\tau'}$.
\end{myproof}

\subsection{On OneMax Problem}

We reuse the notations $\xi$, $\xi'$, $\delta$ and $\cexpect{i,j}$ as in the previous subsection, except that the EAs are running on the OneMax problem.

We use Proposition \ref{prop_cfht2r_OneMax}, which can be simply derived from Proposition \ref{prop_cfht2_OneMax},  and Proposition \ref{prop_cfht3_OneMax} to characterize the CFHT of the (2:2)-EA with one-bit mutation.


\begin{Prop}\label{prop_cfht2r_OneMax}
    The CFHT of $\{\xi'_t\}_{t=0}^{+\infty}$ satisfies that
    \begin{align}
        & \forall i \geq 1, \delta \geq 1: \quad \mathbb{E}(i,i+\delta)-\mathbb{E}(i,i+\delta-1) > 1,\\
        & \forall i \geq 1, \delta \geq 0: \quad \mathbb{E}(i,i+\delta)-\mathbb{E}(i-1,i+\delta) < \frac{n}{i}.
    \end{align}
\end{Prop}

We then analyze the strategy M\&R3 that uses a one-diff-bit crossover.

{\bf M\&R3:} If the two solutions are not identical, within probability 0.5, use the one-diff-bit crossover if $n \geq (i+\delta)(1+\frac{n(i+\delta)}{n-i-\delta})^i$; otherwise, use the one-bit mutation.

 \begin{theorem}\label{them_onemax_mr3}
        For the OneMax problem, given the crossover strategy of the (2:2)-EA being implemented by M\&R4, when $n \geq
        2$, we have $\expect{\tau}\leq \expect{\tau'}$.
    \end{theorem}
    \begin{myproof}
        For any population $x=\{s_1,s_2\}$ such that $\|s_1\|=n-i$ and $\|s_2\|=n-i-\delta(\delta \geq 0)$, it uses either one-bit mutation or one-diff-bit crossover in any step.\\
        a) using the one-diff-bit crossover together with the selection, two offspring populations $\{s'_1,s'_2\}$ are possible to be reproduced. The first one is that $\|s'_1\|=n-i+1 \wedge \|s'_2\|=n-i-\delta$, and the second one is that $\|s'_1\|=n-i \wedge \|s'_2\|=n-i-\delta+1$. \\
        For the first possible offspring population, under the condition $n\geq 2i$, we have that,
        \begin{align}
            & \sum_{y \in \mathcal{X}} P(\xi_{t+1}=y \mid \xi_{t}=x) \expect{\tau'_{t+1} \mid \xi'_{t+1}=y} \\
            & =\cexpect{i-1,i+\delta} < \cexpect{i,i+\delta} -\frac{n}{2i} < \cexpect{i,i+\delta}-1.
        \end{align}
        For the second possible offspring population, under the condition $n \geq (i+\delta)(1+\frac{n(i+\delta)}{n-i-\delta})^i$, we have that,
        \begin{align}
            & \sum_{y \in \mathcal{X}} P(\xi_{t+1}=y \mid \xi_{t}=x) \expect{\tau'_{t+1} \mid \xi'_{t+1}=y} \\
            & =\cexpect{i,i+\delta-1} < \cexpect{i,i+\delta} -\frac{n}{(i+\delta)(1+\frac{n(i+\delta)}{n-i-\delta})^i} \leq \cexpect{i,i+\delta}-1.
        \end{align}
        b) otherwise, it uses the one-bit mutation, which yields
        \begin{align}
            & \sum_{y \in \mathcal{X}} P(\xi_{t+1}=y \mid \xi_{t}=x) \cexpect{\tau'_{t+1} \mid \xi'_{t+1}=y} \\
            & =\sum_{y \in \mathcal{X}} P(\xi'_{t+1}=y \mid \xi'_{t}=x) \cexpect{\tau'_{t+1} \mid \xi'_{t+1}=y} \\
            &= \cexpect{i,i+\delta}-1.
        \end{align}
         Thus, for any population $x$, it satisfies that
        \begin{align}
            & \sum_{y \in \mathcal{X}} P(\xi_{t+1}=y \mid \xi_{t}=x) \expect{\tau'_{t+1} \mid \xi'_{t+1}=y} \\
            & \leq \cexpect{i,i+\delta}-1 \\
            & = \sum_{y \in \mathcal{X}} P(\xi'_{t+1}=y \mid \xi'_{t}=x) \expect{\tau'_{t+1} \mid \xi'_{t+1}=y}.
        \end{align}
        By Theorem \ref{them_main}, we immediately have $\expect{\tau}\leq \expect{\tau'}$.
    \end{myproof}

In the proof, we observed that the one-diff-bit crossover improves either of the solutions in the population. In the case the better solution is improved, the improvement on the population is greater than that by one-bit mutation, under the condition $n\geq 2i$; and in the case the worse solution is improved, under the condition $n \geq (i+\delta)(1+\frac{n(i+\delta)}{n-i-\delta})^i$, the improvement is greater than that by one-bit mutation. Note that the latter condition contains the former. The condition $n \geq (i+\delta)(1+\frac{n(i+\delta)}{n-i-\delta})^i$ will hold when both $i$ and $\delta$ are small, which means the solutions are close to the optimum. Meanwhile, note that the closer the solutions to the optimum, the more time the one-bit mutation takes to make an improvement. Thus M\&R4 works better as the solutions getting closer to the optimum.

\section{Empirical Verification}

To verify the derived asymptotical results, we carried out experiments. On each problem size, we repeat independent runs of each implementation of the EA for $1,000$ times, and then the average running time is recorded as an estimation of the expected running time. For the parameter $p_c$ of the crossover probability in the (2:2)-EA, we use values $\{0,0.1, 0.5, 0.9\}$, where $p_c=0$ equals to the mutation-only EA.

We first verify the results of Theorems \ref{theorem_LO_compare} and \ref{theorem_OM_compare}, which indicate that the running time of (2:2)-EA will increase at least $\Omega(n\frac{p_c}{1-p_c})$ by using the crossover with probability $p_c$. We plot the results of the estimated running time in Figure~\ref{fig_efht}. On both of the problems, it can be seen that the curve of the (2:2)-EA with mutation only (i.e. $p_c$=0) is below all the curves. As $p_c$ increases, the corresponding curve rises, which is consistent with the theoretical results that says using the crossover operator leads to larger expected running time.

To get a closer look of how tight is the bound $\Omega(n\frac{p_c}{1-p_c})$, we calculate the running time \emph{gap} as
$$
(\text{EFHT of crossover-enabled EA} - \text{EFHT of mutation-only EA})\cdot \frac{1}{n}\cdot\frac{1-p_c}{ p_c}.
$$
By the derived the bound $\Omega(n\frac{p_c}{1-p_c})$, the \emph{gap} should be a constant or increase as $n$ increases. We plot the experiment result of the \emph{gap} in Figure~\ref{fig_gap}. We can observe that, for the LeadingOnes problem, the curves of the gap grow in a closely linear trend, and for the OneMax problem, the curves grow in a closely logarithmic trend. The observation confirms bound $\Omega(n\frac{p_c}{1-p_c})$, and also suggests that there is still room to improve the bound.

Theorems \ref{theorem_LO_compare} and \ref{theorem_OM_compare} also say that $\expect{\tau}\leq \expect{\tau'}\cdot \frac{1}{1-p_c}$, i.e., the running time of the crossover-enabled EA is at most $\frac{1}{1-p_c}$ times of that of the mutation-only EA. We then investigate the \emph{ratio} calculated as
$$
\frac{\text{EFHT of crossover-enabled EA}}{\text{EFHT of mutation-only EA}}\cdot (1-p_c),
$$
and compare it with 1. We plot the experiment result of the \emph{ratio} in Figure~\ref{fig_ratio}.
The figure shows that, on the two problems, the \emph{ratio} is consistently bounded by 1 except when the problem size is very small that causes a large variance.

\begin{figure*}[t!]\centering
\begin{minipage}[c]{0.45\linewidth}\centering
    \includegraphics[width=1\linewidth,height=0.75\linewidth]{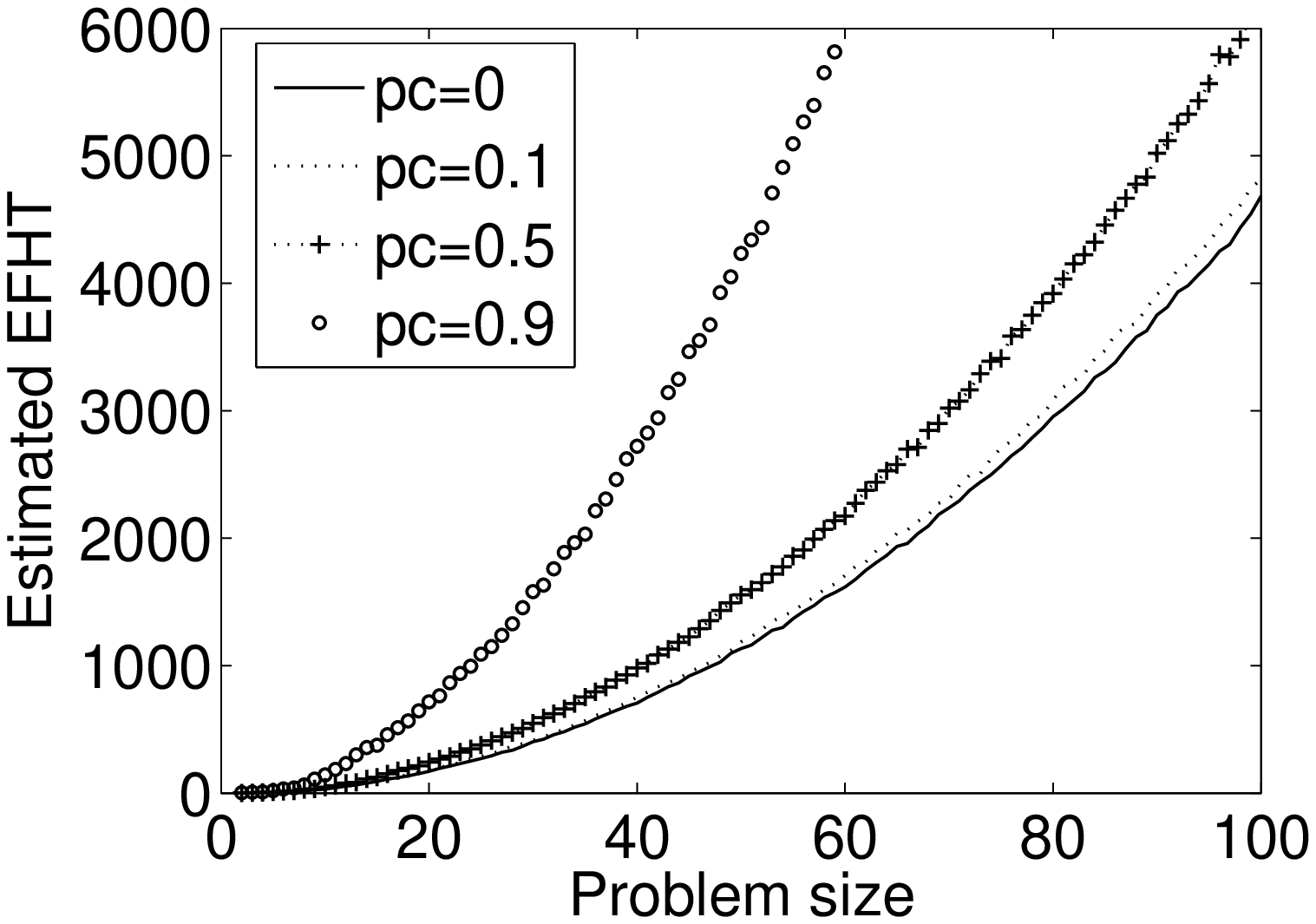}
\end{minipage}
\begin{minipage}[c]{0.45\linewidth}\centering
    \includegraphics[width=1\linewidth,height=0.75\linewidth]{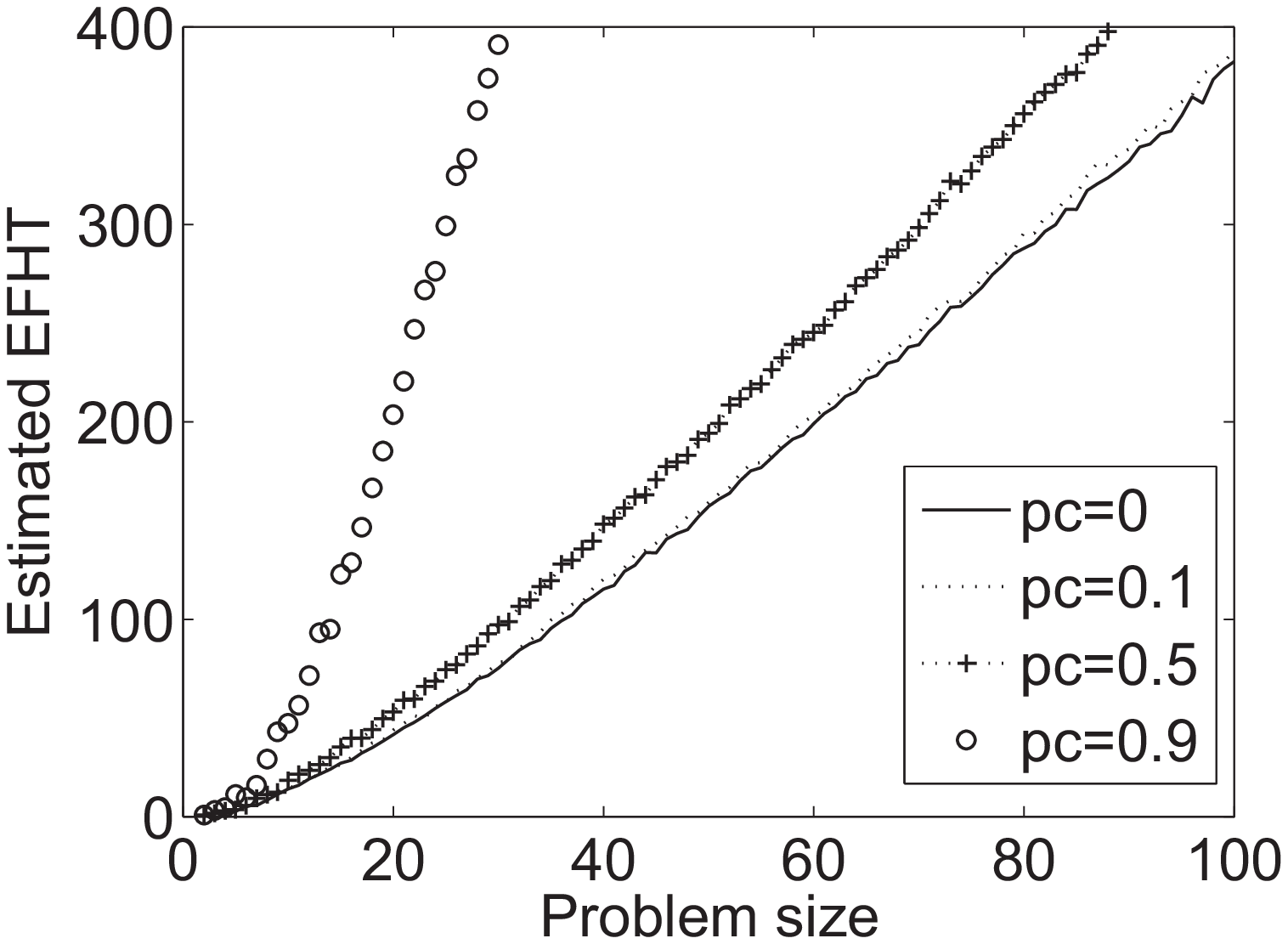}
\end{minipage}\\
\begin{minipage}[c]{0.45\linewidth}\centering
    \small(a) On LeadingOnes problem
\end{minipage}
\begin{minipage}[c]{0.45\linewidth}\centering
    \small(b) On OneMax problem
\end{minipage}\\\vspace{-1em}
\caption{Comparison of the estimated EFHT of (2:2)-EA with one-bit crossover and one-bit mutation.}\label{fig_efht}
\end{figure*}
\begin{figure*}[t!]\centering
\begin{minipage}[c]{0.45\linewidth}\centering
        \includegraphics[width=1\linewidth,height=0.75\linewidth]{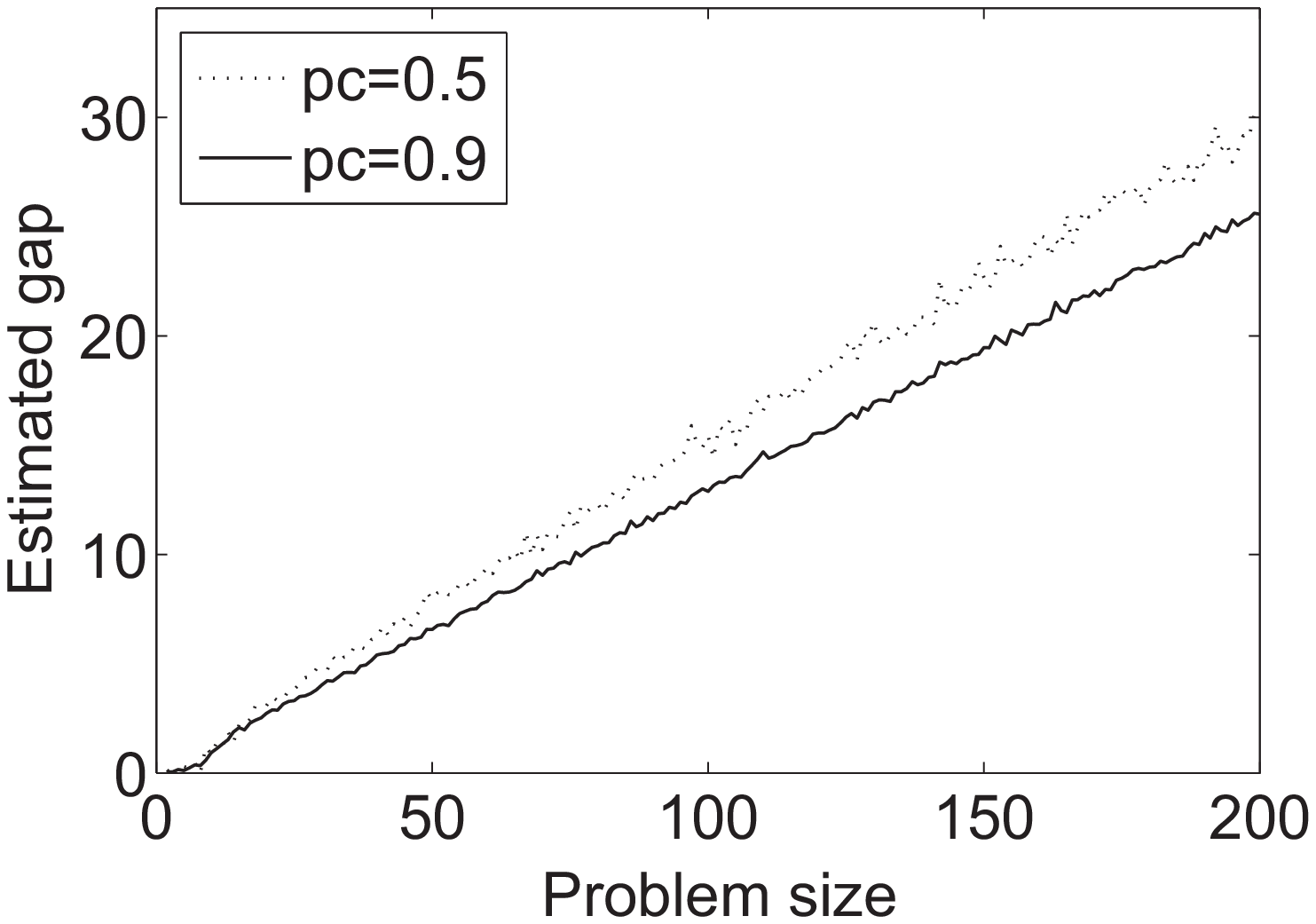}
\end{minipage}
\begin{minipage}[c]{0.45\linewidth}\centering
        \includegraphics[width=1\linewidth,height=0.75\linewidth]{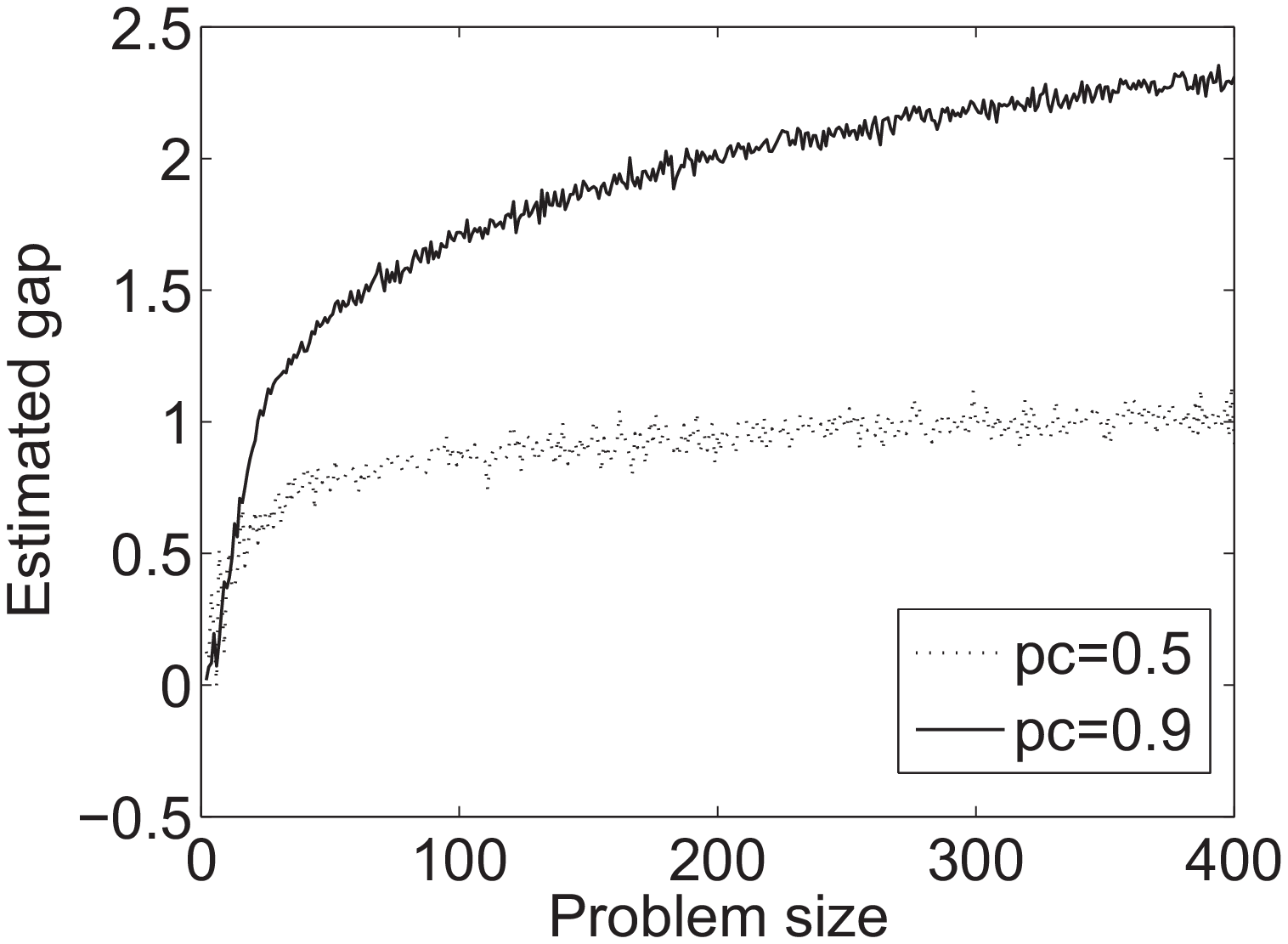}
\end{minipage}
\begin{minipage}[c]{0.45\linewidth}\centering
    \small(a) On LeadingOnes problem
\end{minipage}
\begin{minipage}[c]{0.45\linewidth}\centering
    \small(b) On OneMax problem
\end{minipage}\\\vspace{-1em}
\caption{Estimated \emph{gap} of (2:2)-EA with one-bit crossover and one-bit mutation.}\label{fig_gap}
\end{figure*}

\begin{figure*}[t!]\centering
\begin{minipage}[c]{0.45\linewidth}\centering
        \includegraphics[width=1\linewidth]{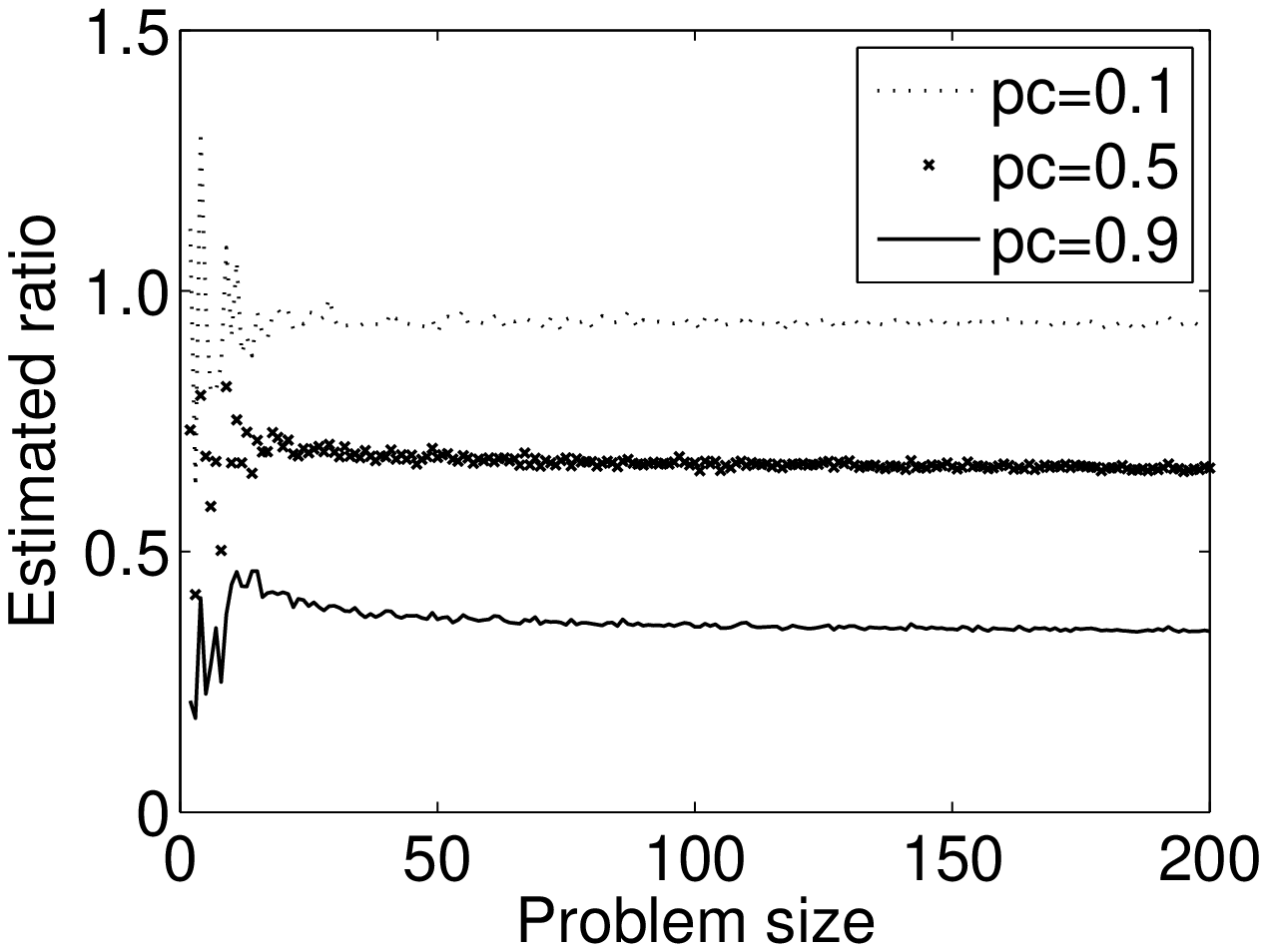}
\end{minipage}
\begin{minipage}[c]{0.45\linewidth}\centering
        \includegraphics[width=1\linewidth]{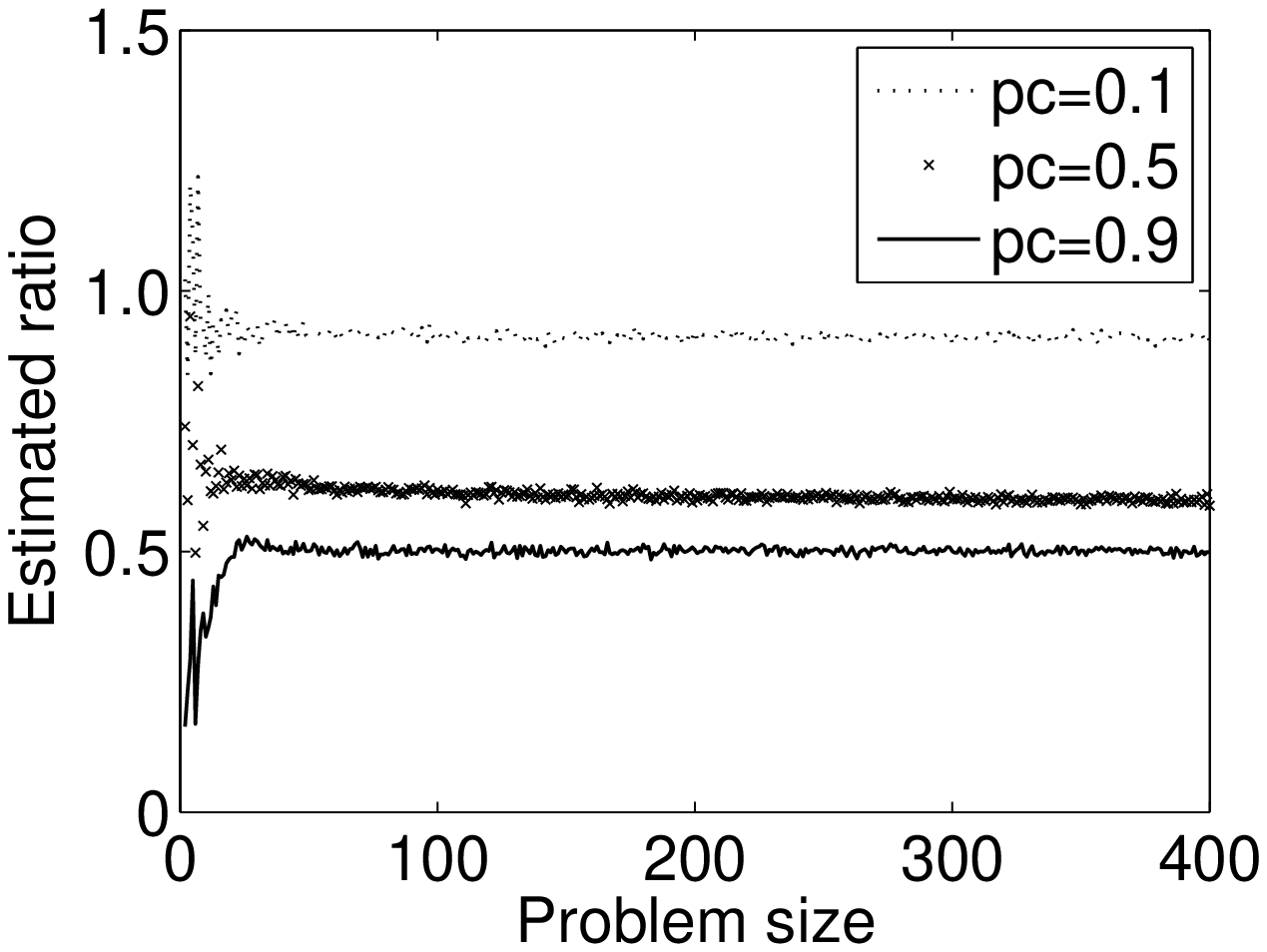}
\end{minipage}
\begin{minipage}[c]{0.45\linewidth}\centering
    \small(a) On LeadingOnes problem
\end{minipage}
\begin{minipage}[c]{0.45\linewidth}\centering
    \small(b) On OneMax problem
\end{minipage}\\\vspace{-1em}
\caption{Estimated \emph{ratio} of (2:2)-EA with one-bit crossover and one-bit mutation.}\label{fig_ratio}
\end{figure*}


We then investigate the crossover strategies. Figure \ref{fig_mr} plots the curves of estimated EFHTs of the (2:2)-EA with different operators. It can be observed that, on the LeadingOnes problem, both the curves of M\&R1a and M\&R1b (the curves are overlapped) are consistently better than the mutation only curve. Since M\&R1a and M\&R1b apply the crossover in different conditions, their combination M\&R1 is better than either M\&R1a or M\&R1b. It also shows that M\&R2, which uses first-diff-point crossover to exchange a part of the solutions at a time, is better than M\&R1, which exchanges one bit at a time. We can also observed that the M\&R3 designed for the OneMax problem performs well than mutation only.

\begin{figure*}[t!]\centering
\begin{minipage}[c]{0.45\linewidth}\centering
        \includegraphics[width=1\linewidth]{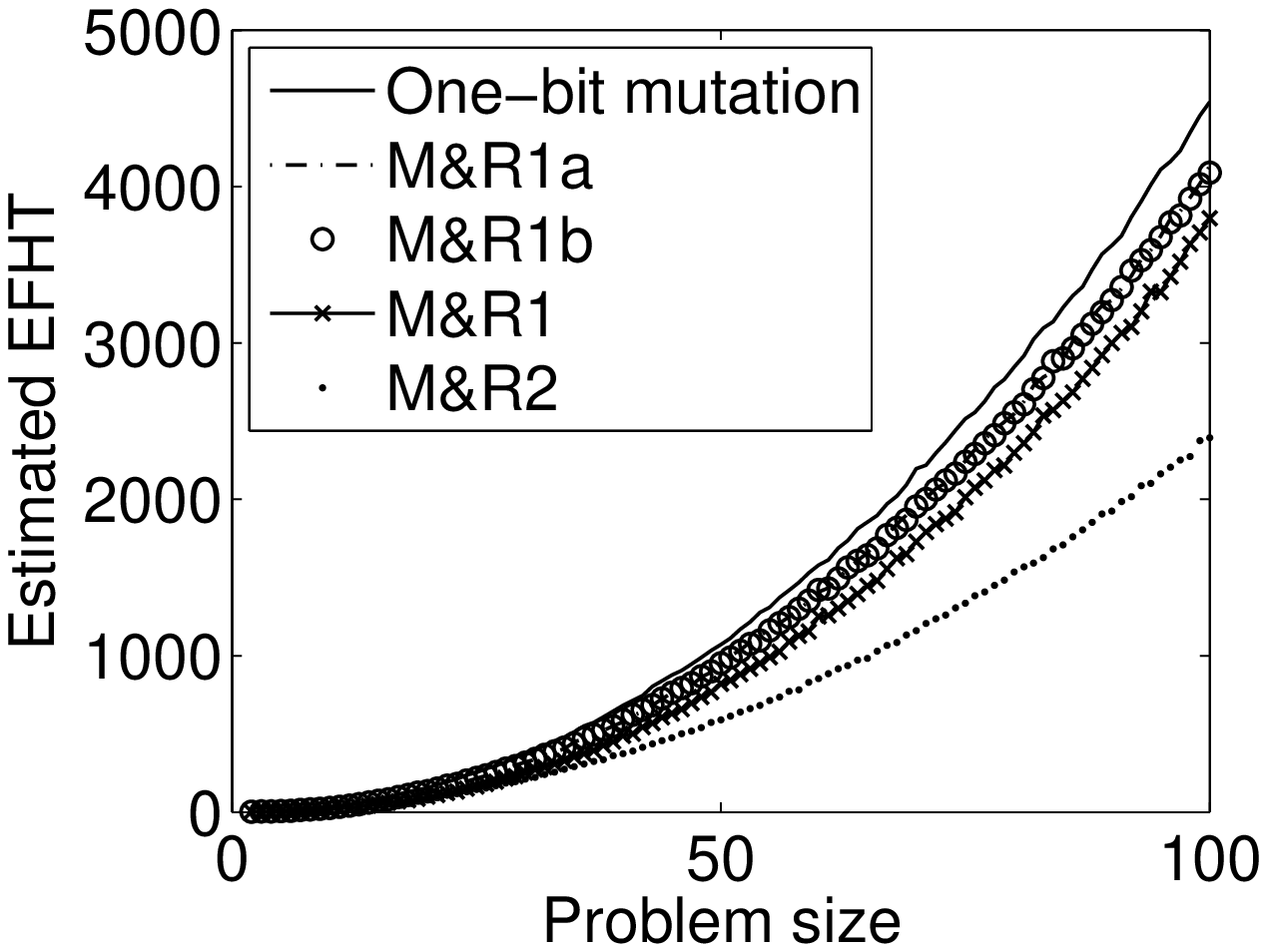}
\end{minipage}
\begin{minipage}[c]{0.45\linewidth}\centering
        \includegraphics[width=1\linewidth]{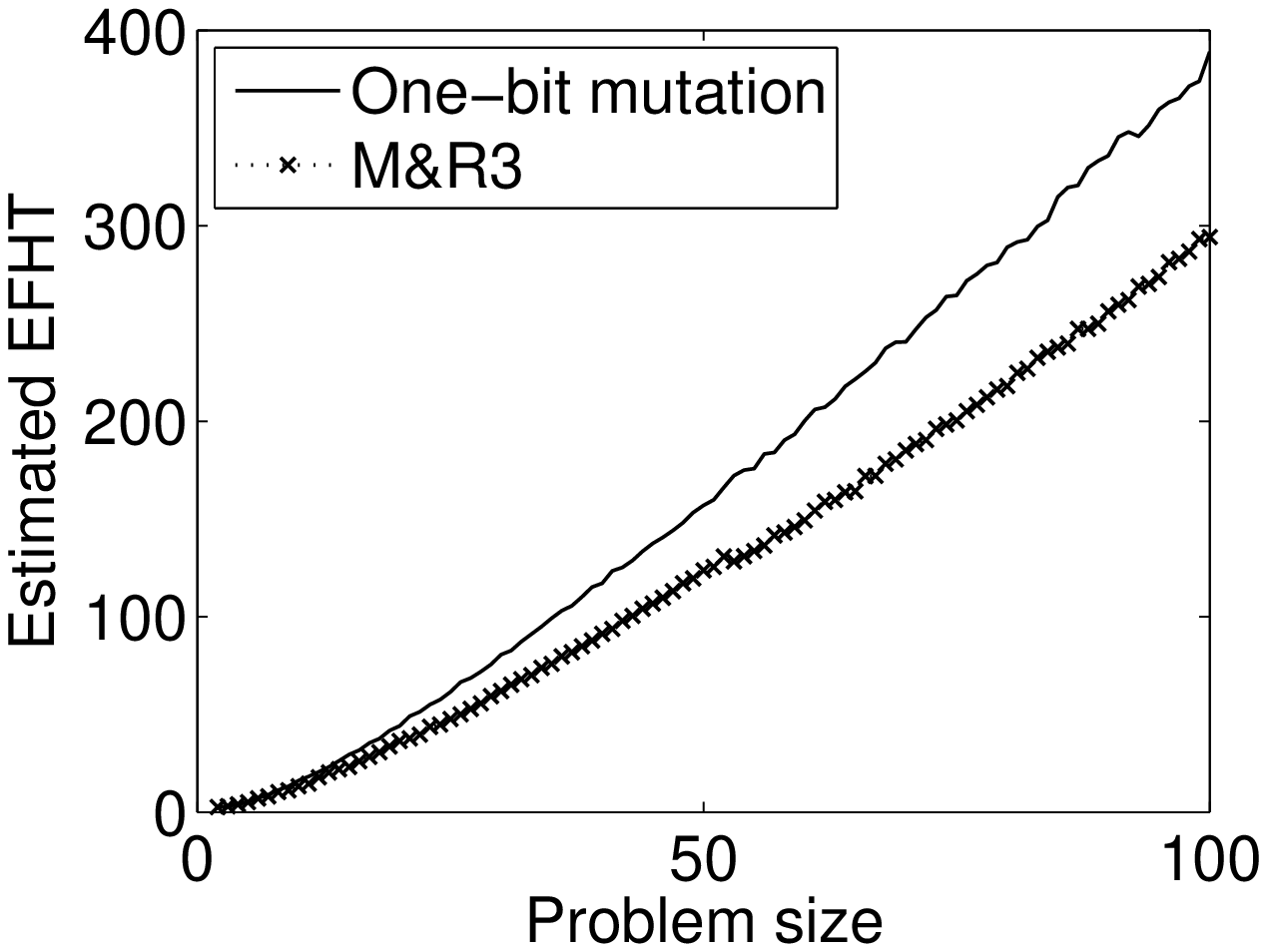}
\end{minipage}
\begin{minipage}[c]{0.45\linewidth}\centering
    \small(a) On LeadingOnes problem
\end{minipage}
\begin{minipage}[c]{0.45\linewidth}\centering
    \small(b) On OneMax problem
\end{minipage}\\\vspace{-1em}
\caption{Estimated EFHT of (2:2)-EA with one-bit mutation and crossover strategies.}\label{fig_mr}
\end{figure*}

\section{Conclusion}

This paper extends our preliminary research \cite{yuy_qianc_zhouzh_ppsn10}. Due to the irregularity and complex interactions to mutation and selection operators, crossover operators are hard to be analyzed. In this paper, we propose the \emph{General Markov Chain Switching Theorem} to facilitate the analysis of crossover. Instead of directly analyze an EA, GMCST compares two EAs. In the comparison, the long-term behavior of one of the two EAs is not required, but only need to analyze the one-step transition behavior. Therefore, by GMCST we can compare a crossover-enabled EA with an easy-to-analysis EA, so that we can analysis the crossover-enabled EA without touching its long-term behavior.

Using GMCST, we analyze several crossover-enabled EAs on the LeadingOnes and the OneMax problems, which are notably two model problems well-studied for mutation-only EAs but have few results for crossover. Our analysis is in three aspects. In the first aspect that is to bound the asymptotic running time of crossover-enabled EAs, we derive upper and lower bound of the (2+2)-EA with several crossover operators. In the second aspect that is to compare the running time of crossover-enabled EAs with their counterpart mutation-only EAs, our analysis shows that the crossover operator slows down the (2:2)-EA, but within a bounded factor. The reason that the crossover is not helpful here may be that the EA dose not create a good diversity among solution on the two uni-model simple problems. In the third aspect that is to use crossover smartly, we design several strategies that use crossover only when necessary, and we prove the effectiveness of the strategies.  These theoretical results are then verified by experiments.

Although we use the GMCST to analyze crossover in this paper, it is worth noting that GMCST is a general tool that can be used for analyzing other sophisticated metaheuristic algorithms, which appear often in real-world applications. We will exploit the power of GMCST in the future.

\section{Acknowledgments}

to be added...

\bibliography{ectheory}
\bibliographystyle{abbrvnat}

\appendix
\section{Proofs for Section \ref{sec:runtime}}\label{appendix_1}

\begin{myproofd}{Proposition \ref{1+1_EA'_LeadingOnes}}
For any solution with $j(1 \leq j \leq n)$ 0's, the number of 0's will never increase since the selection of the (1+1$_>$)-EA, and in one mutation step, a solution with $j-1$ 0's will be generated with probability $\frac{1}{n}$, otherwise the solution keeps unchanged. Thus, the expected steps to decrease the number of 0's by 1 is $n$. Then, the expected running time to get to the optimal solution from a solution with $j$ 0's is $jn$.
\end{myproofd}

\begin{myproofd}{Proposition \ref{1+1_EA_OneMax}}
For any solution with $j(1 \leq j \leq n)$ 0's, the number of 0's will never increase since the selection of the (1+1)-EA, and in one mutation step, a solution with $j-1$ 0's will be generated with probability $\frac{j}{n}$, otherwise the solution keeps unchanged. Thus, the expected steps to decrease the number of 0's by 1 is $\frac{n}{j}$ for the solutions with $j$ 0's. Then, the expected running time to get to the optimal solution from a solution with $j$ 0's is $\sum_{i=1}^j \frac{n}{i} = nH_j$.
\end{myproofd}

 \section{Proofs for Section \ref{sec:compare}}\label{appendix_2}

 \begin{myproofd}{Proposition \ref{prop_transition_LO}}
For the (2:2)-EA with mutation only (i.e., $\{\xi'_t\}^{\infty}_{t=0}$), it behaves like two independent (1+1)-EA which only accepts better solutions. Thus, for any non-optimal solution, the offspring will be accepted only if the first 0 bit is mutated, the probability of which is $\frac{1}{n}$; otherwise, the solution will keep unchanged. Thus, the above one-step transition behavior for $\{\xi'_t\}^{\infty}_{t=0}$ trivially holds.

For the (2:2)-EA with crossover (i.e., $\{\xi_t\}^{\infty}_{t=0}$), it uses crossover with probability $p_c$ in every reproduction step; otherwise, it uses mutation. Since only better solutions will be accepted in the selection procedure, the offspring which flips the first 0 bit of the parent will replace the parent, otherwise, the parent will not change. If it uses mutation, it behaves like $\{\xi'_t\}^{\infty}_{t=0}$. If it uses crossover, we consider two cases in terms of the number of leading ones of the two solutions. In the case of $LO_1=LO_2$, one-bit crossover will never generate offspring solutions which flip the first 0 bits of parent solutions, since the first 0 bits of the two parent solutions are in the same position. In the case of $LO_1<LO_2$, if the crossover position is $LO_1+1$, which happens with probability $\frac{1}{n}$, it will flip the first 0 bit of $s_1$ and $s_2$ will not change, since $s_1(LO_1+1)=0$ and $s_2(LO_1+1)=1$; if the crossover position is $LO_2+1$, which also happens with probability $\frac{1}{n}$, since $s_2(LO_2+1)=0$, it will flip the first 0 bit of the second solution $s_2$ and the first solution $s_1$ will not change in the case of $s_1(LO_2+1)=1$, and the two parent solutions will not change in the case of $s_1(LO_2+1)=0$; otherwise, both the two parent solutions will not change. Thus, the above one-step transition behavior of
$\{\xi_t\}^{\infty}_{t=0}$ trivially holds.
\end{myproofd}

\begin{myproofd}{Proposition \ref{prop_cfht1_LO}}
        Since the population with at least one solution $1^n$ is optimal, it
        is trivial that $\mathbb{E}(i,0)=0$. By the one step transition
        behavior of (2:2)-EA with one-bit mutation only on the LeadingOnes
        problem analyzed in the Proposition \ref{prop_transition_LO}, we
        have $\forall i,j \geq 1,$
        \begin{align}
        &\mathbb{E}(i,j)=1+\frac{n-1}{n^2}\mathbb{E}(i-1,j)+\frac{n-1}{n^2}\mathbb{E}(i,j-1)+\frac{1}{n^2}\mathbb{E}(i-1,j-1)+\frac{(n-1)^2}{n^2}\mathbb{E}(i,j).
        \end{align}
        Thus, $$
        \mathbb{E}(i,j)=\frac{n^2}{2n-1}+\frac{n-1}{2n-1}\mathbb{E}(i-1,j)+\frac{n-1}{2n-1}\mathbb{E}(i,j-1)+\frac{1}{2n-1}\mathbb{E}(i-1,j-1).
        $$
\end{myproofd}

\begin{myproofd}{Proposition \ref{prop_cfht2_LO}}
    We prove the proposition by induction on $i+i+\delta$. All the equalities below are not hard to derive by analyzing the CFHT of $\{\xi'_t\}^{\infty}_{t=0}$ in Proposition \ref{prop_cfht1_LO}, and all the inequalities can be derived by inductive hypothesis.\\
    {\bf (a) Initialization}:\\
        When $i+i+\delta=2$, we have $\mathbb{E}(1,1)-\mathbb{E}(0,1)=\frac{n^2}{2n-1} \leq n-\frac{3n-1}{8}$.\\
        When $i+i+\delta=3$, we have $\mathbb{E}(1,2)-\mathbb{E}(1,1)=\frac{n^2(n-1)}{(2n-1)^2}  \geq \frac{n}{8},\mathbb{E}(1,2)-\mathbb{E}(0,2)=\frac{(3n-2)n^2}{(2n-1)^2}  \leq n-\frac{3n-1}{16}$. \\
    {\bf (b) Inductive Hypothesis}: Assume \\
        for $2 \leq i+i+\delta \leq k$, we have
            \begin{align}
                & \forall i \geq 1, \delta \geq 1:\mathbb{E}(i,i+\delta)-\mathbb{E}(i,i+\delta-1) \geq \frac{n}{2^{\delta+2}}, \\
                & \forall i \geq 1, \delta \geq 0:\mathbb{E}(i,i+\delta)-\mathbb{E}(i-1,i+\delta) \leq n-\frac{(3n-1)}{2^{\delta+3}}.
            \end{align}

        We then prove that the two inequations still hold when $i+i+\delta=k+1$.\\
        First, we consider $\mathbb{E}(i,i+\delta)-\mathbb{E}(i,i+\delta-1)$.\\
        (1) When $i=1$, we have $\mathbb{E}(1,k)-\mathbb{E}(1,k-1) = \frac{n^2(n-1)^{k-1}}{(2n-1)^k} \geq \frac{n}{2^{k+1}}$. \\
        (2) When $i>1$, we have
            \begin{align}
                & \forall i > 1, \delta \geq 1:\mathbb{E}(i,i+\delta)-\mathbb{E}(i,i+\delta-1)\\
                & =\frac{n^2}{2n-1}+\frac{1}{2n-1}\mathbb{E}(i-1,i+\delta-1)+\frac{n-1}{2n-1}\mathbb{E}(i-1,i+\delta)-\frac{n}{2n-1}\mathbb{E}(i,i+\delta-1)\\
                & \geq \frac{n^2}{2n-1}-\frac{n-1}{2n-1}\mathbb{E}(i-1,i+\delta-1)+\frac{n-1}{2n-1}\mathbb{E}(i-1,i+\delta)-\frac{n}{2n-1}(n-\frac{(3n-1)}{2^{\delta+2}})\\
                & \geq \frac{n^2}{2n-1}+\frac{n-1}{2n-1}\frac{n}{2^{\delta+3}}-\frac{n}{2n-1}(n-\frac{(3n-1)}{2^{\delta+2}})\\
                & \geq \frac{n}{2^{\delta+2}},
            \end{align}
        where the first inequality is by $\mathbb{E}(i,i+\delta-1)-\mathbb{E}(i-1,i+\delta-1)\leq
        n-\frac{(3n-1)}{2^{\delta+2}}$ since $i+i+\delta-1=k \wedge i \geq 2 \wedge i+\delta-1-i \geq
        0$, and the second inequality is by $\mathbb{E}(i-1,i+\delta)-\mathbb{E}(i-1,i+\delta-1) \geq \frac{n}{2^{\delta+3}}$ since $i-1+i+\delta=k \wedge i-1 \geq 1 \wedge i+\delta-i+1 \geq 1$.

        Second, we consider  $\mathbb{E}(i,i+\delta)-\mathbb{E}(i-1,i+\delta)$.\\
        (1) When $i=1$, we have $\cexpect{1,k}-\cexpect{0,k}=n-\frac{n(n-1)^k}{(2n-1)^k} \leq n-\frac{3n-1}{2^{k+2}}$.\\
        (2) When $i>1$ and $\delta=0$, we have
            \begin{align}
                & \forall i>1:\mathbb{E}(i,i)-\mathbb{E}(i-1,i)\\
                &=\frac{n^2}{2n-1}+\frac{2(n-1)}{2n-1}\mathbb{E}(i-1,i)+\frac{1}{2n-1}\mathbb{E}(i-1,i-1)-\mathbb{E}(i-1,i)\\
                &=\frac{n^2}{2n-1}-\frac{1}{2n-1}\mathbb{E}(i-1,i)+\frac{1}{2n-1}\mathbb{E}(i-1,i-1)\\
                &\leq \frac{n^2}{2n-1}-\frac{n}{8(2n-1)}\\
                &\leq n-\frac{3n-1}{8},
            \end{align}
            where the first inequality is by $\mathbb{E}(i-1,i)-\mathbb{E}(i-1,i-1)\geq \frac{n}{8}$ since $i-1+i=k \wedge i-1 \geq 1 \wedge i-i+1 \geq 1$.\\
        (3) When $i>1$ and $\delta>0$, we have
            \begin{align}
                & \forall i>1:\mathbb{E}(i,i+\delta)-\mathbb{E}(i-1,i+\delta)\\
                & =\frac{n^2}{2n-1}+\frac{1}{2n-1}\mathbb{E}(i-1,i+\delta-1)+\frac{n-1}{2n-1}\mathbb{E}(i,i+\delta-1)-\frac{n}{2n-1}\mathbb{E}(i-1,i+\delta)\\
                & \leq \frac{n^2}{2n-1}+\frac{n-1}{2n-1}\mathbb{E}(i,i+\delta-1)-\frac{n-1}{2n-1}\mathbb{E}(i-1,i+\delta-1)-\frac{n^2}{(2n-1)2^{\delta+3}}\\
                & \leq \frac{n^2}{2n-1}+\frac{n-1}{2n-1}(n-\frac{3n-1}{2^{\delta+2}})-\frac{n^2}{(2n-1)2^{\delta+3}}\\
                & \leq n-\frac{3n-1}{2^{\delta+3}},
            \end{align}
        where the first inequality is by $\mathbb{E}(i-1,i+\delta)-\mathbb{E}(i-1,i+\delta-1)\geq \frac{n}{2^{\delta+3}}$ since $i-1+i+\delta=k \wedge i-1 \geq 1 \wedge i+\delta-i+1 \geq 1$, and the second inequality
        is by $\mathbb{E}(i,i+\delta-1)-\mathbb{E}(i-1,i+\delta-1)\leq n-\frac{3n-1}{2^{\delta+2}}$ since $i+i+\delta-1=k \wedge i \geq 2 \wedge i+\delta-1-i \geq
        0$.

    {\bf (c) Conclusion}: According to (a) and (b), we have
            \begin{align}
                & \forall i \geq 1, \delta \geq 1:\mathbb{E}(i,i+\delta)-\mathbb{E}(i,i+\delta-1)  \geq \frac{n}{2^{\delta+2}},\\
                & \forall i \geq 1, \delta \geq 0:\mathbb{E}(i,i+\delta)-\mathbb{E}(i-1,i+\delta) \leq n-\frac{3n-1}{2^{\delta+3}}.
            \end{align}
\end{myproofd}

\begin{myproofd}{Proposition \ref{prop_equalLO}}
        We prove this proposition by induction on $t$.\\
        {\bf{(a) Initialization}} For $t=0$, we have
        \begin{align}
            & \pi_0(LO_1=LO_2<n)\\
            &=\frac{1}{2}\cdot \frac{1}{2} + \frac{1}{2^2} \cdot \frac{1}{2^2} +\cdots+\frac{1}{2^{n-1}} \cdot \frac{1}{2^{n-1}}+ \frac{1}{2^{n}} \cdot \frac{1}{2^{n}}\\
            & = \frac{1}{3}-\frac{1}{3 \cdot 4^n},
        \end{align}
        where the $i$-th term $\frac{1}{2^i} \cdot \frac{1}{2^i}$ in the
        first equality is the probability that $LO_1=LO_2=i-1$, which can be easily derived since the initial distribution $\pi_0$ is uniform. \\
        {\bf{(b) Inductive Hypothesis}} Assume that for $0 \leq t \leq K-1$,
        \begin{align}
         & \pi_t(LO_1=LO_2 <n) \geq (\frac{1}{3}-\frac{1}{3 \cdot 4^n})(p_c+(1-p_c)(1-\frac{1}{n})^2)^t.
        \end{align}
         Then, for $t=K$, we have
         \begin{align}
         & \pi_K(LO_1=LO_2 <n)\\
         & \geq \pi_{K-1}(LO_1=LO_2 <n) \cdot (p_c+(1-p_c)(1-\frac{1}{n})^2)\\
         & \geq (\frac{1}{3}-\frac{1}{3 \cdot 4^n})(p_c+(1-p_c)(1-\frac{1}{n})^2)^{K-1} \cdot (p_c+(1-p_c)(1-\frac{1}{n})^2)\\
         &=  (\frac{1}{3}-\frac{1}{3 \cdot 4^n})(p_c+(1-p_c)(1-\frac{1}{n})^2)^{K},
         \end{align}
         where the first inequality is since one-bit crossover and one-bit mutation which does not flip the first the 0 bit of the parent solution will keep
         $LO_1=LO_2$ in the next population when the current population satisfies that $LO_1=LO_2$, and the second inequality is by inductive hypothesis.

         {\bf{(c) Conclusion}} From (a) and (b), the proposition holds.
\end{myproofd}

 \begin{myproofd}{Proposition \ref{prop_transition_OneMax}}
    For the (2:2)-EA with mutation only (i.e.,
    $\{\xi'_t\}^{\infty}_{t=0}$), it behaves like two independent
    (1+1)-EA which only accepts better solutions. Thus, for any
    non-optimal solution, the offspring will be accepted only if one 0
    bit is mutated, the probability of which is $\frac{i}{n}$ for
    one-bit mutation on a solution with $i$ number of 0's; otherwise,
    the solution will keep unchanged. Thus, the above one-step
    transition behavior for $\{\xi'_t\}^{\infty}_{t=0}$ trivially holds.

    For the (2:2)-EA with crossover (i.e., $\{\xi_t\}^{\infty}_{t=0}$),
    it uses crossover with probability $p_c$ in every reproduction step;
    otherwise, it uses mutation. Since only better solutions will be
    accepted in the selection procedure, the offspring which flips one 0
    bit of the parent through either one-bit mutation or one-bit
    crossover will replace the parent, otherwise, the parent will not
    change. If it uses mutation, it behaves like
    $\{\xi'_t\}^{\infty}_{t=0}$. If it uses crossover, we consider the
    position of the crossover point, briefly denoted as {\it cp}. If
    $s_1({\it cp})=0 \wedge s_2({\it cp}=1)$, which happens with
    probability $\frac{k}{n}$, it will flip one 0 bit of $s_1$ and $s_2$
    will not change; if $s_1({\it cp})=1 \wedge s_2({\it cp}=0)$, which
    happens with probability $\frac{j-i+k}{n}$, it will flip one 0 bit
    of the second solution $s_2$ and the first solution $s_1$ will not
    change; otherwise, both the two parent solutions will not change,
    since one-bit crossover on the position where the two parent
    solution have a same bit will not introduce new bits for any
    solution. Thus, the above one-step transition behavior of
    $\{\xi_t\}^{\infty}_{t=0}$ trivially holds.
\end{myproofd}

\begin{myproofd}{Proposition \ref{prop_cfht1_OneMax}}
    Since the population with at least one solution $1^n$ is optimal, it
    is trivial that $\mathbb{E}(i,0)=0$. By the one step transition
    behavior of (2:2)-EA with one-bit mutation only on the OneMax
    problem analyzed in the Proposition \ref{prop_transition_OneMax}, we
    have $\forall i,j \geq 1,$
    $$
    \mathbb{E}(i,j)=1+\frac{ij}{n^2}\mathbb{E}(i-1,j-1)+\frac{i(n-j)}{n^2}\mathbb{E}(i-1,j)+\frac{(n-i)j}{n^2}\mathbb{E}(i,j-1)+\frac{(n-i)(n-j)}{n^2}\mathbb{E}(i,j).
    $$
    Thus,
    $$
    \mathbb{E}(i,j)=\frac{n^2}{(i+j)n-ij}+\frac{ij}{(i+j)n-ij}\mathbb{E}(i-1,j-1)+\frac{i(n-j)}{(i+j)n-ij}\mathbb{E}(i-1,j)+\frac{(n-i)j}{(i+j)n-ij}\mathbb{E}(i,j-1).
    $$
\end{myproofd}

\begin{myproofd}{Proposition \ref{prop_cfht2_OneMax}}
 We prove the proposition by induction on $i+i+\delta$. All the equalities below are not hard to derive by analyzing the CFHT of $\{\xi'_t\}^{\infty}_{t=0}$ in Proposition \ref{prop_cfht1_OneMax}, and all the inequalities can be derived by inductive hypothesis.\\
    {\bf (a) Initialization}:\\
        When $i+i+\delta=2$, we have $\mathbb{E}(1,1)-\mathbb{E}(0,1)=\frac{n^2}{2n-1} <\frac{5n+1}{8}$.\\
        When $i+i+\delta=3$, we have $\mathbb{E}(1,2)-\mathbb{E}(1,1)=\frac{n^2(n-1)}{(2n-1)(3n-2)} > \frac{n}{8},\mathbb{E}(1,2)-\mathbb{E}(0,2)=\frac{(4n-3)n^2}{(2n-1)(3n-2)}  < \frac{13n+1}{16}$. \\
    {\bf (b) Inductive Hypothesis}: Assume \\
        for $2 \leq i+i+\delta \leq k$, we have
            \begin{align}
                & \forall i \geq 1, \delta \geq 1:\mathbb{E}(i,i+\delta)-\mathbb{E}(i,i+\delta-1) > \frac{n}{2^{\delta+1}(i+\delta)}, \\
                & \forall i \geq 1, \delta \geq 0:\mathbb{E}(i,i+\delta)-\mathbb{E}(i-1,i+\delta) < (1-\frac{3}{2^{\delta+3}})\frac{n}{i}+\frac{1}{2^{\delta+3}}.
            \end{align}
        We then prove that the two inequations still hold when $i+i+\delta=k+1$.\\
        First, we consider $\mathbb{E}(i,i+\delta)-\mathbb{E}(i,i+\delta-1)$.\\
        (1) When $i=1$, $\cexpect{1,k} < n-\frac{(k+2)n-(k+1)}{(k+1)2^{k+1}}$, which can be easily proved by induction on $k(k \geq 1)$.\\
            Then, we have $\forall k\geq2: \cexpect{1,k}-\cexpect{1,k-1}=\frac{n^2-n\cexpect{1,k-1}}{(k+1)n-k}> \frac{n}{k2^k}$.\\
        (2) When $i>1$, we have
        \begin{align}
            & \forall i > 1, \delta \geq 1,\mathbb{E}(i,i+\delta)-\mathbb{E}(i,i+\delta-1)\\
            & = \frac{n^2}{(2i+\delta)n-i(i+\delta)}+\frac{i(i+\delta)}{(2i+\delta)n-i(i+\delta)}\mathbb{E}(i-1,i+\delta-1)\\
            & \quad +\frac{i(n-i-\delta)}{(2i+\delta)n-i(i+\delta)}\mathbb{E}(i-1,i+\delta)-\frac{in}{(2i+\delta)n-i(i+\delta)}\mathbb{E}(i,i+\delta-1)\\
            & > \frac{n^2}{(2i+\delta)n-i(i+\delta)}-\frac{i(n-i-\delta)}{(2i+\delta)n-i(i+\delta)}\mathbb{E}(i-1,i+\delta-1)\\
            & \quad +\frac{i(n-i-\delta)}{(2i+\delta)n-i(i+\delta)}\mathbb{E}(i-1,i+\delta)-\frac{in}{(2i+\delta)n-i(i+\delta)}((1-\frac{3}{2^{\delta+2}})\frac{n}{i}+\frac{1}{2^{\delta+2}})\\
            & > \frac{n^2}{(2i+\delta)n-i(i+\delta)}+\frac{i(n-i-\delta)}{(2i+\delta)n-i(i+\delta)}\frac{n}{(i+\delta)2^{\delta+2}}\\
            & \quad -\frac{in}{(2i+\delta)n-i(i+\delta)}((1-\frac{3}{2^{\delta+2}})\frac{n}{i}+\frac{1}{2^{\delta+2}})\\
            & >\frac{n}{(i+\delta)2^{\delta+1}},
        \end{align}
        where the first inequality is by $ \mathbb{E}(i,i+\delta-1)-\mathbb{E}(i-1,i+\delta-1)< (1-\frac{3}{2^{\delta+2}})\frac{n}{i}+\frac{1}{2^{\delta+2}}$ since $i+i+\delta-1=k \wedge i \geq 2 \wedge i+\delta-1-i \geq
        0$, and the second inequality is by $\mathbb{E}(i-1,i+\delta)-\mathbb{E}(i-1,i+\delta-1)>\frac{n}{(i+\delta)2^{\delta+2}}$ since $i-1+i+\delta=k \wedge i-1 \geq 1 \wedge i+\delta-i+1 \geq 1$.

        Second, we consider  $\mathbb{E}(i,i+\delta)-\mathbb{E}(i-1,i+\delta)$.\\
        (1) When $i=1$, we have $\forall k \geq 1: \cexpect{1,k}-\cexpect{0,k}=\cexpect{1,k} < (1-\frac{3}{2^{k+2}})\frac{n}{i}+\frac{1}{2^{k+2}}$, which can be proved by induction on $k$.\\
        (2) When $i>1$ and $\delta=0$, we have
        \begin{align}
            & \forall i>1:\mathbb{E}(i,i)-\mathbb{E}(i-1,i)\\
            &=\frac{n^2}{2in-i^2}-\frac{i^2}{2in-i^2}\mathbb{E}(i-1,i)+\frac{i^2}{2in-i^2}\mathbb{E}(i-1,i-1)\\
            &< \frac{n^2}{2in-i^2}-\frac{ni^2}{4i(2in-i^2)}\\
            &\leq \frac{5n+i}{8i},
        \end{align}
        where the first inequality is by $\mathbb{E}(i-1,i)-\mathbb{E}(i-1,i-1)> \frac{n}{4i}$ since $i-1+i=k \wedge i-1 \geq 1 \wedge i-i+1 \geq 1$.\\
        (3) When $i>1$ and $\delta>0$, we have
        \begin{align}
            & \forall i>1:\mathbb{E}(i,i+\delta)-\mathbb{E}(i-1,i+\delta)\\
            & = \frac{n^2}{(2i+\delta)n-i(i+\delta)}+\frac{i(i+\delta)}{(2i+\delta)n-i(i+\delta)}\mathbb{E}(i-1,i+\delta-1)\\
            & \quad +\frac{(n-i)(i+\delta)}{(2i+\delta)n-i(i+\delta)}\mathbb{E}(i,i+\delta-1)-\frac{(i+\delta)n}{(2i+\delta)n-i(i+\delta)}\mathbb{E}(i-1,i+\delta)\\
            & < \frac{n^2}{(2i+\delta)n-i(i+\delta)}-\frac{(n-i)(i+\delta)}{(2i+\delta)n-i(i+\delta)}\mathbb{E}(i-1,i+\delta-1)\\
            & \quad +\frac{(n-i)(i+\delta)}{(2i+\delta)n-i(i+\delta)}\mathbb{E}(i,i+\delta-1)-\frac{n^2}{((2i+\delta)n-i(i+\delta))2^{\delta+2}}\\
            & <\frac{n^2}{(2i+\delta)n-i(i+\delta)}+\frac{(n-i)(i+\delta)}{(2i+\delta)n-i(i+\delta)}((1-\frac{3}{2^{\delta+2}})\frac{n}{i}\\
            & \quad +\frac{1}{2^{\delta+2}})-\frac{n^2}{((2i+\delta)n-i(i+\delta))2^{\delta+2}}\\
            & <(1-\frac{3}{2^{\delta+3}})\frac{n}{i}+\frac{1}{2^{\delta+3}},
        \end{align}
        where the first inequality is by $\mathbb{E}(i-1,i+\delta)-\mathbb{E}(i-1,i+\delta-1)>\frac{n}{(i+\delta)2^{\delta+2}}$ since $i-1+i+\delta=k \wedge i-1 \geq 1 \wedge i+\delta-i+1 \geq
        1$, and the second inequality is by $\mathbb{E}(i,i+\delta-1)-\mathbb{E}(i-1,i+\delta-1)<(1-\frac{3}{2^{\delta+2}})\frac{n}{i}+\frac{1}{2^{\delta+2}}$ since $i+i+\delta-1=k \wedge i \geq 2 \wedge i+\delta-1-i \geq
        0$.

    {\bf (c) Conclusion}: According to (a) and (b), we have
            \begin{align}
                &\forall i \geq 1, \delta \geq 1:  \cexpect{i,i+\delta}-\cexpect{i,i+\delta-1}>\frac{n}{2^{\delta+1}(i+\delta)} ,  \\
                &\forall i \geq 1, \delta \geq 0:  \cexpect{i,i+\delta}-\cexpect{i-1,i+\delta} <(1-\frac{3}{2^{\delta+3}})\frac{n}{i}+\frac{1}{2^{\delta+3}}.
            \end{align}
 \end{myproofd}

 \begin{myproofd}{Proposition \ref{prop_cfht3_OneMax}}
 We prove the proposition by induction on $i+i+\delta$. All the equalities below are not hard to derive by analyzing the CFHT of $\{\xi'_t\}^{\infty}_{t=0}$ in Proposition \ref{prop_cfht1_OneMax}, and all the inequalities can be derived by inductive hypothesis.\\
    {\bf (a) Initialization}:\\
        When $i+i+\delta=1$, we have $\cexpect{0,1}-\cexpect{0,0}=0 < \frac{n}{2}$. \\
        When $i+i+\delta=2$, we have $\cexpect{0,2}-\cexpect{0,1}=0 < \frac{n}{4}$, $\cexpect{1,1}-\cexpect{0,1}=\frac{n^2}{2n-1} > \frac{n}{2}$. \\
    {\bf (b) Inductive Hypothesis}: Assume \\
        for $1 \leq i+i+\delta \leq k$, we have
            \begin{align}
                & \forall i \geq 0, \delta \geq 1: \cexpect{i,i+\delta}-\cexpect{i,i+\delta-1} < \frac{n}{2(i+\delta)}, \\
                & \forall i \geq 1, \delta \geq 0: \cexpect{i,i+\delta}-\cexpect{i-1,i+\delta} > \frac{n}{2i}.
            \end{align}
        We then prove that the two inequations still hold when $i+i+\delta=k+1$ .\\
        First, we consider $\cexpect{i,i+\delta}-\cexpect{i,i+\delta-1}$.\\
        (1) When $i=0$, we have $\cexpect{0,k+1}-\cexpect{0,k}=0 < \frac{n}{2(k+1)}$. \\
        (2) When $i>0$, we have
            \begin{align}
                & \forall i > 0, \delta \geq 1:\mathbb{E}(i,i+\delta)-\mathbb{E}(i,i+\delta-1)\\
                & = \frac{n^2}{(2i+\delta)n-i(i+\delta)}+\frac{i(i+\delta)}{(2i+\delta)n-i(i+\delta)}\mathbb{E}(i-1,i+\delta-1)\\
                & \quad +\frac{i(n-i-\delta)}{(2i+\delta)n-i(i+\delta)}\mathbb{E}(i-1,i+\delta)-\frac{in}{(2i+\delta)n-i(i+\delta)}\mathbb{E}(i,i+\delta-1)\\
                & < \frac{n^2}{(2i+\delta)n-i(i+\delta)}-\frac{i(n-i-\delta)}{(2i+\delta)n-i(i+\delta)}\mathbb{E}(i-1,i+\delta-1)\\
                & \quad +\frac{i(n-i-\delta)}{(2i+\delta)n-i(i+\delta)}\mathbb{E}(i-1,i+\delta)-\frac{n^2}{2((2i+\delta)n-i(i+\delta))}\\
                & < \frac{n^2}{(2i+\delta)n-i(i+\delta)}+\frac{in(n-i-\delta)}{2(i+\delta)((2i+\delta)n-i(i+\delta))}-\frac{n^2}{2((2i+\delta)n-i(i+\delta))}\\
                & =\frac{n}{2(i+\delta)},
            \end{align}
            where the first inequality is by $\mathbb{E}(i,i+\delta-1)-\mathbb{E}(i-1,i+\delta-1)>\frac{n}{2i}$ since $i+i+\delta-1=k \wedge i \geq 1 \wedge i+\delta-1-i \geq
            0$, and the second inequality is by $\mathbb{E}(i-1,i+\delta)-\mathbb{E}(i-1,i+\delta-1)<\frac{n}{2(i+\delta)}$ since $i-1+i+\delta=k \wedge i-1 \geq 0 \wedge i+\delta-i+1 \geq 1$.

        Second, we consider  $\mathbb{E}(i,i+\delta)-\mathbb{E}(i-1,i+\delta)$.\\
        (1) When $\delta=0$, we have
            \begin{align}
                & \forall i \geq 1: \mathbb{E}(i,i)-\mathbb{E}(i-1,i)\\
                &=\frac{n^2}{2in-i^2}-\frac{i^2}{2in-i^2}\mathbb{E}(i-1,i)+\frac{i^2}{2in-i^2}\mathbb{E}(i-1,i-1)\\
                &>\frac{n^2}{2in-i^2}-\frac{in}{2(2in-i^2)}\\
                &=\frac{n}{2i},
            \end{align}
            where the inequality is by $\mathbb{E}(i-1,i)-\mathbb{E}(i-1,i-1)<\frac{n}{2i}$ since $i-1+i=k \wedge i-1 \geq 0 \wedge i-i+1=1$.\\
        (2) When $\delta>0$, we have
            \begin{align}
                & \forall i \geq 1:\mathbb{E}(i,i+\delta)-\mathbb{E}(i-1,i+\delta)\\
                & = \frac{n^2}{(2i+\delta)n-i(i+\delta)}+\frac{i(i+\delta)}{(2i+\delta)n-i(i+\delta)}\mathbb{E}(i-1,i+\delta-1)\\
                & \quad +\frac{(n-i)(i+\delta)}{(2i+\delta)n-i(i+\delta)}\mathbb{E}(i,i+\delta-1)-\frac{(i+\delta)n}{(2i+\delta)n-i(i+\delta)}\mathbb{E}(i-1,i+\delta)\\
                & >\frac{n^2}{(2i+\delta)n-i(i+\delta)}-\frac{(n-i)(i+\delta)}{(2i+\delta)n-i(i+\delta)}\mathbb{E}(i-1,i+\delta-1)\\
                &  \quad +\frac{(n-i)(i+\delta)}{(2i+\delta)n-i(i+\delta)}\mathbb{E}(i,i+\delta-1)-\frac{n^2}{2((2i+\delta)n-i(i+\delta))}\\
                & >\frac{n^2}{(2i+\delta)n-i(i+\delta)}+\frac{n(n-i)(i+\delta)}{2i((2i+\delta)n-i(i+\delta))}-\frac{n^2}{2((2i+\delta)n-i(i+\delta))}\\
                & =\frac{n}{2i},
            \end{align}
            where the first inequality is by $\mathbb{E}(i-1,i+\delta)-\mathbb{E}(i-1,i+\delta-1)<\frac{n}{2(i+\delta)}$ since $i-1+i+\delta=k \wedge i-1\geq 0 \wedge i+\delta-i+1 \geq
            1$, and the second inequality is by $\mathbb{E}(i,i+\delta-1)-\mathbb{E}(i-1,i+\delta-1)>\frac{n}{2i}$ since $i+i+\delta-1=k \wedge i \geq 1 \wedge i+\delta-1-i \geq 0$.

    {\bf (c) Conclusion}: According to (a) and (b), we have
            \begin{align}
                & \forall i \geq 0, \delta \geq 1: \quad \mathbb{E}(i,i+\delta)-\mathbb{E}(i,i+\delta-1) < \frac{n}{2(i+\delta)},\\
                & \forall i \geq 1, \delta \geq 0: \quad \mathbb{E}(i,i+\delta)-\mathbb{E}(i-1,i+\delta) > \frac{n}{2i}.
            \end{align}
 \end{myproofd}

 \begin{myproofd}{Proposition \ref{prop_onemax_dist1}}
    We prove the proposition by induction on $t$.\\
    {\bf{(a) Initialization}} First, we have $p_0(0,0)= p_0(0,1)=\frac{1}{4}$, since the initial population is random. \\
    {\bf{(b) Inductive Hypothesis}} Assume for $0 \leq t \leq K-1$,
    \begin{align}
    & p_t(0,0)=\frac{1}{4}(1-\frac{2(1-p_c)}{n}+\frac{1-p_c}{n^2})^t;\\
    & p_t(0,1)=
    \begin{cases}
    \frac{1}{4}(1+\frac{t}{2n-1})(1-\frac{1}{n})^t,& \text{if $p_c=\frac{n-1}{2n-1}$,}\\
    \frac{n-1}{4(1-2n+\frac{n}{1-p_c})}(1-\frac{2(1-p_c)}{n}+\frac{1-p_c}{n^2})^t+\frac{2-3n+\frac{n}{1-p_c}}{4(1-2n+\frac{n}{1-p_c})}(1-\frac{1}{n})^t,&\text{otherwise.}
    \end{cases}
    \end{align}
    Then, for $t=K$, we have
    \begin{align}
    & p_K(0,0)= p_{K-1}(0,0)(p_c+(1-p_c)(1-\frac{1}{n})^2)\\
    & =\frac{1}{4}(1-\frac{2(1-p_c)}{n}+\frac{1-p_c}{n^2})^{K-1}(p_c+(1-p_c)(1-\frac{1}{n})^2)\\
    &=\frac{1}{4}(1-\frac{2(1-p_c)}{n}+\frac{1-p_c}{n^2})^K,
    \end{align}
    where the first equality can be derived by analyzing the one-step
    transition behavior of $\{\xi_t\}^{\infty}_{t=0}$ in Proposition
    \ref{prop_transition_OneMax}, and the second equality is by
    inductive hypothesis. Similarly, we have
    \begin{align}
    &
    p_K(0,1)=p_cp_{K-1}(0,1)(1-\frac{1}{n})+(1-p_c)(p_{K-1}(0,1)(1-\frac{1}{n})+p_{K-1}(0,0)(1-\frac{1}{n})\frac{1}{n})\\
    &=(1-\frac{1}{n})p_{K-1}(0,1)+\frac{1-p_c}{n}(1-\frac{1}{n})p_{K-1}(0,0)
    \end{align}
    When $p_c=\frac{n-1}{2n-1}$, we have
    \begin{align}
    &p_K(0,1)=\frac{1}{4}(1-\frac{1}{n})(1+\frac{K-1}{2n-1})(1-\frac{1}{n})^{K-1}+\frac{1}{2n-1}(1-\frac{1}{n})(\frac{1}{4}(1-\frac{1}{n})^{K-1})\\
    &=\frac{1}{4}(1+\frac{K}{2n-1})(1-\frac{1}{n})^{K}
    \end{align}
    When $p_c \neq \frac{n-1}{2n-1}$, we have
    \begin{align}
    &p_K(0,1)=(1-\frac{1}{n})(\frac{n-1}{4(1-2n+\frac{n}{1-p_c})}(1-\frac{2(1-p_c)}{n}+\frac{1-p_c}{n^2})^{K-1}+\frac{2-3n+\frac{n}{1-p_c}}{4(1-2n+\frac{n}{1-p_c})}(1-\frac{1}{n})^{K-1})\\
    &\quad +\frac{1-p_c}{n}(1-\frac{1}{n})(\frac{1}{4}(1-\frac{2(1-p_c)}{n}+\frac{1-p_c}{n^2})^{K-1})\\
    &=\frac{n-1}{4(1-2n+\frac{n}{1-p_c})}(1-\frac{2(1-p_c)}{n}+\frac{1-p_c}{n^2})^K+\frac{2-3n+\frac{n}{1-p_c}}{4(1-2n+\frac{n}{1-p_c})}(1-\frac{1}{n})^K.
    \end{align}
    {\bf{(c) Conclusion}} From (a) and (b), the proposition holds.
\end{myproofd}

\begin{myproofd}{Proposition \ref{prop_onemax_dist2}}
    \begin{align}
    &\expect{\frac{N_x(0,1)}{N_x(0,1)+N_x(0,0)} \mid x \sim \pi_{t+1}}\\
    &=\sum_{x \in \mathcal{X}}\pi_{t+1}(x)\frac{N_x(0,1)}{N_x(0,1)+N_x(0,0)} \\
    &=\sum_{x,y \in \mathcal{X}} \pi_{t}(x) P(\xi_{t+1}=y \mid\xi_{t}=x) \frac{N_y(0,1)}{N_y(0,1)+N_y(0,0)}\\
    &=\sum_{x \in \mathcal{X}} \pi_t(x)
    (p_c(\frac{N_x(0,1)}{n}\frac{N_x(0,1)-1}{N_x(0,1)+N_x(0,0)-1}+\frac{n-N_{x}(0,1)}{n}\frac{N_x(0,1)}{N_x(0,1)+N_x(0,0)})\\
    &+(1-p_c)(\frac{(n-N_x(0,0)-N_x(0,1))(n-N_x(0,0)-N_x(1,0))}{n^2}\frac{N_x(0,1)}{N_x(0,1)+N_x(0,0)}\\
    &+\frac{N_x(1,0)(n-N_x(0,0)-N_x(0,1))}{n^2}\frac{N_x(0,1)}{N_x(0,1)+N_x(0,0)}\\
    &+\frac{N_x(0,0)(n-N_x(0,0)-N_x(0,1))}{n^2}\frac{N_x(0,1)+1}{N_x(0,1)+N_x(0,0)}\\
    &+\frac{N_x(0,0)(n-N_x(0,0)-N_x(1,0))}{n^2}\frac{N_x(0,1)}{N_x(0,1)+N_x(0,0)-1}\\
    &+\frac{N_x(0,1)(n-N_x(0,0)-N_x(1,0))}{n^2}\frac{N_x(0,1)-1}{N_x(0,1)+N_x(0,0)-1}\\
    &+\frac{N_x(0,1)N_x(1,0)}{n^2}\frac{N_x(0,1)-1}{N_x(0,1)+N_x(0,0)-1}+\frac{N_x(0,0)N_x(1,0)}{n^2}\frac{N_x(0,1)}{N_x(0,1)+N_x(0,0)-1}\\
    &+\frac{N_x(0,0)}{n^2}\frac{N_x(0,1)}{N_x(0,1)+N_x(0,0)-1}+\frac{N_x(0,0)(N_x(0,0)-1)}{n^2}\frac{N_x(0,1)+1}{N_x(0,1)+N_x(0,0)-1}\\
    &+\frac{N_x(0,1)N_x(0,0)}{n^2}\frac{N_x(0,1)}{N_x(0,1)+N_x(0,0)-1})\\
    &=\sum_{x \in \mathcal{X}}
    \pi_t(x)(1+(1-p_c)\frac{N_x(0,0)}{nN_x(0,1)}-p_c\frac{N_x(0,0)}{n(N_x(0,0)+N_x(0,1)-1)})\frac{N_x(0,1)}{N_x(0,0)+N_x(0,1)}\\
    & \leq\sum_{x \in \mathcal{X}}
    \pi_t(x)(1+(1-p_c)\frac{N_x(0,0)}{nN_x(0,1)})\frac{N_x(0,1)}{N_x(0,0)+N_x(0,1)}\\
    &=(1-\frac{1-p_c}{n})\sum_{x \in \mathcal{X}}
    \pi_t(x)\frac{N_x(0,1)}{N_x(0,0)+N_x(0,1)}+\frac{1-p_c}{n}.\\
    &=(1-\frac{1-p_c}{n})\expect{\frac{N_x(0,1)}{N_x(0,1)+N_x(0,0)} \mid
    x \sim \pi_{t}}+\frac{1-p_c}{n},
    \end{align}
    where the second equality is by $\pi_{t+1}(x)=\sum_{y \in
    \mathcal{X}} P(\xi_{t+1}=x \mid \xi_t=y)\pi_t(y)$, and the third
    equality can be derived by analyzing the one-step transition
    behavior in Proposition \ref{prop_transition_OneMax} plus the change
    of $N_x(0,1)/(N_x(0,1)+N_x(0,0))$ in adjacent time steps.
\end{myproofd}

\begin{myproofd}{Proposition \ref{prop_onemax_dist3}}
    We prove it by induction on $t$.\\
    {\bf{(a) Initialization}} First, when $t=0$, we have
    \begin{align}
    & \sum_{x \in \mathcal{X}} \pi_0(x) \frac{N_x(0,1)}{N_x(0,1)+N_x(0,0)}\\
    &=\sum^n_{i=0}\sum^i_{k=0} \frac{k}{i} \frac{2^{n-i} {n \choose i} {i \choose k} }{4^n}=\frac{1}{4^n} \sum^n_{i=0} 2^{n-i} {n \choose i} \cdot  2^{i-1}\\
    &=\frac{1}{2} \leq
    1-p_0(0,1)-p_0(0,0)-(1-\frac{1-p_c}{n})^0=\frac{1}{2},
    \end{align}
    where the first equality is since the initial population is
    random.\\
    {\bf{(b) Inductive Hypothesis}} Assume that for $0 \leq t \leq K-1$,
    it holds that
    $$
    \sum_{x \in \mathcal{X}} \pi_t(x) \frac{N_x(0,1)}{N_x(0,1)+N_x(0,0)}
    \leq 1-p_t(0,1)-p_t(0,0)-(1-\frac{1-p_c}{n})^t.
    $$
    Then, we consider $t=K$.
    \begin{align}
    & \sum_{x \in \mathcal{X}} \pi_K(x) \frac{N_x(0,1)}{N_x(0,1)+N_x(0,0)}\\
    & \leq (1-\frac{1-p_c}{n})\sum_{x \in \mathcal{X}} \pi_{K-1}(x)\frac{N_x(0,1)}{N_x(0,1)+N_x(0,0)}+\frac{1-p_c}{n}\\
    & \leq
    (1-\frac{1-p_c}{n})(1-p_{K-1}(0,1)-p_{K-1}(0,0)-(1-\frac{1-p_c}{n})^{K-1})+\frac{1-p_c}{n},
    \end{align}
    where the first inequality is by Proposition
    \ref{prop_onemax_dist2}, and the second is by inductive
    hypothesis.\\
    When $p_c=\frac{n-1}{2n-1}$, we have
    \begin{align}
    &\sum_{x \in \mathcal{X}} \pi_K(x) \frac{N_x(0,1)}{N_x(0,1)+N_x(0,0)}\\
    & \leq
    1-\frac{2n-2}{4(2n-1)}(2+\frac{K-1}{2n-1})(1-\frac{1}{n})^{K-1}-(\frac{2n-2}{2n-1})^K\\
    & \leq 1-\frac{1}{4}(2+\frac{K}{2n-1})(1-\frac{1}{n})^K-(\frac{2n-2}{2n-1})^K\\
    &=1-p_K(0,1)-p_K(0,0)-(1-\frac{1-p_c}{n})^K,
    \end{align}
    where the first inequality and the last equality are by Proposition
    \ref{prop_onemax_dist1}. \\
    When $p_c \neq \frac{n-1}{2n-1}$, we have
    \begin{align}
    &\sum_{x \in \mathcal{X}} \pi_K(x) \frac{N_x(0,1)}{N_x(0,1)+N_x(0,0)}\\
    & \leq 1-(1-\frac{1-p_c}{n})(\frac{\frac{n}{1-p_c}-n}{4(1-2n+\frac{n}{1-p_c})}(1-\frac{2(1-p_c)}{n}+\frac{1-p_c}{n^2})^{K-1}+\frac{2-3n+\frac{n}{1-p_c}}{4(1-2n+\frac{n}{1-p_c})}(1-\frac{1}{n})^{K-1}\\
    &\quad +(1-\frac{1-p_c}{n})^{K-1}) \\
    &\leq 1-(\frac{\frac{n}{1-p_c}-n}{4(1-2n+\frac{n}{1-p_c})}(1-\frac{2(1-p_c)}{n}+\frac{1-p_c}{n^2})^{K}+\frac{2-3n+\frac{n}{1-p_c}}{4(1-2n+\frac{n}{1-p_c})}(1-\frac{1}{n})^{K})-(1-\frac{1-p_c}{n})^K\\
    &=1-p_K(0,1)-p_K(0,0)-(1-\frac{1-p_c}{n})^K,
    \end{align}
    where the first inequality and the last equality are by Proposition
    \ref{prop_onemax_dist1}. \\
    {\bf{(c) Conclusion}} From (a) and (b), we have
    $$
    \forall t: \sum_{x \in \mathcal{X}} \pi_t(x)
    \frac{N_x(0,1)}{N_x(0,1)+N_x(0,0)} \leq
    1-p_t(0,1)-p_t(0,0)-(1-\frac{1-p_c}{n})^t.
    $$
\end{myproofd}

\section{Proofs for Section \ref{sec:condition}}\label{appendix_3}

\begin{myproofd}{Proposition \ref{prop_cfht3_LO}}\label{app_cfht3_LO}
    We prove the proposition by induction on $i+i+\delta$.\\
    {\bf (a) Initialization}:\\
        When $i+i+\delta=1$, we have $\cexpect{0,1}-\cexpect{0,0}=0 < \frac{n}{2}$. \\
        When $i+i+\delta=2$, we have $\cexpect{0,2}-\cexpect{0,1}=0 < \frac{n}{2}$, $\cexpect{1,1}-\cexpect{0,1}=\frac{n^2}{2n-1} > \frac{n}{2}$. \\
    {\bf (b) Inductive Hypothesis}: Assume \\
        for $1 \leq i+i+\delta \leq k$, we have
        \begin{equation}
            \begin{aligned}
                & \forall i \geq 0, \delta \geq 1: \cexpect{i,i+\delta}-\cexpect{i,i+\delta-1} < \frac{n}{2}, \\
                & \forall i \geq 1, \delta \geq 0: \cexpect{i,i+\delta}-\cexpect{i-1,i+\delta} > \frac{n}{2}.
            \end{aligned}
        \end{equation}
        We then prove that the two inequations still hold when $i+i+\delta=k+1$ .\\
        First, we consider $\cexpect{i,i+\delta}-\cexpect{i,i+\delta-1}$.\\
        (1) When $i=0$, we have $\cexpect{0,k+1}-\cexpect{0,k}=0 < \frac{n}{2}$. \\
        (2) When $i>0$, we have\\
        $\cexpect{i,i+\delta}=1+\frac{1}{n^2}\cexpect{i-1,i+\delta-1}+\frac{n-1}{n^2}\cexpect{i-1,i+\delta}+\frac{n-1}{n^2}\cexpect{i,i+\delta-1}+\frac{(n-1)^2}{n^2}\cexpect{i,i+\delta}$.\\
        Then,
        $\cexpect{i,i+\delta}=\frac{n^2}{2n-1}+\frac{1}{2n-1}\cexpect{i-1,i+\delta-1}+\frac{n-1}{2n-1}\cexpect{i-1,i+\delta}+\frac{n-1}{2n-1}\cexpect{i,i+\delta-1}$.\\
        Thus, we have
        \begin{equation}
            \begin{aligned}
                & \forall i > 0, \delta \geq 1:\mathbb{E}(i,i+\delta)-\mathbb{E}(i,i+\delta-1)\\
                & =\frac{n^2}{2n-1}+\frac{1}{2n-1}\mathbb{E}(i-1,i+\delta-1)+\frac{n-1}{2n-1}\mathbb{E}(i-1,i+\delta)-\frac{n}{2n-1}\mathbb{E}(i,i+\delta-1)\\
                & < \frac{n^2}{2n-1}-\frac{n-1}{2n-1}\mathbb{E}(i-1,i+\delta-1)+\frac{n-1}{2n-1}\mathbb{E}(i-1,i+\delta)-\frac{n^2}{2(2n-1)}\\
                &(\text{since } i+i+\delta-1=k, i \geq 1 \text{ and }  i+\delta-1-i \geq 0, \text{ then } \mathbb{E}(i,i+\delta-1)-\mathbb{E}(i-1,i+\delta-1)>\frac{n}{2})\\
                & < \frac{n^2}{2n-1}+\frac{n(n-1)}{2(2n-1)}-\frac{n^2}{2(2n-1)}\\
                &(\text{since }  i-1+i+\delta=k, i-1 \geq 0 \text{ and } i+\delta-i+1 \geq 1 , \text{ then } \mathbb{E}(i-1,i+\delta)-\mathbb{E}(i-1,i+\delta-1)<\frac{n}{2})\\
                & =\frac{n}{2}.
            \end{aligned}
        \end{equation}
        Second, we consider  $\mathbb{E}(i,i+\delta)-\mathbb{E}(i-1,i+\delta)$.\\
        (1) When $\delta=0$, we have\\
        \begin{equation}
            \begin{aligned}
                & \forall i \geq 1: \mathbb{E}(i,i)-\mathbb{E}(i-1,i)\\
                &=\frac{n^2}{2n-1}+\frac{2(n-1)}{2n-1}\mathbb{E}(i-1,i)+\frac{1}{2n-1}\mathbb{E}(i-1,i-1)-\mathbb{E}(i-1,i)\\
                &=\frac{n^2}{2n-1}-\frac{1}{2n-1}\mathbb{E}(i-1,i)+\frac{1}{2n-1}\mathbb{E}(i-1,i-1)\\
                &>\frac{n^2}{2n-1}-\frac{n}{2(2n-1)}\\
                &(since \; i-1+i=k,\; i-1 \geq 0 \; and \; i-i+1=1,\; then \; \mathbb{E}(i-1,i)-\mathbb{E}(i-1,i-1)<\frac{n}{2})\\
                &=\frac{n}{2}.
            \end{aligned}
        \end{equation}
        (2) When $\delta>0$, we have\\
        \begin{equation}
            \begin{aligned}
                & \forall i \geq 1:\mathbb{E}(i,i+\delta)-\mathbb{E}(i-1,i+\delta)\\
                & =\frac{n^2}{2n-1}+\frac{1}{2n-1}\mathbb{E}(i-1,i+\delta-1)+\frac{n-1}{2n-1}\mathbb{E}(i,i+\delta-1)-\frac{n}{2n-1}\mathbb{E}(i-1,i+\delta)\\
                & >\frac{n^2}{2n-1}+\frac{n-1}{2n-1}\mathbb{E}(i,i+\delta-1)-\frac{n-1}{2n-1}\mathbb{E}(i-1,i+\delta-1)-\frac{n^2}{2(2n-1)}\\
                & (since \; i-1+i+\delta=k,\;,i-1\geq 0 \; and \; i+\delta-i+1 \geq 1,\; then \; \mathbb{E}(i-1,i+\delta)-\mathbb{E}(i-1,i+\delta-1)<\frac{n}{2})\\
                & >\frac{n^2}{2n-1}+\frac{(n-1)n}{2(2n-1)}-\frac{n^2}{2(2n-1)}\\
                &(since \; i+i+\delta-1=k,\; i \geq 1 \; and \; i+\delta-1-i \geq 0,\; then \; \mathbb{E}(i,i+\delta-1)-\mathbb{E}(i-1,i+\delta-1)>\frac{n}{2})\\
                & =\frac{n}{2}
            \end{aligned}
        \end{equation}
    {\bf (c) Conclusion}: According to (a) and (b), we have \\
        \begin{equation}
            \begin{aligned}
                \forall i \geq 0, \delta \geq 1: \quad \mathbb{E}(i,i+\delta)-\mathbb{E}(i,i+\delta-1) < \frac{n}{2},\\
                \forall i \geq 1, \delta \geq 0: \quad \mathbb{E}(i,i+\delta)-\mathbb{E}(i-1,i+\delta) > \frac{n}{2}.
            \end{aligned}
        \end{equation}
    \end{myproofd}
\end{document}